\documentclass{article}

\usepackage{microtype}
\usepackage{graphicx}
\usepackage{subcaption}

\usepackage{hyperref}
\usepackage{amssymb}

\usepackage[accepted]{icml2026}

\usepackage{url}            
\usepackage{booktabs}       
\usepackage{amsfonts}       
\usepackage{nicefrac}       
\usepackage{amsthm} 
\usepackage{xcolor}         
\usepackage{amsmath}
\usepackage{enumitem}
\usepackage{tcolorbox}
\usepackage{pifont}
\usepackage{xspace}
\usepackage{wrapfig,lipsum,booktabs}
\usepackage{adjustbox}
\usepackage{booktabs,longtable,amsmath}
\usepackage[table,xcdraw]{xcolor}
\usepackage{multirow}

\usepackage[capitalize,noabbrev]{cleveref}

\theoremstyle{plain}
\newtheorem{theorem}{Theorem}[section]
\newtheorem{proposition}[theorem]{Proposition}

\newtheorem{definition}[theorem]{Definition}
\newtheorem{assumption}[theorem]{Assumption}
\newtheorem{remark}[theorem]{Remark}

\usepackage[textsize=tiny]{todonotes}

\icmltitlerunning{On Information Self-Locking in Reinforcement Learning for Active Reasoning of LLM agents}

\usepackage{amsmath,amsfonts,bm}

\def\eqref#1{equation~\ref{#1}}

\def\1{\bm{1}}

\def\rvw{{\mathbf{w}}}

\begin{document}

\definecolor{mydarkblue}{rgb}{0,0.08,0.45}
\definecolor{myorange}{rgb}{0.85, 0.45, 0.1} 
\definecolor{mydarkgray}{rgb}{0.45, 0.45, 0.45}
\newcommand{\ob}{\color{mydarkblue}}
\newcommand{\reb}{\color{red}}
\newcommand{\oo}{\color{myorange}}
\newcommand{\og}{\color{mydarkgray}}
\newcommand{\proj}{{\textsc{AReW}}\xspace}
\newcommand{\infoSL}{SeL\xspace}
\newcommand{\GN}{{GuessNumbers}\xspace}
\newcommand{\SP}{{SituationPuzzles}\xspace}
\newcommand{\CD}{{CircuitDecoding}\xspace}
\newcommand{\MR}{{MovieRecommendation}\xspace}
\newcommand{\sel}{{SeL}\xspace}
\newcommand{\PE}{{PreferenceEstimation}\xspace}
\newcommand{\var}[1]{\color{mydarkblue}{\texttt{\{#1\}}}}
\newcommand{\ready}[1]{\textcolor{red}{#1}}
\newcommand{\yq}[1]{{\color{cyan}{yq: #1}}}
\newcommand{\pan}[1]{{\color{red}{pan: #1}}}
\newcommand{\ML}[1]{{\color{teal}M.L.:[#1]}}
\definecolor{fred}{HTML}{D62727}
\newcommand{\fan}[1]{\textcolor{fred}{\textbf{Fan:} #1}}
\twocolumn[
  \icmltitle{
  On Information Self-Locking in Reinforcement Learning \\
  for Active Reasoning of LLM agents
  }



  \icmlsetsymbol{equal}{*}
  \icmlsetsymbol{lead}{$\dagger$}
  \icmlsetsymbol{co}{$\ddagger$}
  \begin{icmlauthorlist}
    \icmlauthor{Deyu Zou}{equal,cuhk}
    \icmlauthor{Yongqiang Chen}{equal,cuhk}
    \icmlauthor{Fan Feng}{sd}
    \icmlauthor{Mufei Li}{gt}
    \icmlauthor{Pan Li}{gt}
    \icmlauthor{Yu Gong}{lead,db}
    \icmlauthor{James Cheng}{co,cuhk}
  \end{icmlauthorlist}

  \icmlaffiliation{cuhk}{The Chinese University of Hong Kong}
  \icmlaffiliation{gt}{Georgia Institute of Technology}
  \icmlaffiliation{db}{ByteDance}
  \icmlaffiliation{sd}{University of California San Diego}

  \icmlcorrespondingauthor{James Cheng}{jcheng@cse.cuhk.edu.hk}

  \icmlkeywords{Agentic Reinforcement Learning, Agentic Active Reasoning, Self-locking}

  \vskip 0.3in
]



\printAffiliationsAndNotice{\icmlEqualContribution \, \icmlpl \, \icmlca}

\begin{abstract}
   Reinforcement learning (RL) has become a \textit{de facto} paradigm for building LLM-based agents that act, interact, and reason over extended task horizons.
  However, in \textbf{\textit{active reasoning}} 
  where agents must elicit new observations through interaction with the environment to solve the task,
  we find that outcome-based RL can induce a systematic failure mode which we call \textbf{\textit{information self-locking (SeL)}}: 
  agents 
  fail both to {elicit informative feedback} and to internalize obtained evidence.
  To understand the issue, we trace agentic behaviors into two coupled capabilities: \textbf{\textit{Action Selection (AS)}}, which determines  observation streams, 
  and \textbf{\textit{Belief Tracking (BT)}}, which updates the agent’s internal task understanding. 
  Theoretical  and empirical analyses reveal a \textbf{\textit{bidirectional}} bottleneck that leads to SeL: weak BT obscures the credit of informative actions, while weak AS deprives BT of useful evidence. 
  This coupling weakens the learning signal for both capabilities and leads to SeL.
  To mitigate this issue, we propose \proj, a simple yet effective \underline{A}dvantage-\underline{Rew}eighting method that uses easy-to-obtain directional critiques to reallocate credit within trajectories.
  Extensive experiments across $\mathbf{9}$ agentic tasks of varying complexity show that \proj significantly mitigates SeL, yielding up to 60-point gains in final performance.  Code is available at
\url{https://github.com/unimpor/T3}.
\end{abstract}

\section{Introduction}

Reinforcement learning (RL) with outcome-based rewards has demonstrated great success in improving the reasoning capabilities of large language models (LLMs)~\citep{wang2024reinforcement,srivastava2025technical,xu2025towards,guo2025deepseek}. Recently, this paradigm has become increasingly important for building LLM agents that solve tasks through extended interaction with external environments, rather than from a single prompt alone~\citep{zhang2025landscape,plaat2025agentic}. In contemporary agentic applications~\citep{zhang2025deep, sager2026comprehensive}, such as deep research, coding, computer use, \textit{etc.}, the agent often does not receive all task-relevant information upfront. Instead, the agent must make progress by actively taking environment-facing actions, receiving new observations or feedback, and updating its understanding of the task state over multiple turns. We refer to this setting as \textbf{\emph{agentic active reasoning}}.

Despite the success of outcome-based RL, we observe a recurring failure mode when training LLM-based agents for active reasoning. Rather than learning to acquire and use more task-relevant information, agents can drift into low-information interaction patterns: they take actions that reveal little useful evidence, and even when useful observations are returned, they do not reliably incorporate them into subsequent reasoning. We refer to this training regime as \textbf{\emph{information self-locking}} (\textit{\textbf{\infoSL}}). 
Under \infoSL, outcome-based RL may still improve apparent task performance through interaction-insensitive shortcuts, but it fails to strengthen the information acquisition and evidence-integration behaviors needed for active reasoning.
This raises a crucial research question: \emph{Why does SeL happen under outcome-based RL for agentic active reasoning, and how can we mitigate it?}

\begin{figure*}[t]
     \centering
\includegraphics[width=0.9\linewidth,trim=40 50 20 50,clip]{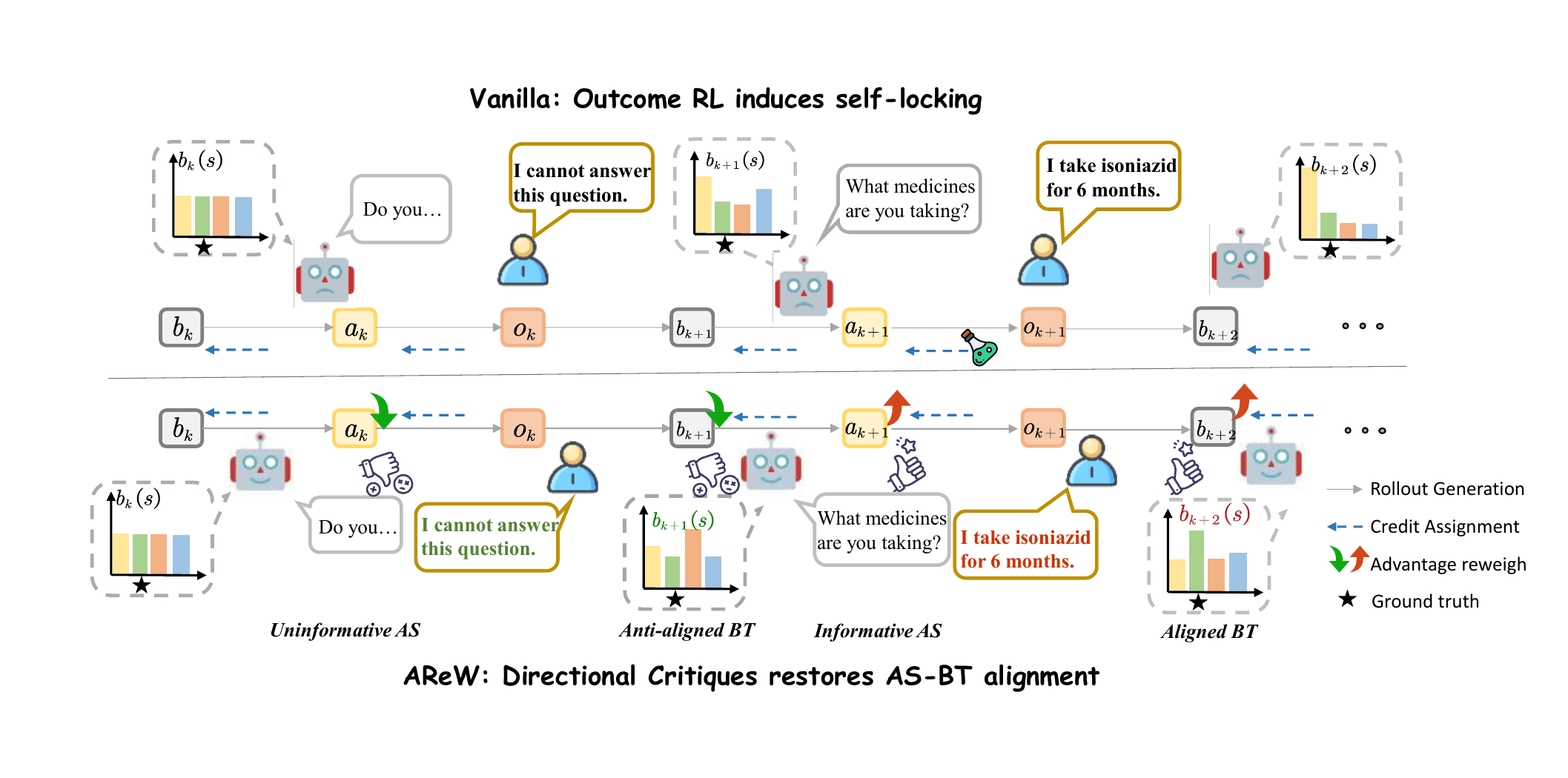}
\vspace{-0.05in}
     \caption{
     Overall illustration of information self-locking (\sel) and its mitigation.
$(b,a,o)$ denote the agent’s internal belief, its chosen action, and the resulting feedback at each turn.
Under vanilla outcome-based RL (top), the agent can become trapped in a self-locking regime: deficient belief tracking masks contributions of  informative actions, leading to misaligned credit assignment.
Our \proj (bottom) introduces advantage reweighting via directional critiques, correcting the learning signal and assisting to mitigate \sel in active reasoning.
}
    \label{fig:pipeline}
    \vspace{-.2in}
\end{figure*}

To answer this question, we start by exploring the coupled nature of the credit assignment system in active reasoning. Specifically, the agent's performance depends on two intertwined capabilities: \emph{\textbf{Action Selection}} (\textbf{\textit{AS}}), which determines the observation stream, and \emph{\textbf{\textit{Belief Tracking}}} (\textbf{\textit{BT}}), which governs evidence internalization. Crucially, these two are mutually dependent: 
informative actions receive credit only when the resulting observations are properly utilized by BT,
while effective BT refinement requires a high-quality information stream from AS.
Concretely, empirical evidence (Sec.~\ref{sec:phenomenon}) reveals this coupling stagnation: even as rewards increase, both AS and BT capabilities remain dormant, suggesting that the sparse outcome reward fails to decouple and individually improve these two fundamental gears.

To further formalize this intuition, we theoretically characterize the coupled learning dynamics of AS and BT under outcome-based optimization (Sec.~\ref{sec:theory}). 
Our analysis decomposes the outcome-gradient into AS and BT channels and demonstrates that, in a low-information regime, their improvement signals become capability-limited: 
weak belief tracking reduces the reward sensitivity to informative actions, thereby masking the learning signal for AS; 
at the same time, deficient AS restricts the evidence stream available for belief refinement, limiting the upward drift of BT. 
Together, these effects form a coupled learning bottleneck (Thm.~\ref{thm:sel_lock_informal}), where the outcome-gradient cannot provide a sufficiently strong improvement signal for either capability, causing the agent to remain trapped in a persistent low-information regime (SeL regime, Def.~\ref{def:locking_regime}).

Built on this theoretical understanding, we propose \proj, a lightweight and calibration-free \underline{A}dvantage-\underline{Rew}eighting framework for breaking \infoSL. 
{The key observation is that active reasoning often provides easy-to-obtain local diagnostic signals, \textit{e.g.}, 
whether a query is informative can be inferred from whether the user reveals new evidence in response. 
\proj uses them as weak \textit{directional critiques} for AS and BT, rather than implementing calibrated intermediate rewards or training a dense reward model.}
It then injects these critiques at the policy-gradient level by reweighting stepwise advantages, reallocating credit within each trajectory from negatively critiqued decisions to positively critiqued ones while preserving the original outcome reward. 
This reward-preserving credit reallocation supplies non-degenerate learning signals in the SeL regime, enabling AS and BT to improve jointly and escape \infoSL.

Extensive experiments across \textbf{9} agentic active reasoning tasks  of varying complexity show that \proj consistently mitigates \infoSL across tasks, algorithms, and model families, and is robust to critique design of different noise levels. Beyond improving final performance, \proj fundamentally alters training dynamics: agents recover information-seeking interaction patterns and exhibit sustained growth in both action selection and belief tracking capabilities.

\section{Self-Locking in Active Reasoning}
\label{sec:phenomenon}

\subsection{Preliminaries}

In active reasoning, an LLM agent solves a task under partial observability: the information provided at the beginning of an episode is insufficient for completing the task, and the agent must interact with an external environment to acquire task-relevant observations. At each turn, the agent takes an environment-facing action $a_t \in \mathcal A$, such as issuing a query or invoking a tool, and receives an observation $o_t \in \mathcal O$ from the environment. The agent must then use the accumulated interaction history to update its understanding of the task state and decide what action to take next.

We model active reasoning as a Partially Observable Markov Decision Process (POMDP) $(\mathcal S,\mathcal A,\mathcal O,T,O,R,\gamma)$~\citep{kaelbling1998planning}, where $\mathcal S$ is the latent state space, $\mathcal A$ is the action space, $\mathcal O$ is the observation space, $T(s'\mid s,a)$ is the transition dynamics, $O(o\mid s,a)$ is the observation model, $R$ is the reward function, and $\gamma$ is the discount factor. In many agentic active reasoning tasks, the latent state $s^\star \in \mathcal S$ may represent the hidden user preference, diagnosis, solution, \textit{etc}. The agent cannot directly observe $s^\star$.
{Belief tracking} is central to active reasoning. For analysis, we associate the agent with a model belief $b_t\in\Delta(\mathcal S)$ at turn $t\in\{0,\ldots,H\}$, which represents its internal understanding of the latent task state and the current progress of problem solving. 
In this work, the belief is mainly leveraged as an analytical abstraction of the agent's task-state understanding induced by its parameters and interaction history.


Under this abstraction, the behavior of an agent with parameters \(\omega\) can be decomposed into two coupled processes.

\textbf{Action Selection (AS).}
Given its current belief, the agent selects an environment-facing action
\(a_t \sim \pi_\omega^{\mathrm{as}}(\cdot \mid b_t)\),
which determines what observations may become available next through
\(o_t \sim O(\cdot \mid s^\star,a_t)\).
The role of AS is therefore to shape the future observation stream by choosing actions that reduce uncertainty or otherwise advance task completion.

\textbf{Belief Tracking (BT).}
After receiving the observation \(o_t\), the agent updates its internal task-state understanding through a belief-update kernel
$b_{t+1} \sim \pi_\omega^{\mathrm{bt}}(\cdot\mid b_t,a_t,o_t)$,
which integrates the newly acquired observation with information accumulated from previous interaction rounds.

Together, these two kernels define the agent's interleaved behavior
$\Pi_\omega := \bigl(\pi_\omega^{\mathrm{as}}, \pi_\omega^{\mathrm{bt}}\bigr)$.
Over a horizon \(H\), \(\Pi_\omega\) induces trajectories
$\tau=(b_0,a_0,o_0,b_1,\ldots,a_{H-1},o_{H-1},b_H)$
with likelihood $p_\omega(\tau)$ which can be factorized as follows, 
{\setlength{\abovedisplayskip}{3pt}
\setlength{\belowdisplayskip}{3pt}
\setlength{\abovedisplayshortskip}{3pt}
\setlength{\belowdisplayshortskip}{3pt}
\begin{equation*}
p_\omega(b_0)
\prod_{t=0}^{H-1}
\pi_\omega^{\mathrm{as}}(a_t\mid b_t)
O(o_t\mid s^\star,a_t)
\pi_\omega^{\mathrm{bt}}(b_{t+1}\mid b_t,a_t,o_t).
\end{equation*}}%
This decomposition highlights the information pipeline underlying active reasoning: AS controls what evidence the agent can obtain, while BT controls whether that evidence is internalized and used for subsequent decisions.

\subsection{Testbeds and decomposed proxies}
\label{sec:phenomenon_setup}

With the above decomposition, we can track the fine-grained behaviors of agents trained with outcome-based RL.
We now consider two interactive benchmarks, where we introduce proxies that separately track AS and BT dynamics. The full details can be found in Appendix~\ref{app:dataset}.

\textbf{Preference Estimation (PE).}
Adapted from \citet{badola2025multi},
PE is an interactive preference inference task under constrained information acquisition.
The agent is given a finite set of items $\mathcal{X} = \{x_1,\dots,x_N\}$, where each item $x_i$ is represented by a known $D$-attribute vector $\mathbf{a}_i \in \mathbb{R}^D$.
The user has an unknown latent preference vector $\mathbf{w}^\star \in [0,1]^D$.
Through interaction, the agent maintains and iteratively refines an estimate $\mathbf{w}_t \in [0,1]^D$ of the user preference.
At each round, the agent actively selects an attribute subspace $\mathcal{S}_t \subseteq \{1,\dots,D\}$ where $|\mathcal{S}_t|=S$ and a pair of items $(x_i, x_j) \in \mathcal{X} \times \mathcal{X}$ designed to elicit the user's preference feedback restricted to $\mathcal{S}_t$.
Based on the feedback, the agent updates its belief state.
The objective is to recover $\mathbf{w}^\star$ under sparse, outcome-based supervision. We consider two variants, \textbf{PE-{G}}ated where $1<S<D$, and \textbf{PE-{F}}ull where $S=D$, \textit{i.e.,} all dimensions are covered each turn.

\textbf{Proxies in PE-G.}
As it's hard to precisely quantify the informativeness, we introduce a binary proxy that is simple to implement and effective in tracking the AS behavior. 
Specifically, for a queried attribute subspace $\mathcal{S}_t$ and item pair $(x_i,x_j)$, we define the AS indicator
$\mathrm{AS}_t = \mathbb{I}\!\left[
\exists\, k_1,k_2 \in \mathcal{S}_t \;\text{s.t.,}\;
a_i^{(k_1)} > a_j^{(k_1)} \wedge
a_i^{(k_2)} < a_j^{(k_2)}
\right]$,
which means neither item strictly dominates the other on $\mathcal{S}_t$, ensuring that the resulting feedback is informative.
BT evaluates whether the agent can incorporate such informative feedback.
We measure BT by the improvement in similarity between the estimate and the ground-truth preference,
$\mathrm{BT}_t = \mathrm{sim}(\mathbf{w}_{t+1}, \mathbf{w}^\star)
-
\mathrm{sim}(\mathbf{w}_{t}, \mathbf{w}^\star)$,
where $\mathrm{sim}(\cdot,\cdot)$ denotes cosine similarity.
Positive $\mathrm{BT}_t$ indicates effective absorption of newly acquired information.

\textbf{MediQ.}
Adapted from \citet{li2024mediq}, agents in MediQ require asking the patient questions to identify the best hypothesis for the patient's symptoms.
The agent is provided with a clinical vignette and an associated medical question whose answer lies in a finite hypothesis set of size $D$.
The agent maintains a belief estimate $\mathbf{w}_t \in [0,1]^D$ for $D$ candidate hypotheses.
Through interaction, the agent actively queries the LLM-simulated user for diagnostic information, receives structured feedback, and updates each hypothesis score accordingly.
The learning objective is to progressively concentrate belief mass onto the correct hypothesis.

\textbf{Proxies in MediQ.}
AS is quantified by the amount of novel diagnostic evidence elicited by the queries.
Let $\mathcal{E}_t$ denote the set of atomic clinical facts revealed at turn $t$, we define
$\mathrm{AS}_t =\bigl|\mathcal{E}_t \setminus \textstyle\bigcup_{\tau < t} \mathcal{E}_\tau\bigr|$,
to capture the information gain of each query.
BT measures whether newly observed evidence sharpens hypothesis discrimination.
Let $\mathrm{gt}$ denote the ground-truth hypothesis index.
We define BT via the change in belief margin
$\mathrm{BT}_t=
\Delta\!\left(
\rvw_t^{(\mathrm{gt})} - \max_{j\neq \mathrm{gt}} \rvw_t^{(j)}
\right),$
aggregated across turns, where larger positive values indicate more effective belief refinement.

\begin{figure*}
    \centering
    \captionsetup{aboveskip=4pt}
    \subfloat[\label{fig:3.1.1}PE-G$_{S=2}$]{\includegraphics[width=0.25\textwidth]{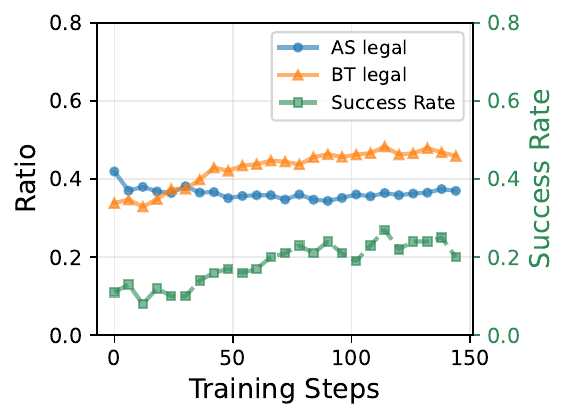}} 
    \subfloat[\label{fig:3.1.2}MediQ]{\includegraphics[width=0.25\textwidth]{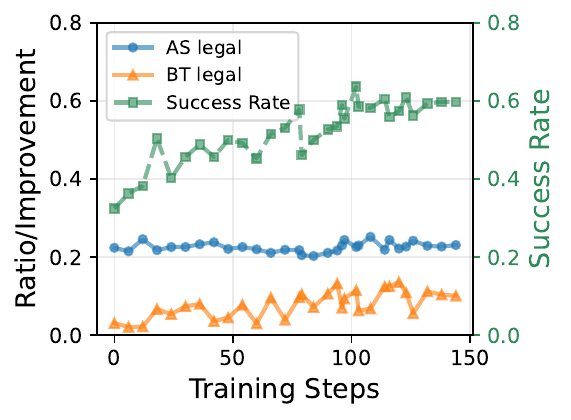}} 
    \subfloat[\label{fig:3.1.3}PE-G$_{S=2}$]{\includegraphics[width=0.248\textwidth]{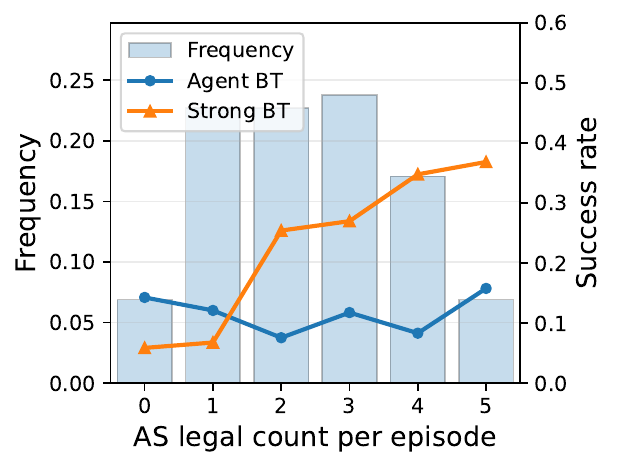}} 
    \subfloat[\label{fig:3.1.4}MediQ]{\includegraphics[width=0.248\textwidth]{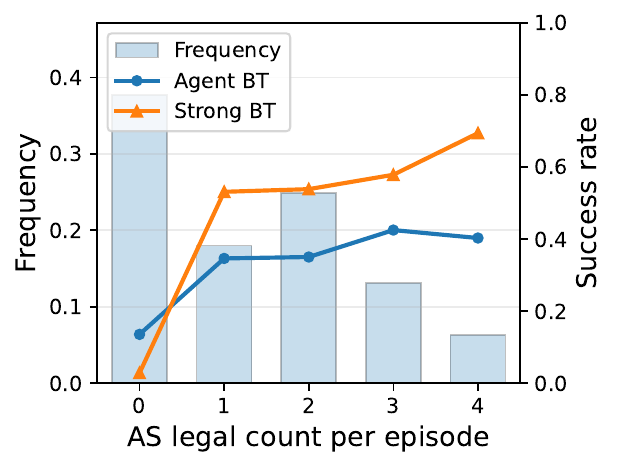}}
    \caption{(a)/(b): the training dynamics of
outcome reward, per-turn AS, and per-turn BT proxies in PE-G$_{S=2}$ and MediQ datasets (Qwen-2.5-7B-Instruct). 
(c)/(d): correlation between the reward and AS proxies in PE-G$_{S=2}$ and MediQ (Qwen-2.5-7B-Instruct). In strong BT patterns (by human-defined rules or frontier LLMs), the same AS sequence exhibits stronger correlation with the final reward. }
    \label{fig:sec3.1}
    \vspace{-0.19in}
\end{figure*}

\subsection{Failure modes in reinforcement learning training}
\label{subsec:isl_observations}
Despite the success of RL with outcome-based rewards, interestingly, we find that LLM agents exhibit several failure modes across both active-reasoning testbeds during training.

\textbf{Observation 1: Reward improvements do not translate into increased information acquisition.}
Fig.~\ref{fig:3.1.1} (PE-G) and Fig.~\ref{fig:3.1.2} (MediQ) report the training dynamics of episode reward, per-turn AS, and per-turn BT.
{ Across both datasets, we observe a pronounced decoupling: while the reward can be improved over training, BT exhibits only limited gains, and AS fails to improve, often plateauing or even degrading.}
This observation raises interesting questions about the confounding behaviors of AS and BT, as AS can not improve even with an improved BT.

To isolate the effect, we analyze the relationship between AS and reward under different BT capabilities.
Specifically, we fix identical action sequences and compare outcomes when the observation stream is processed by (i) the agent’s internal BT versus (ii) stronger belief-update mechanisms, \textit{e.g.}, human-defined update rules or frontier reasoning models.
Since the environment dynamics and action sequences are identical across conditions, the difference can be attributed solely to the belief update mechanism.

\textbf{Observation 2: Weak belief tracking masks the contribution of informative actions.}
As shown in Fig.~\ref{fig:3.1.3} and \ref{fig:3.1.4}, the correlation between AS and reward is substantially higher under strong BT, but remains weak when using the agent’s own BT.
This indicates that the contribution of AS to the reward is \emph{masked} when belief updates are unreliable: even high-information actions yield little reward improvement if their information is not incorporated into internal belief.
As a result, policy optimization cannot yield stable learning signals to 
reinforce informative AS choices.

\textbf{Observation 3: Conservative action selection limits belief refinement and induces interaction-insensitive shortcuts.}
Complementary to Obs.~2, we now examine the reverse direction of the coupling.
When AS gets conservative and yields little informative evidence, BT is deprived of meaningful signals to learn from.
Under outcome-only supervision, this even incentivizes shortcut behaviors that reduce reliance on interaction, reinforcing a low-information training regime.
We observe that as training progresses, agents become less sensitive to informative observations and increasingly rely on early-stage context.
In MediQ, we intervene by replacing all patient feedback with \texttt{Unknown} while keeping all other configurations unchanged.
Notably, the induced performance drop becomes smaller after RL training ($41.25\!\rightarrow\!30.50$ w/o RL versus $61.00\!\rightarrow\!55.50$ with RL; see Fig.~\ref{fig:4.2.1}), suggesting that interaction-derived evidence has a weaker causal effect on the final decision.
Crucially, this reduced sensitivity is accompanied by an  increase in \emph{belief consistency} (Fig.~\ref{fig:4.2.1}; $78.7$ w/o RL versus $92.8$ with RL): the agent increasingly adheres to its initial judgment instead of revising beliefs in response to interaction, which reflects a more “stubborn” belief update pattern.
Together, these form \textbf{\emph{interaction under-utilization}}: once conservative AS restricts information exposure and weak BT struggles to internalize evidence, RL pressure favors non-interactive heuristics that stabilize outcomes while further suppressing information exploration and evidence usage.

Taken together, these observations indicate SeL emerges from a bidirectional coupling between AS and BT.
The reward-relevant value of AS is mediated by the agent’s ability to absorb information through BT, while BT is in turn constrained by the information budget induced by AS.
This mutual dependence can trap training dynamics in a low-information regime, giving rise to a self-locking behavior.

\section{Understanding Self-Locking}
\label{sec:theory}

To formally understand the \sel behaviors of AS and BT, we present a theoretical framework for \sel. Due to space constraints, we defer the full details to Appendix~\ref{sec:isl_theory_u_gated}.

\textbf{AS and BT capabilities.} 
We consider an agent represented by the interleaved behavior
$\Pi_\omega=(\pi_\omega^{\mathrm{as}},\pi_\omega^{\mathrm{bt}})$, where
$\pi_\omega^{\mathrm{as}}$ selects environment-facing actions and
$\pi_\omega^{\mathrm{bt}}$ updates the model belief.
Let $\tau=(b_0,a_0,o_0,\ldots,b_H)\sim\Pi_\omega$ denote the on-policy trajectory, where $b_t$ is the model belief.
To isolate AS from the model's BT mechanism, we also consider an oracle-belief trajectory
$\bar\tau=(\bar b_0,\bar a_0,\bar o_0,\ldots,\bar b_H)\sim\Pi_\omega^{\mathrm{orc}}$,
where barred variables denote \textbf{Bayesian belief updates} using the \textbf{same} action-selection kernel, and $\Pi_\omega^{\mathrm{orc}}:=(\pi_\omega^{\mathrm{as}},\mathsf{BayesUpd})$.
We quantify belief quality by the potential $\Psi(b):=b(s^\star)\in[0,1]$, which measures the confidence mass assigned to the true latent state $s^\star$.
Then, we quantify AS and BT through the following two capability indices.

\begin{definition}[AS Informativeness]
\label{def:as_information_supply}
Given an oracle-belief trajectory $\bar\tau\sim\Pi_\omega^{\mathrm{orc}}$ induced by the action-selection kernel $\pi_\omega^{\mathrm{as}}$ and Bayesian belief updates $\mathsf{BayesUpd}$, the {AS informativeness} is defined as the expected total improvement in oracle belief quality,
$
I_{\mathrm{AS}}(\omega)
\;:=\;
\mathbb E_{\bar\tau\sim\Pi_\omega^{\mathrm{orc}}}
\!\left[\Psi(\bar b_H)-\Psi(\bar b_0)\right].
$
\end{definition}

We next characterize how much of the theoretically supplied information is actually absorbed
by the agent's belief-tracking dynamics.

\begin{definition}[BT Capability]
\label{def:bt_progress_drift}
For an on-policy trajectory $\tau\sim\Pi_\omega$, define
$\Delta\Psi_t:=\Psi(b_{t+1})-\Psi(b_t)$ as the one-step model belief progress at turn $t$, and $\Delta\Psi_t^+:=\max(\Delta\Psi_t,0)$ as the \textbf{absorbed} belief progress.
The belief-tracking capability is defined as
$
C_{\mathrm{BT}}(\omega)
:= \mathbb E_{\tau\sim\Pi_\omega}\!\big[\sum_{t=0}^{H-1} \Delta\Psi_t^+\big].
$
\end{definition}

\textbf{Self-Locking regime.}
We formalize the notion of self-locking via a two-dimensional low-AS and low-BT region:
\begin{definition}[Self-Locking Regime]
\label{def:locking_regime}
With $\delta,\varepsilon>0$, we define the {locking regime} as the subset of parameter space
with low AS informativeness and low BT capability:
\[
\setlength{\abovedisplayskip}{3pt}
\setlength{\belowdisplayskip}{3pt}
\setlength{\abovedisplayshortskip}{3pt}
\setlength{\belowdisplayshortskip}{3pt}
\mathcal R_{\delta,\varepsilon}
\;:=\;
\big\{\omega\in\Omega:\ I_{\mathrm{AS}}(\omega)\le \delta,\ \ C_{\mathrm{BT}}(\omega)\le \varepsilon\big\},
\]
which represents a low-information and low-BT regime.

\end{definition}

\textbf{Policy-gradient decomposition of the outcome reward.}
Given the agent $\Pi_\omega=(\pi_\omega^{\mathrm{as}},\pi_\omega^{\mathrm{bt}})$ trained with outcome-based reward, the policy gradient can be written as
\[
\setlength{\abovedisplayskip}{3pt}
\setlength{\belowdisplayskip}{3pt}
\setlength{\abovedisplayshortskip}{3pt}
\setlength{\belowdisplayshortskip}{3pt}
\nabla_\omega J(\omega)
=
\mathbb E_{\tau\sim\Pi_\omega}\!\big[ R(\tau)\,\nabla_\omega \log p_\omega(\tau)\big].
\]

Then we expand $\nabla_\omega \log p_\omega(\tau)$ as follows:
{
\setlength{\abovedisplayskip}{3pt}
\setlength{\belowdisplayskip}{3pt}
\setlength{\abovedisplayshortskip}{3pt}
\setlength{\belowdisplayshortskip}{3pt}
\begin{equation*}
\sum_{t=0}^{H}\nabla_\omega\log \pi_\omega^{\mathrm{bt}}(b_t\mid c_t)
+
\sum_{t=0}^{H-1}\nabla_\omega\log \pi_\omega^{\mathrm{as}}(a_t\mid b_t).
\end{equation*}}

Here $c_t$ denotes the BT-conditioning context.
We analyze the \textit{\textbf{normalized}} AS- and BT-channel update directions:
\[
\setlength{\abovedisplayskip}{3pt}
\setlength{\belowdisplayskip}{3pt}
\setlength{\abovedisplayshortskip}{3pt}
\setlength{\belowdisplayshortskip}{3pt}
    g_{\mathrm{as}}(\omega)\!=\!
\mathbb E_{\omega}\!\left[
\frac{1}{H}\sum_{t=0}^{H-1}
\nabla_\omega\log \pi_\omega^{\mathrm{as}}(a_t\mid b_t)
A^{\mathrm{as}}_t(b_t,a_t)
\right],
\]
\[
\setlength{\abovedisplayskip}{3pt}
\setlength{\belowdisplayskip}{3pt}
\setlength{\abovedisplayshortskip}{3pt}
\setlength{\belowdisplayshortskip}{3pt}
g_{\mathrm{bt}}(\omega)\!=\!\mathbb E_{\omega}\!\left[
\frac{1}{H+1}\sum_{t=0}^{H}
\nabla_\omega\log\pi_\omega^{\mathrm{bt}}(b_t\mid c_t)
A^{\mathrm{bt}}_t(c_t,b_t)
\right],
\]
where $A^{\mathrm{as}}$ and $A^{\mathrm{bt}}$ are the channel-isolated advantages of AS and BT, respectively.
Naturally, we define the AS-channel projected update as
$
\mathcal T_{\mathrm{as}}(\omega) := \omega + \eta g_{\mathrm{as}}(\omega),
$
and the BT-channel  update as
$
\mathcal T_{\mathrm{bt}}(\omega) := \omega + \eta g_{\mathrm{bt}}(\omega).
$
The coupling effects between AS and BT can be further characterized by the one-sided projected drifts:
\[
\setlength{\abovedisplayskip}{3pt}
\setlength{\belowdisplayskip}{3pt}
\setlength{\abovedisplayshortskip}{3pt}
\setlength{\belowdisplayshortskip}{3pt}
\Delta_{\mathrm{as}}^+ I_{\mathrm{AS}}(\omega)
:=
\big(I_{\mathrm{AS}}(\mathcal T_{\mathrm{as}}(\omega))-I_{\mathrm{AS}}(\omega)\big)_+,
\]
\[
\setlength{\abovedisplayskip}{3pt}
\setlength{\belowdisplayskip}{3pt}
\setlength{\abovedisplayshortskip}{3pt}
\setlength{\belowdisplayshortskip}{3pt}
\Delta_{\mathrm{bt}}^+ C_{\mathrm{BT}}(\omega)
:=
\big(C_{\mathrm{BT}}(\mathcal T_{\mathrm{bt}}(\omega))-C_{\mathrm{BT}}(\omega)\big)_+.
\]

With these quantities, we draw the following result:

\begin{theorem}[Informal]
\label{thm:sel_lock_informal}
Fix $\delta,\varepsilon>0$.
Under the regularity assumptions detailed in Appendix~\ref{sec:self_locking_main}, including Lipschitz reward-belief coupling, bounded gradients, action-invariant harmful belief drift, BT-capability-limited absorption, residual AS-budget comparability, and conservative propagation of belief quality, for any $\omega\in\mathcal R_{\delta,\varepsilon}$,
the one-sided projected drifts satisfy the following inequality:
\[
\begin{pmatrix}
\Delta_{\mathrm{as}}^+ I_{\mathrm{AS}}(\omega)\\[2pt]
\Delta_{\mathrm{bt}}^+ C_{\mathrm{BT}}(\omega)
\end{pmatrix}
\ \preceq\
\eta
\begin{pmatrix}
0 & \alpha\\
\beta_I & \beta_C
\end{pmatrix}
\begin{pmatrix}
I_{\mathrm{AS}}(\omega)\\[2pt]
C_{\mathrm{BT}}(\omega)
\end{pmatrix}
\ +\ o(\eta),
\]
where $\preceq$ denotes elementwise inequality and $\alpha,\beta_I,\beta_C>0$ are problem-dependent constants.
Moreover, for agents initialized inside $\mathcal R_{\delta,\varepsilon}$, the agent cannot leave the locking regime for a non-trivial number of policy-update steps, with the explicit escape-time lower bound given in Appendix~\ref{sec:self_locking_main}.
\end{theorem}

The formal version of Thm.~\ref{thm:sel_lock_informal} along with the proof are given in Appendix~\ref{sec:self_locking_main}.
Intuitively, Thm.~\ref{thm:sel_lock_informal} shows that, when under the \sel regime, the learning signals from outcome reward are weakened by the limited AS and BT capabilities, scaling linearly with the current levels of $I_{\mathrm{AS}}$ and $C_{\mathrm{BT}}$.
Consequently, when the model is initialized within the \sel regime, it requires significant policy update steps to escape the \sel regime.
In practice, this indicates that once training enters the \sel regime, it is unlikely to recover without explicit interventions that restore informative local credit signals for AS and BT.
\section{Breaking Self-Locking with Directions}
\label{sec:method}
The previous section shows that \infoSL arises from entangled and weakened credit assignment between AS and BT under outcome-based RL.
While dense calibrated supervision for intermediate decisions is difficult in long-horizon agentic tasks, many active-reasoning environments provide easy-to-obtain \emph{\textbf{directional}} signals at the step level; for example, an action's informativeness can often be inferred from whether the environment reveals new evidence in response.
Motivated by this observation, we propose \proj, a lightweight framework that converts such uncalibrated directional signals into policy-gradient credit reallocation, improving AS and BT without introducing calibrated intermediate rewards.

\subsection{Stepwise directional critiques}
\label{sec:method_channel}

{The AS/BT decomposition of agentic behavior allows us to exploit respective directional critiques. 
For analysis and implementation, the reasoning process of the agent can be organized as alternating between (i) an \emph{Action Round}, in which the agent outputs an environment-facing action, and (ii) an \emph{Update Round}, in which the agent receives the new observation and updates its internal belief.}

\textbf{AS directional critique.} For the AS channel, we assign a directional critique $z_t^{\mathrm{as}}\in\{-1,0,+1\}$ to each executed action, where $+1$ indicates that the action elicits informative feedback or useful evidence from the environment or user, $-1$ indicates an uninformative action, and $0$ denotes an abstention.
Intuitively, $z_t^{\mathrm{as}}$ encourages  agents to strategically take actions that induce information helpful to reasoning.

\textbf{BT directional critique.} For the BT channel, the critique should reflect whether newly acquired information is effectively
incorporated into the agent's internal belief state.
While we cannot directly access the agent's belief state, we can still acquire a scalar readout $\widehat\Psi_t\in[0,1]$ that tracks task-relevant confidence over turns from the agent, such as through prompting.
{Importantly, $\widehat\Psi_t$ is used purely as instrumentation: it is neither assumed to coincide with, nor to recover, the latent analytical belief $b_t$.}
We  then define
$
z_t^{\mathrm{bt}}
:=
\mathrm{Sign}\!\left(\widehat\Psi_{t+1}-\widehat\Psi_t\right),
$
where positive values indicate that the readout moves in a truth-aligned direction.

\subsection{Injecting directional critiques into policy-gradient}
\label{sec:critique_injection}

\textbf{Margin-aware auxiliary objective.} We inject  directional critiques via an auxiliary objective that
(i) acts locally at the critiqued steps and (ii) induces a gradient that can be combined with standard
policy gradients without modifying  task rewards.
To this end, for a trajectory $\tau$ with labels $\{z_t\}$, define the positively and negatively
critiqued index sets
$
\mathcal P_\tau:=\{t\mid z_t=+1\}$,
and
$
\mathcal N_\tau:=\{t\mid z_t=-1\}
$. 
{Note that we use $z_t$ to denote  $z_t^{\mathrm{as}}$, $z_t^{\mathrm{bt}}$, or both. The construction is channel-agnostic}.
Whenever counts $|\mathcal P_\tau|>0$ and $|\mathcal N_\tau|>0$, we define an \textbf{\emph{intra-trajectory likelihood-margin}} objective
{
\setlength{\abovedisplayskip}{3pt}
\setlength{\belowdisplayskip}{3pt}
\setlength{\abovedisplayshortskip}{3pt}
\setlength{\belowdisplayshortskip}{3pt}
\begin{equation}
\label{eq:gap_objective}
\widehat{\mathcal L}(\omega;\tau)
:=
\frac{1}{|\mathcal P_\tau|}\sum_{t\in\mathcal P_\tau}\log \pi_{\omega,t}
-
\frac{1}{|\mathcal N_\tau|}\sum_{t\in\mathcal N_\tau}\log \pi_{\omega,t}.
\end{equation}}%
Here $\log \pi_{\omega,t}$ denotes the aggregate log-probability assigned by the agent to the {decision segment}
taken at step $t$ under parameters $\omega$.
Note that the ``decision'' can correspond to an action-selection
decision or a belief-update decision, and the construction here thus applies to both AS and BT channels.

Eq.~\ref{eq:gap_objective} directly encourages the agent to increase the log-probability mass on
positively critiqued decisions \emph{relative to} negatively critiqued ones, without introducing calibrated intermediate
rewards or training a separate discriminator for $z_t$.
Crucially, Eq.~\ref{eq:gap_objective} is additive over time and therefore naturally compatible with multi-turn
credit assignment.

\textbf{Implied per-step coefficients.}
Notably, Eq.~\ref{eq:gap_objective}  has a gradient of the same form as standard
policy gradients:
{
\setlength{\abovedisplayskip}{3pt}
\setlength{\belowdisplayskip}{3pt}
\setlength{\abovedisplayshortskip}{3pt}
\setlength{\belowdisplayshortskip}{3pt}
\begin{equation*}
\nabla_\omega\,\widehat{\mathcal L}
=
\sum_{t=0}^{H-1} u_t\nabla_\omega \log \pi_{\omega,t}; \,\,\,
u_t
:=
\begin{cases}
\frac{1}{|\mathcal P_\tau|} & \text{if } z_t=+1,\\[1pt]
\frac{-1}{|\mathcal N_\tau|} & \text{if } z_t=-1,\\[1pt]
0 & \text{if } z_t=0.
\end{cases}
\end{equation*}}%
Here the sign of $u_t$ matches the critique direction, so the auxiliary gradient pushes probability
mass in the intended direction.
Furthermore, 
$\sum_{t=0}^{H-1} u_t=0$ exhibits a centering property,
which means the auxiliary term induces a \emph{pure likelihood margin} rather than a uniform likelihood shift.

\textbf{Minimal injection via advantage reweighting.} 
Let ${\mathcal J}_{\mathrm{RL}}$ denote the standard actor surrogate used by a
policy-gradient RL algorithm.
We consider the augmented surrogate
{
\setlength{\abovedisplayskip}{3pt}
\setlength{\belowdisplayskip}{3pt}
\setlength{\abovedisplayshortskip}{3pt}
\setlength{\belowdisplayshortskip}{3pt}
\begin{equation}
\label{eq:augmented_surrogate}
\widehat{\mathcal L}_{\mathrm{aug}}(\omega)
\;:=\;
{\mathcal J}_{\mathrm{RL}}(\omega)
\;+\;
\lambda\,\mathbb E_{\tau}\!\left[\widehat{\mathcal L}(\omega;\tau)\right],
\end{equation}}%
where $\lambda> 0$ controls the strength of critique injection, and the expectation is taken over on-policy trajectories.
The resulting gradient update can thus be written as
{
\setlength{\abovedisplayskip}{3pt}
\setlength{\belowdisplayskip}{3pt}
\setlength{\abovedisplayshortskip}{3pt}
\setlength{\belowdisplayshortskip}{3pt}
\begin{equation*}
\nabla_\omega\,\widehat{\mathcal L}_{\mathrm{aug}}(\omega)
\;\propto\;
\mathbb E_{\tau}\!\left[
\sum_{t=0}^{H-1}
\Big( A_t+\lambda\,u_t\Big)\,
\nabla_\omega \log \pi_{\omega,t}
\right],
\end{equation*}}%
which shows that injecting the critiques requires only a
\emph{minimal} modification to the actor update: it suffices to apply an additive shaping to the
advantage
$\widehat A_t \leftarrow  A_t+\lambda\,u_t$
while keeping the outcome reward, critic target, and the underlying RL optimization machinery unchanged.
Since the coefficients $u_t$ are derived from the likelihood-margin objective, the resulting update
can be directly interpreted as \textbf{\textit{reallocating}} policy-gradient magnitude from negatively critiqued steps to
positively critiqued ones within the same trajectory, aligned with the directional critiques.

\begin{figure*}
    \centering
    \captionsetup{aboveskip=4pt}
    \subfloat[\label{fig:m.1.1}PE-G$_{S=3}$]{\includegraphics[width=0.25\textwidth]{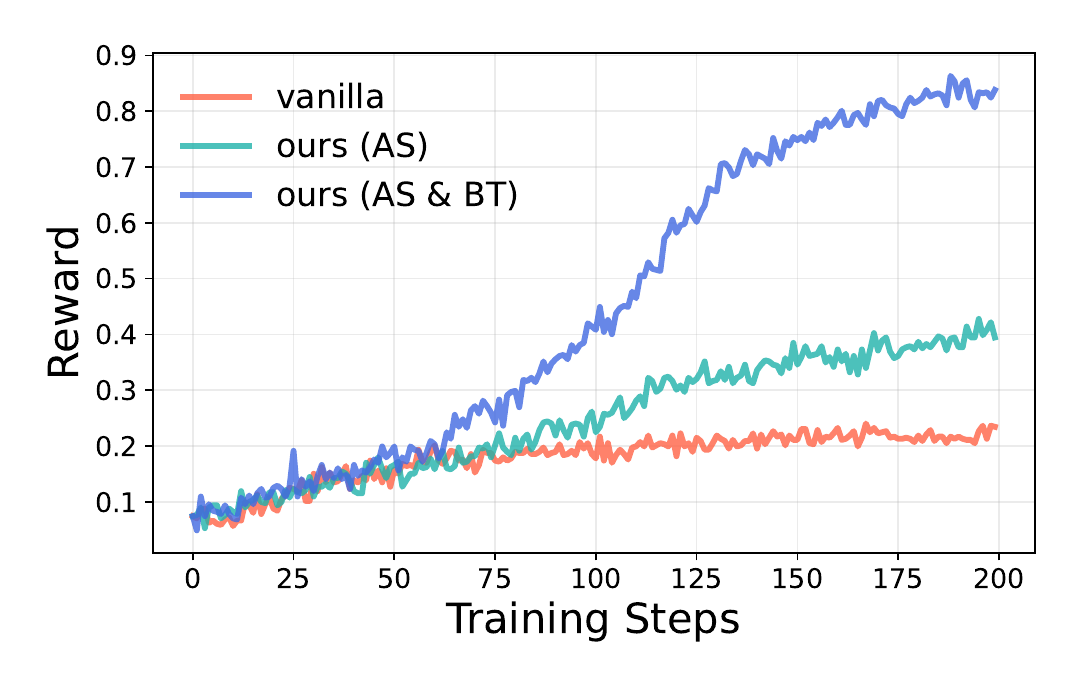}} 
    \subfloat[\label{fig:m.1.2}PE-F$_{D=8}$]{\includegraphics[width=0.248\textwidth]{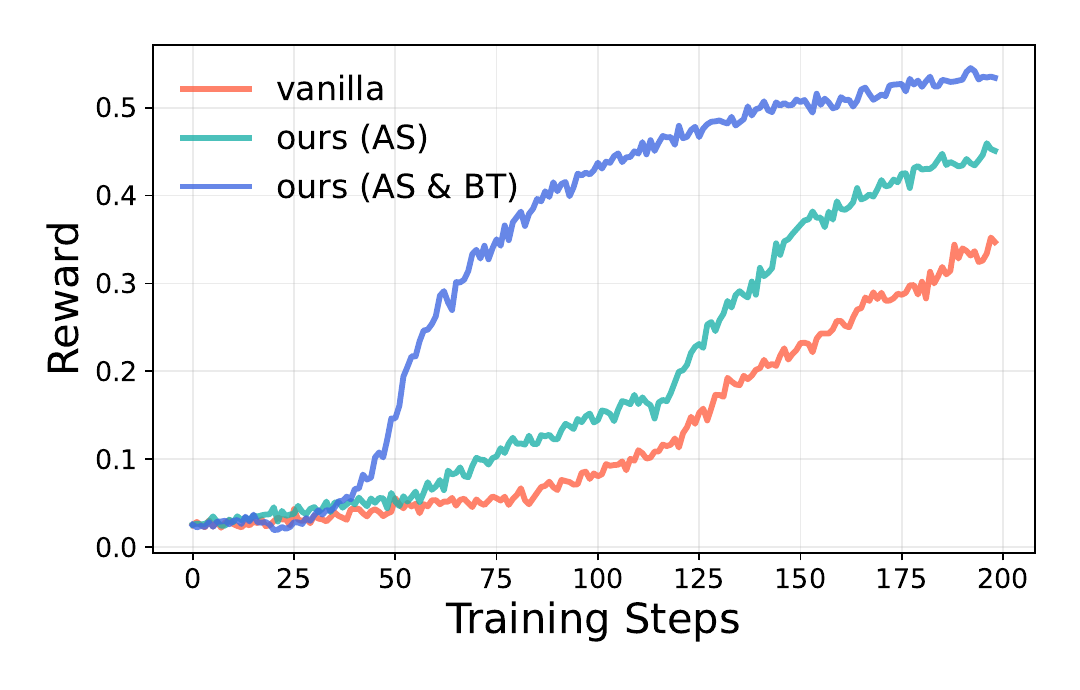}} 
    \subfloat[\label{fig:m.1.3}MediQ]{\includegraphics[width=0.248\textwidth]{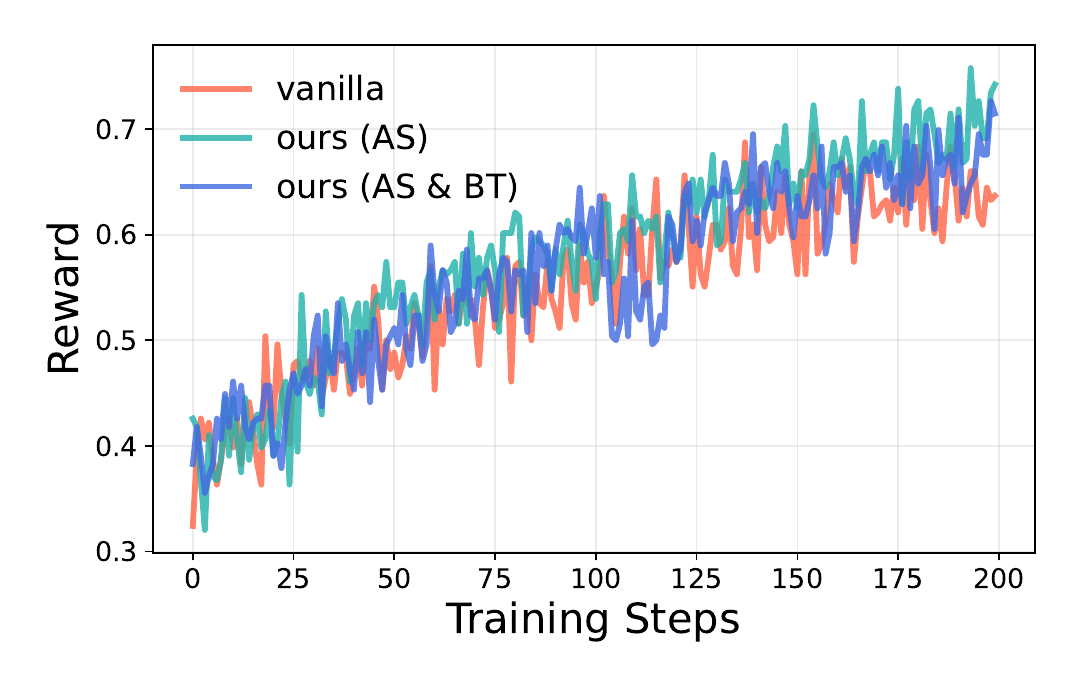}} 
    \subfloat[\label{fig:m.1.4}FloDial-Hard]{\includegraphics[width=0.248\textwidth]{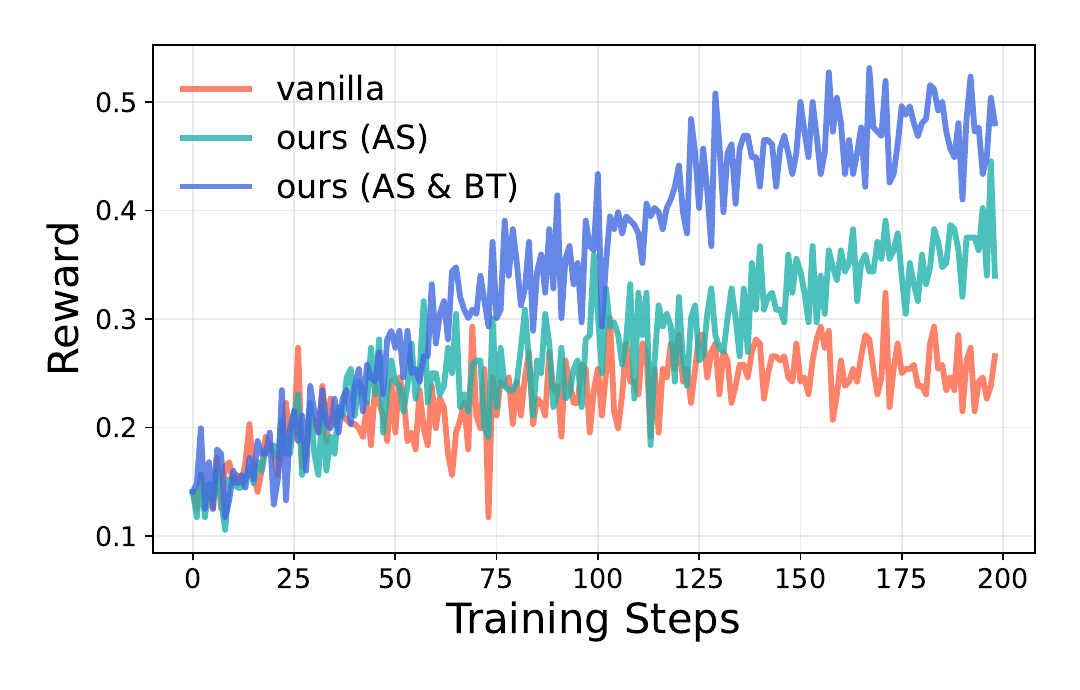}}
    \caption{Training dynamics of rewards, evaluated under the PPO algorithm with Qwen-2.5-7B-Instruct across vanilla PPO, PPO with \proj -- \textsc{as-only} and PPO with \proj -- \textsc{as+bt}.}
    \label{fig:sec3.1}
    \vspace{-.1in}
\end{figure*}
\begin{table*}[]
\caption{Main results (average outcome reward on test sets) across three domains and seven active reasoning tasks. }
\vspace{-0.05in}
\label{tab:main}
\centering
\begin{sc}
\resizebox{\textwidth}{!}{
\begin{tabular}{lccccccc}
\midrule
& \multicolumn{4}{c}{Preference Estimation} & \multicolumn{1}{l}{Medical Diagnosis} & \multicolumn{2}{c}{Troubleshooting} \\
\cmidrule(lr){2-5} \cmidrule(lr){6-6} \cmidrule(lr){7-8} \cmidrule(lr){6-6} \cmidrule(lr){7-7}
& PE-G$_{S=2}$   & PE-G$_{S=3}$     & PE-F$_{D=8}$     & PE-F$_{D=6}$     & MediQ                                 & FloDial-Easy     & FloDial-Hard      \\ \midrule
\multicolumn{8}{l}{\textbf{Direct Inference}}                                                                                                                                               \\
o4-mini                         & 17.11           & 21.15       & 8.42       & 12.47       & 74.67                                     & 35.00                & 26.33                      \\
Qwen-2.5-7B-Inst.               & 15.00          & 10.00      & 2.67    & 4.20     & 39.00                                    & 24.33            & 13.33                      \\
Llama-3.1-8B-Inst.              & 18.00          & 12.33   & 3.14    & 5.70     & 35.25                                 & 24.33            & 17.33                       \\ \midrule
\multicolumn{8}{l}{\textbf{PPO-trained (Qwen-2.5-7B-Inst.)}}                                                                                                                                \\
Vanilla                         & 24.00          & 18.33   & 30.52   & 32.03   & 50.50                                  & 37.33            & 21.33                      \\
\proj -- AS only 
& 46.00 {\oo {\scriptsize $\uparrow$ 22.0}} 
& 32.00 {\oo {\scriptsize $\uparrow$ 13.7}} 
& 39.62 {\oo {\scriptsize $\uparrow$ 9.1}} 
& 42.10 {\oo {\scriptsize $\uparrow$ 10.1}} 
& 57.25 {\oo {\scriptsize $\uparrow$ 6.8}} 
& 43.67 {\oo {\scriptsize $\uparrow$ 6.3}} 
& 36.00 {\oo {\scriptsize $\uparrow$ 14.7}}                      \\
\proj -- AS + BT 
& 49.33 {\oo {\scriptsize $\uparrow$ 25.3}} 
& 80.33 {\oo {\scriptsize $\uparrow$ 62.0}} 
& 47.89 {\oo {\scriptsize $\uparrow$ 17.4}} 
& 44.47 {\oo {\scriptsize $\uparrow$ 12.4}} 
& 61.25 {\oo {\scriptsize $\uparrow$ 10.8}} 
& 41.00 {\oo {\scriptsize $\uparrow$ 3.7}} 
& 42.33 {\oo {\scriptsize $\uparrow$ 21.0}}                      \\ \midrule
\multicolumn{8}{l}{\textbf{PPO-trained (Llama-3.1-8B-Inst.)}}                                                                                                                               \\
Vanilla                         & 27.33       & 11.00      & 55.21   & 6.00       & 63.50                                  & 24.33            & 31.00                      \\
\proj -- AS only 
& 49.00 {\oo {\scriptsize $\uparrow$ 21.7}} 
& 73.00 {\oo {\scriptsize $\uparrow$ 62.0}} 
& 55.61 {\og {\scriptsize $\uparrow$ 0.4}} 
& 56.91 {\oo {\scriptsize $\uparrow$ 50.9}} 
& 71.75 {\oo {\scriptsize $\uparrow$ 8.3}} 
& 41.67 {\oo {\scriptsize $\uparrow$ 17.3}} 
& 42.00 {\oo {\scriptsize $\uparrow$ 11.0}}                      \\
\proj -- AS + BT 
& 54.67 {\oo {\scriptsize $\uparrow$ 27.3}} 
& 77.67 {\oo {\scriptsize $\uparrow$ 66.7}} 
& 61.28 {\oo {\scriptsize $\uparrow$ 6.1}} 
& 54.65 {\oo {\scriptsize $\uparrow$ 48.7}} 
& 70.75 {\oo {\scriptsize $\uparrow$ 7.3}} 
& 44.33 {\oo {\scriptsize $\uparrow$ 20.0}} 
& 49.00 {\oo {\scriptsize $\uparrow$ 18.0}}                      \\ \midrule
\end{tabular}
}
\end{sc}
\vspace{-.19in}
\end{table*}

\subsection{Theoretical and Practical Discussion}
\label{sec:why_repairs_self_locking}

In contrast to the theoretical characterization of \infoSL in Thm.~\ref{thm:sel_lock_informal}, the directional critique breaks the negative confounding of limited AS and BT.
Let $\widehat{\mathcal T}_{\mathrm{as}}(\omega)$ and $\widehat{\mathcal T}_{\mathrm{bt}}(\omega)$ denote
the corresponding critique-shaped updates induced by the likelihood-margin objective in Eq.~\ref{eq:augmented_surrogate}, respectively.
The following proposition demonstrates that the effectiveness of \proj is characterized by the ``weighted accuracy'' of $z_t$ critiques. See the complete formulation in Appendix~\ref{app:breaks}.

\begin{proposition}
\label{prop:weighted_acc_AS_main}
Under the setting of Sec.~\ref{sec:critique_injection} and Appendix~\ref{app:breaks}, denote the stepwise weight as
$
w_t(\omega)
:=
|u_t|\,
\big|\bar A_t(\bar b_t,\bar a_t)\big|\,
\|\nabla_\omega \log \pi_\omega^{\mathrm{as}}(\bar a_t\mid \bar b_t)\|^2
$
and let
$
W(\omega)
:=
\mathbb E\!\left[\sum_{t=0}^{H-1} w_t(\omega)\right].
$
Then
the critique quality is measured by the {weighted accuracy}
\[
\setlength{\abovedisplayskip}{3pt}
\setlength{\belowdisplayskip}{3pt}
\setlength{\abovedisplayshortskip}{3pt}
\setlength{\belowdisplayshortskip}{3pt}
\mathrm{Acc}_{\mathrm{as}}(\omega)
\;:=\;
\frac{
\mathbb E\!\left[\sum_{t=0}^{H-1} w_t(\omega)\,\mathbf 1\{z_t=y_t\}\right]
}{
\mathbb E\!\left[\sum_{t=0}^{H-1} w_t(\omega)\right]
}
\in[0,1].
\]
Moreover, the first-order improvement in AS informativeness induced by \proj\ satisfies (and the BT-side analogue)
\[
\setlength{\abovedisplayskip}{3pt}
\setlength{\belowdisplayskip}{3pt}
\setlength{\abovedisplayshortskip}{3pt}
\setlength{\belowdisplayshortskip}{3pt}
I_{\mathrm{AS}}\!\big(\widehat{\mathcal T}_{\mathrm{as}}(\omega)\big)
-
I_{\mathrm{AS}}\!\big(\mathcal T_{\mathrm{as}}(\omega)\big)
=
\eta\,W(\omega)\,\Big(2\,\mathrm{Acc}_{\mathrm{as}}(\omega)-1\Big).
\]
\end{proposition}
In particular,
\proj does not rely on perfectly accurate critiques.
The proposition shows that 
\proj\ is effective whenever
$\mathrm{Acc}_{\mathrm{as}}(\omega)>\tfrac12$; this only requires directional critiques to be better than random in weighted accuracy, not calibrated step rewards. We empirically show that \proj is robust to critique noise (see Robustness Analyses in Sec.~\ref{sec:results}).

{\textbf{Practical discussion.}
In practice, our framework naturally maps to LLM agents with \emph{alternating} action-selection (AS) and belief-tracking (BT) rounds during interaction.
In each AS round, the agent takes an environment-facing action intended to reduce the task uncertainty.
In the subsequent BT round, the agent is instructed to explicitly update its confidence over the task states using the latest feedback, where the confidence assigned to the ground-truth candidate $s^\star$ induces $\widehat{\Psi}_t$ when training labels provide $s^\star$.
In practice, 
simple critiques can assist to escape SeL and improve performance.
For AS, action-level signals $z_t^{\mathrm{as}}$ can be directly inferred from user or environment feedback.
For BT, stepwise critique labels $z_t^{\mathrm{bt}}$ can be constructed by comparing changes in candidate confidence (particularly that of $s^\star$) upon receiving feedback.
We defer concrete prompt templates and implementation details to Appendix~\ref{app:impl}.}

\begin{figure*}
    \centering
    \captionsetup{aboveskip=4pt}
    \subfloat[\label{fig:4.1.1}PE-G$_{S=3}$]{\includegraphics[width=0.25\textwidth]{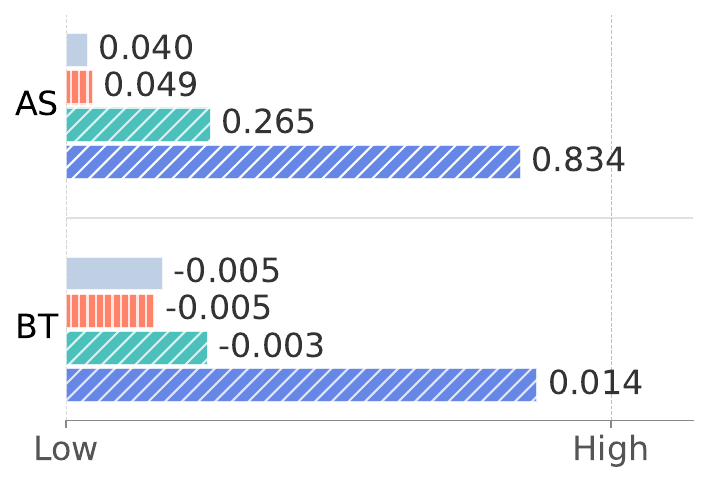}} 
    \subfloat[\label{fig:4.1.2}PE-F$_{D=8}$]{\includegraphics[width=0.248\textwidth]{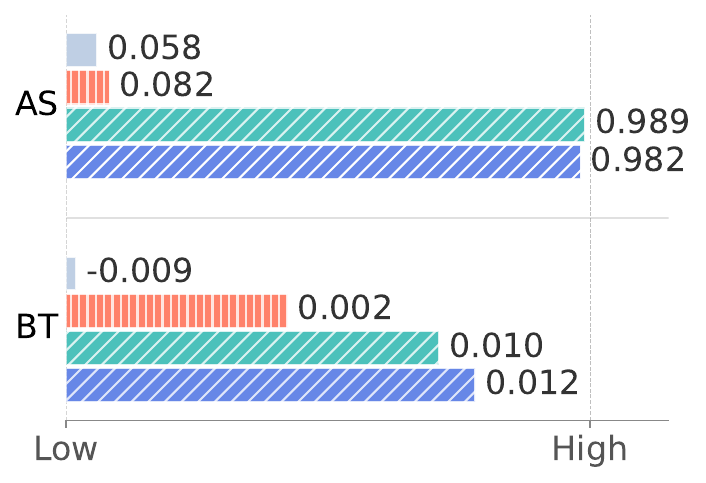}} 
    \subfloat[\label{fig:4.1.3}MediQ]{\includegraphics[width=0.248\textwidth]{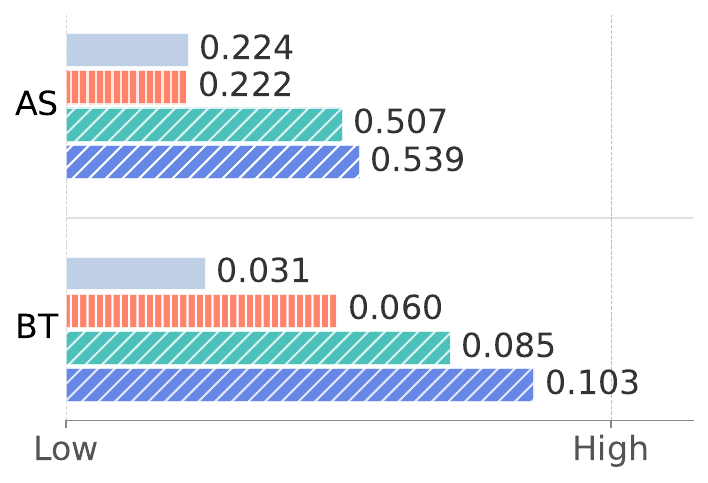}} 
    \subfloat[\label{fig:4.1.4}FloDial-Hard]{\includegraphics[width=0.248\textwidth]{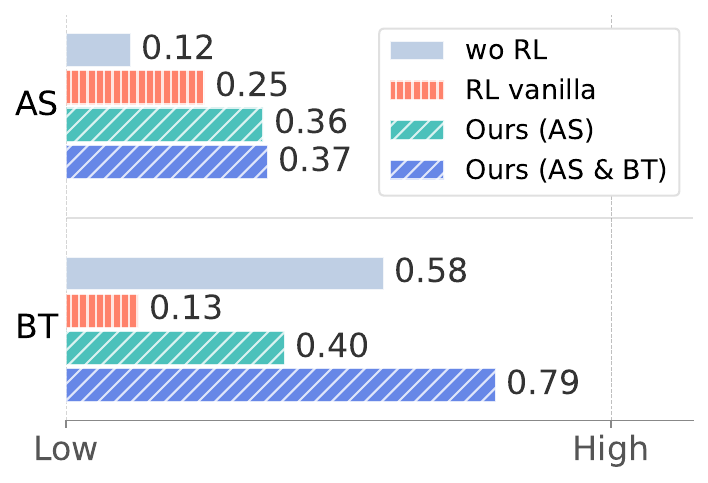}}
    \caption{Evaluations of AS and BT capabilities under PPO algorithm with Qwen-2.5-7B-Instruct across vanilla PPO, PPO with \proj.}
    \label{fig:sec4.1}
    \vspace{-0.1in}
\end{figure*}
\begin{figure*}
    \centering
    \captionsetup{aboveskip=4pt}
    \subfloat[\label{fig:a.1}Rewards]{\includegraphics[width=0.25\textwidth,trim=0 20 0 0,clip]{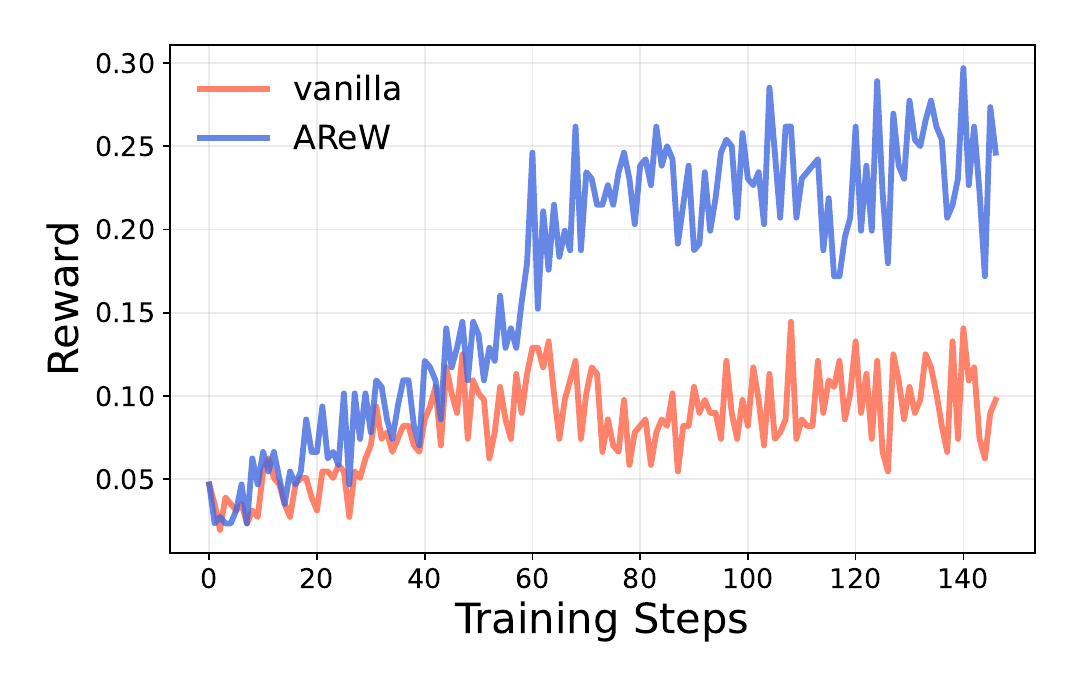}} 
    \subfloat[\label{fig:a.2}Tool Execution Errors]{\includegraphics[width=0.25\textwidth,trim=0 20 0 0,clip]{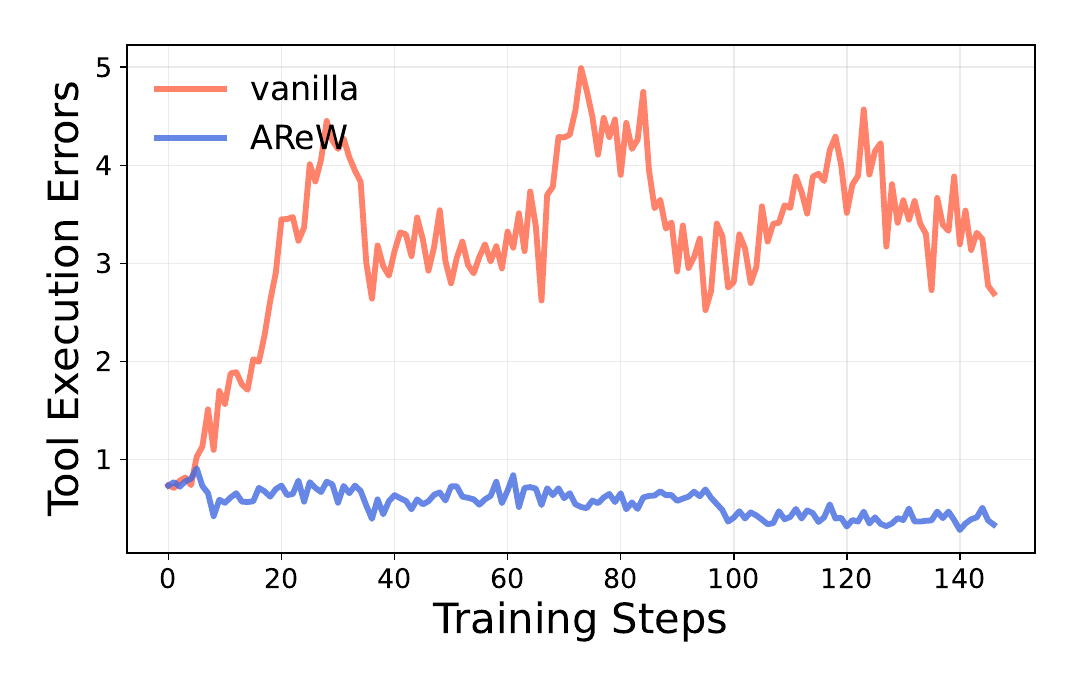}} 
    \subfloat[\label{fig:a.3}Response Tokens]{\includegraphics[width=0.25\textwidth,trim=0 20 0 0,clip]{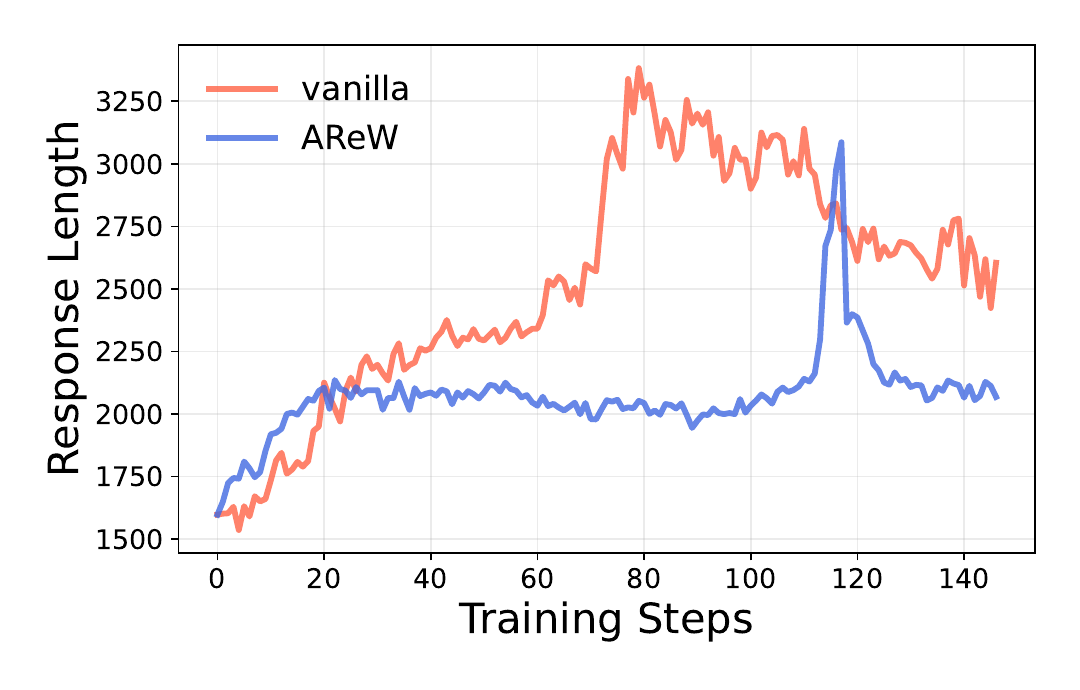}} 
    \subfloat[\label{fig:a.4}Interaction Turns]{\includegraphics[width=0.25\textwidth,trim=0 20 0 0,clip]{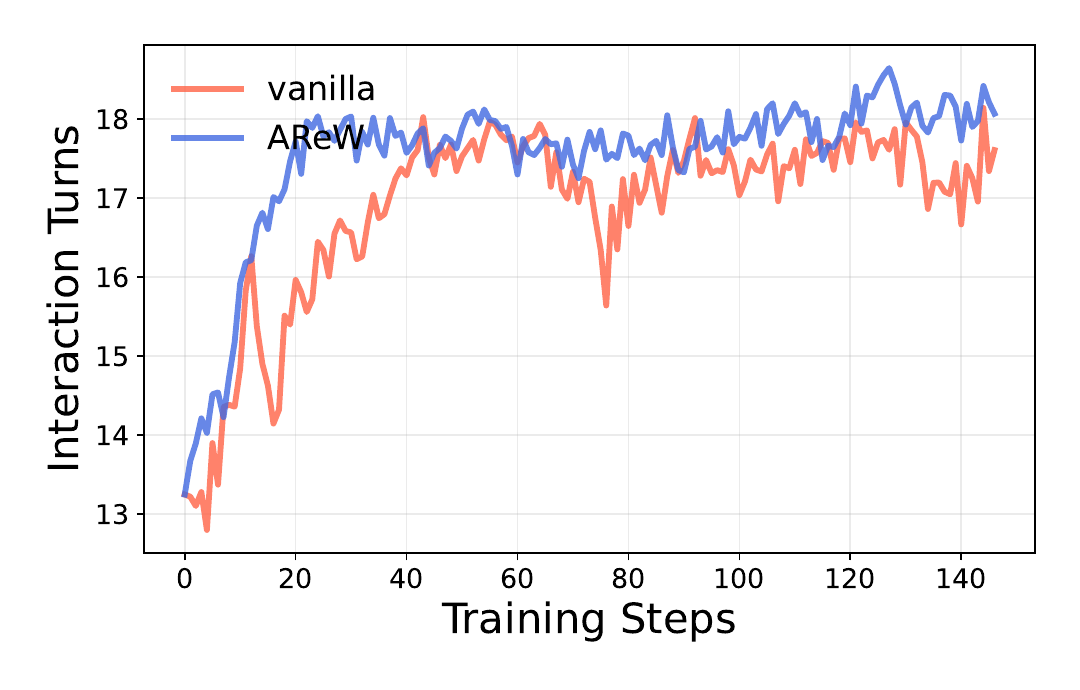}} 
    \caption{ 
     Training dynamics \textit{wrt.} training steps on (a) rewards, (b) tool execution errors, (c) consumed response tokens and (d) interaction turns  over Tau2Bench-Telecom solo setting, comparing vanilla PPO and PPO with \proj  on Qwen2.5-7B-Instruct.
    }
    \label{fig:tau2-solo}
    \vspace{-0.19in}
\end{figure*}

\section{Experiments}
\label{sec:exp}

\subsection{Experimental Setup}
\label{sec:setup}

\textbf{Datasets.}
We evaluate \proj across \textbf{\textit{4}} interactive domains: 
{preference estimation}, {medical diagnosis}, {troubleshooting}, and {customer service}.
Specifically, preference estimation includes PE-G and PE-F  (Sec.~\ref{sec:phenomenon}); medical diagnosis corresponds to MediQ (Sec.~\ref{sec:phenomenon}); and troubleshooting is instantiated by FloDial in both easy and hard modes~\citep{raghu2021end,hu2024uncertainty}, where an agent asks diagnostic questions to resolve user-reported issues.
We further evaluate \proj on the customer-service domain using $\tau^2$-bench~\citep{barres2025tau}, a practical agentic benchmark in which the agent must interact with a user and external tools to complete realistic service tasks.
Across all domains, supervision is provided only at the end of the interaction.

\textbf{Baselines.} To evaluate the effectiveness of \proj, we compare it against the following baselines: 1) Direct Inference without Training, where we evaluate representative proprietary reasoning LLMs o4-mini; 2) PPO~\citep{schulman2017proximal}, 3) GRPO~\citep{shao2024deepseekmath}, and 4) GSPO~\citep{zheng2025group}. 
See more details in Appendix~\ref{app:baselines}. 
For each RL mechanism, we consider two variants of \proj: \textsc{as-only} and \textsc{as+bt}, which work on the AS side critique $z_t^\mathrm{AS}$ and both AS and BT sides $z_t^\mathrm{AS},z_t^\mathrm{BT}$, respectively.

\textbf{Evaluation metrics.}
We report average reward as well as AS and BT capability proxies on the test datasets.
PE-F is evaluated under a continuous reward defined by normalized similarity improvement, whereas all other datasets use binary rewards.
Details of rewards and the proxies used to approximate AS and BT are provided in Appendix~\ref{app:impl}.

\textbf{Implementation details.}
The main experiments of RL training are conducted on {Qwen2.5-7B-Instruct}~\citep{yang2024qwen2} and  LLaMA-3.1-8B-Instruct~\citep{grattafiori2024llama}. For the PE-G and PE-F tasks, the interactive feedback is rule-based; for the MediQ and FloDial datasets, we leverage {Qwen2.5-14B-Instruct} to simulate the ``user'' and provides the interactive
feedback.
See more  details in Appendix~\ref{app:impl}.

\subsection{Experimental Results and Analyses}
\label{sec:results}

In this part, we present experimental results on the first three settings to evaluate the effectiveness of \proj and to analyze its impact on learning dynamics in agentic active reasoning.
We begin with overall performance comparisons, followed by examining how \proj affects reward optimization, AS and BT capabilities, and robustness analysis. See complementary experimental analyses in Appendix~\ref{app:add_exp_analyses}.

\textbf{Overall Performance.}
We first report the main results across the evaluated domains.
Table~\ref{tab:main} summarizes final-task performance.
We further visualize dynamics of  reward optimization and information-related behaviors.
Specifically, Fig.~\ref{fig:sec3.1} shows episode-level reward trajectories, while Fig.~\ref{fig:sec4.1} reports proxy measurements of  AS and BT capabilities introduced in Section~\ref{sec:phenomenon}.
As shown in Table~\ref{tab:main}, \proj consistently outperforms the vanilla PPO baseline, achieving notable improvements in 27 out of 28 evaluated settings.
Among these, \textsc{as+bt} further largely outperforms the \textsc{as-only} variants in 11 out of 14 cases.
These results indicate the effectiveness of \proj on mitigating self-locking.

\begin{figure*}
    \centering
    
    \captionsetup{aboveskip=4pt}
    \subfloat[\label{fig:4.2.1}MediQ]{\includegraphics[width=0.25\textwidth]{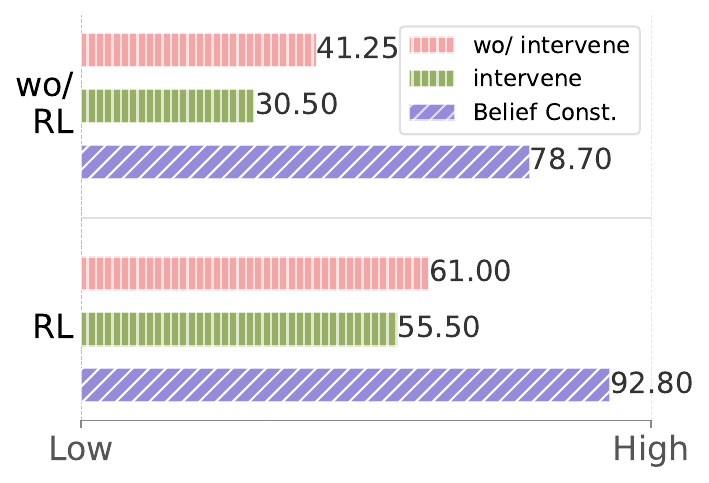}} 
    \hspace{3em}
    \subfloat[\label{fig:4.2.2}PE-F$_{D=8}$ (GRPO)]{\includegraphics[width=0.25\textwidth]{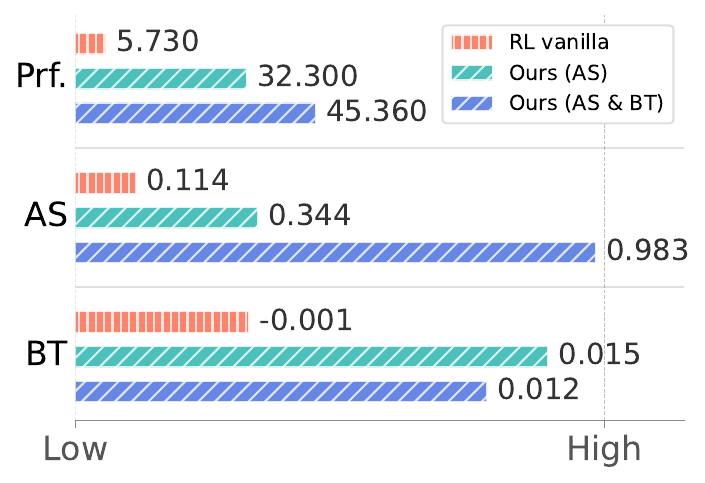}} 
    \hspace{3em}
    \subfloat[\label{fig:4.2.3}FloDial-Easy (GSPO)]{\includegraphics[width=0.25\textwidth]{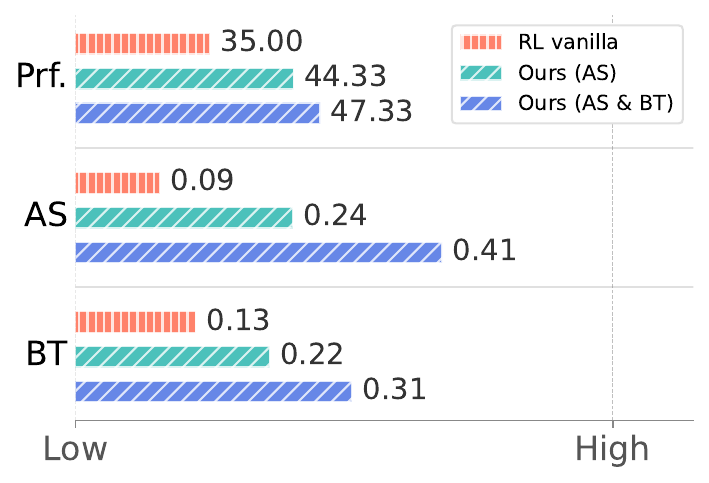}} 
    %
    %
   \caption{
    (a): outcome-RL reduces sensitivity to interactive feedback while increasing belief consistency.
    (b)/(c): Evaluations of AS and BT capabilities under GRPO and GSPO (Qwen-2.5-7B-Instruct).
    }
    \label{fig:sec4.2}
    \vspace{-0.1in}
\end{figure*}
\begin{figure*}
    \centering
    \captionsetup{aboveskip=4pt}
    \subfloat[\label{fig:tau2s-1}Rewards]{\includegraphics[width=0.25\textwidth,trim=0 20 0 0,clip]{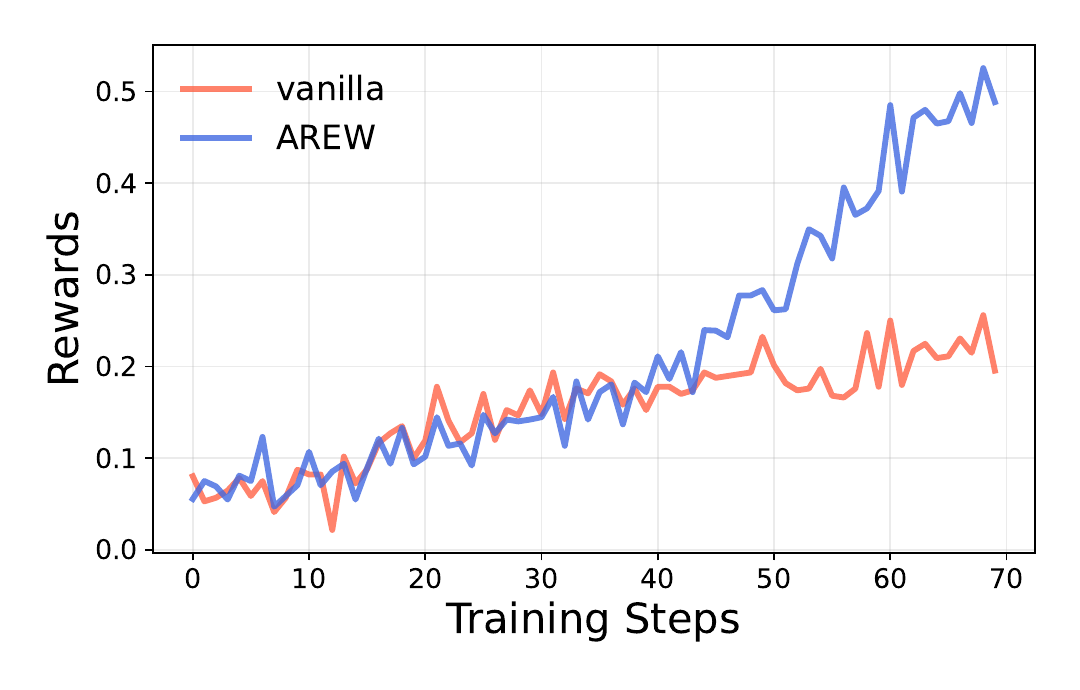}} 
    \hspace{3em}
    \subfloat[\label{fig:tau2s-2}Message Turn Ratio]{\includegraphics[width=0.25\textwidth,trim=0 20 0 0,clip]{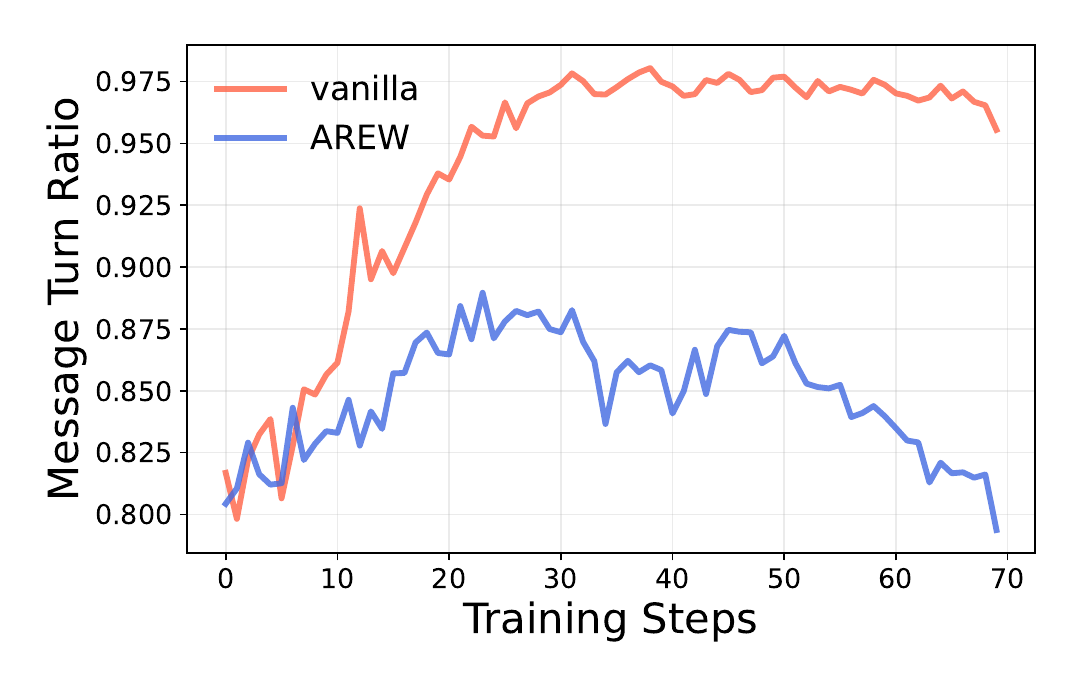}} 
    \hspace{3em}
    \subfloat[\label{fig:tau2s-3}Tool Turn Ratio]{\includegraphics[width=0.25\textwidth,trim=0 20 0 0,clip]{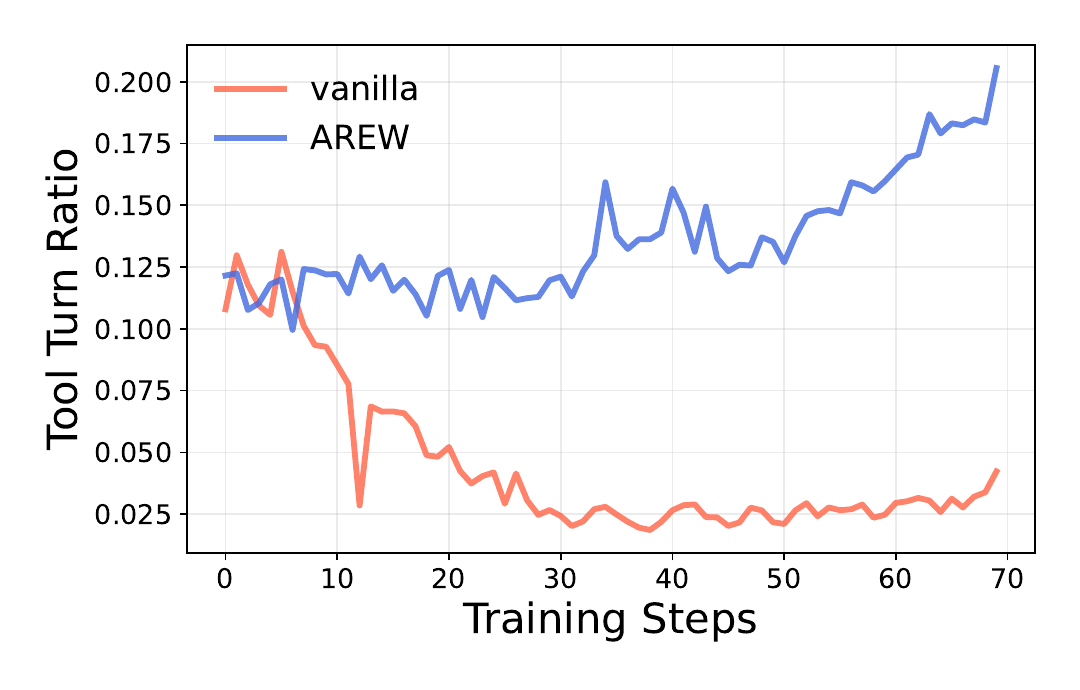}} 
    \caption{ 
     Training dynamics on Tau2Bench-Telecom under the standard dual-control setting, comparing vanilla PPO and PPO with AREW on Qwen2.5-14B-Instruct. We report (a) rewards, (b) message-turn ratio, and (c) assistant tool-turn ratio.
    }
    \label{fig:tau2s}
    \vspace{-0.19in}
\end{figure*}

\textbf{Reward dynamics.}
Fig.~\ref{fig:sec3.1} reveals three representative training behaviors.
In some tasks (Fig.~\ref{fig:m.1.1}), vanilla PPO fails to improve rewards throughout training, remaining trapped in SeL while \proj effectively breaks SeL and achieves continual improvement.
In others (Fig.~\ref{fig:m.1.2} and \ref{fig:m.1.4}), rewards can have some limited increase, but \proj\ achieves faster convergence and higher asymptotic performance.
Notably, we also observe cases where reward curves appear comparable across methods, yet \proj\ yields superior AS and BT proxy scores and better final-task performance (See Table~\ref{tab:main}), corresponding to Obs.~3 discussed in Sec.~\ref{subsec:isl_observations}.

\textbf{AS and BT behaviors.}
Fig.~\ref{fig:sec4.1} further disentangles AS and BT.
The \textsc{as-only} variant outperforms the baseline in AS proxies in all cases, and the \textsc{as+bt} variant achieves additional gains in BT over the \textsc{as-only}.
Interestingly, the pure \textsc{as-only} variant already leads to measurable improvements in BT, reflecting the intrinsic coupling between information acquisition and belief updates, illustrating that breaking SeL brings benefits for both AS and BT channels.

\textbf{Effectiveness across different RL algorithms.}
Beyond PPO, we additionally consider group-based RL variants, including GRPO and GSPO.
While these methods sample multiple trajectories per step, we empirically observe that \emph{self-locking behaviors can still arise} in interactive settings, { suggesting that increasing rollout sampling alone may be  insufficient to resolve the underlying coupling between AS and BT.}
As shown in Fig.~\ref{fig:4.2.2} and \ref{fig:4.2.3}, 
\proj consistently improves final task performance, while simultaneously strengthening AS and BT proxies, analogous to those in PPO.
These results indicate that \proj remains effective across different RL mechanisms.

\textbf{Robustness Analysis.}
We evaluate the robustness of \proj\ by randomly flipping stepwise directional critiques with probability $\alpha$.
As shown in Table~\ref{tab:2}, increasing the perturbation level leads to a gradual reduction in final performance.
However, \proj\ consistently remains competitive with, and often outperforms, the vanilla baseline across a wide range of $\alpha$.
This trend is consistent with Prop.~\ref{prop:weighted_acc_AS_main}, which suggests that performance gains can be sustained as long as the weighted accuracy of critiques is not severely degraded.
Even under strong perturbations, \proj\ does not collapse.
These results indicate that \proj\ is robust to critique noise and can tolerate imperfect directional signals in practice. A more systematic analysis is provided in Appendix~\ref{app:add_exp_analyses}.

\begin{table}[]

\centering
\caption{Final Performance under different strength of directional critique perturbation (controlled by $\alpha$).}
\label{tab:2}
\begin{sc}
\resizebox{\columnwidth}{!}{
\begin{tabular}{lccccccc}
\midrule
                                 & \multirow{2}{*}{Vanilla} & \multicolumn{6}{c}{Perturbation ratio $\alpha$}        \\ \cmidrule(lr){3-8} 
                                 &                          & 0     & 0.1   & 0.2   & 0.3   & 0.4   & 0.5   \\ \midrule
PE-G$_{S=3}$      & 18.3                    & 80.3 & 40.0 & 65.0 & 31.3 & 22.3 & 30.3 \\
FloDial-Hard & 21.3                    & 36.0 & 30.3 & 29.0 & 27.6 & 30.6 & 23.3 \\ \midrule

\end{tabular}}
\end{sc}
\vspace{-.22in}
\end{table}

\subsection{Extension to Practical Customer-Service Agents}
\label{sec:tau2}

We further evaluate \proj on $\tau^2$-bench-Telecom~\citep{barres2025tau}, a practical customer-service benchmark that requires agents to solve telecom troubleshooting tasks through multi-turn dialogue and tool use.
Unlike the controlled domains above, $\tau^2$-bench exposes a dual-control environment: the assistant can use backend tools, while the simulated user can also execute user-side diagnostic or repair actions.
This creates an additional credit-assignment challenge: the agent may obtain partial progress either by using its own tools correctly or by guiding the user to perform actions.

\textbf{Solo setting.}
We consider the no-user {solo} setting, where the Qwen2.5-7B agent has direct control over the task-solving process.
AREW is instantiated using only benchmark-native online signals, without additional annotation, preference data, or learned reward models.
Negative critiques are derived from runtime failures such as malformed tool calls, invalid executions, and repeated actions, while positive critiques are derived from progress indicators exposed by the task evaluator, e.g., whether the trajectory newly matches an expected action.
Fig.~\ref{fig:tau2-solo} shows that AREW substantially improves reward while sharply reducing tool execution errors.
The improvement is not achieved by simply increasing interaction length or response verbosity: AREW uses fewer response tokens and maintains a comparable number of interaction turns.
This suggests that lightweight directional critiques can provide useful optimization signals even in realistic long-horizon tool-use environments.

\textbf{Standard dual-control setting.}
We study the standard interactive setting, where the Qwen2.5-14B agent must coordinate with a GPT-4o-simulated user while deciding when to communicate and when to use its own tools.
This setting reveals a practical form of self-locking.
Under vanilla PPO, early reward gains are largely associated with user-side operation: the policy increasingly guides the user to execute repair actions (Fig.~\ref{fig:tau2s-2}), while assistant-side tool use decreases (Fig.~\ref{fig:tau2s-3}).
Although this behavior can solve a subset of tasks, it creates a path dependence toward the easiest available progress channel and weakens learning of the intended assistant-side tool-use behavior.
Trajectory inspection around step~60 further supports this interpretation: successful validation cases are dominated by user-side repair, but tasks requiring assistant-side repair remain hard.

AREW changes this learning dynamic.
As shown in Fig.~\ref{fig:tau2s},
over the matched 70-step training, AREW assigns more credit to useful assistant-side tool decisions and reduces over-reliance on user-side repair.
AREW maintains substantially more assistant read/write tool turns than vanilla PPO, while requiring fewer user tool hops, raising the average reward from $0.20$ to $0.50$,
a $2.5\times$ relative increase over vanilla PPO.
Thus, in $\tau^2$-bench, AREW not only improves early reward but also counteracts an interaction-induced shortcut, reallocating optimization pressure toward the agent-side behavior that the benchmark is designed to evaluate.

\section{Conclusion}
We study \emph{information self-locking} (SeL) in long-horizon multi-turn agentic active reasoning and show that it arises from a structural failure of credit assignment with  bidirectional coupling between action selection (AS) and belief tracking (BT).
We provide both theoretical and empirical evidence that standard outcome-based RL can be trapped in SeL.
We propose \proj, a critique-driven reweighting approach that selectively reallocates optimization signal along trajectories.
Experiments demonstrate consistent gains, robustness to noisy critiques, effectiveness on multiple RL mechanisms, and improved training dynamics, AS and BT capabilities across multiple benchmarks.
We believe this perspective opens up new directions for designing robust learning mechanisms for interactive reasoning agents.

\clearpage
\section*{Acknowledgements}
We thank the reviewers for their constructive comments and suggestions. Deyu Zou, Yongqiang Chen, and James Cheng were supported by a CRF (No. C2005-24Y) from the RGC of Hong Kong. 
Mufei Li and Pan Li are in part supported by NSF grants III-2239565 and III-2428777.
This work was a collaboration between Husky Data+AI Lab at CUHK and ByteDance, supported by a ByteDance University Collaboration Project.

\section*{Impact Statement}

This paper presents work whose goal is to advance the field of Machine
Learning. There are many potential societal consequences of our work, none
which we feel must be specifically highlighted here.

\nocite{langley00}

\bibliography{example_paper}
\bibliographystyle{icml2026}

\newpage
\appendix
\onecolumn

\tableofcontents
\section{Additional Discussions}
\subsection{Connection to On-Policy Local Supervision}
\label{app:connection_on_policy_local_supervision}

Recent attention to on-policy distillation (OPD) methods for post-training revisit distillation from an on-policy perspective: the student first samples its own trajectories, and a teacher model then provides local feedback on the student-visited states \citep{gu2024minillm,li2026rethinking,song2026survey}. 
This perspective is closely related to a broader design principle for long-horizon LLM post-training: terminal outcome supervision is often too sparse to identify which local decisions should be reinforced, and useful learning signals should therefore be injected at the states and decisions actually visited by the current policy. \proj shares this broad principle with OPD, but targets a different failure mode and uses a different source of local feedback.

\textbf{A unified form of \proj and OPD.} This connection can be expressed at the policy-gradient interface. Let $h_t$ denote the on-policy history at step $t$, and let $d_t$ denote the sampled local decision segment. In standard token-level OPD, $d_t$ can be a sampled token; in our AS-BT decomposition, $d_t$ can correspond to an action-selection decision or a belief-update decision. A reverse-KL form of OPD minimizes
\[
D_{\mathrm{KL}}\!\left(
\pi_\omega(\cdot\mid h_t)\,\Vert\,\nu(\cdot\mid h_t)
\right)
=
\mathbb E_{d_t\sim\pi_\omega(\cdot\mid h_t)}
\left[
\log \pi_\omega(d_t\mid h_t)
-
\log \nu(d_t\mid h_t)
\right],
\]
where $\nu$ denotes the teacher policy. Equivalently, up to constants and score-function baselines, maximizing the negative reverse KL yields a local policy-gradient coefficient
\[
c^{\mathrm{OPD}}_t
=
\beta\log \frac{\nu(d_t\mid h_t)}{\pi_\omega(d_t\mid h_t)},
\]
which reinforces sampled decisions that are relatively more likely under the teacher than under the student, and suppresses those over-assigned by the student relative to the teacher.

\proj admits an analogous actor-side interface, but with a different local signal sources. Specifically, when injected into the actor update, \proj yields (\textit{cf.}, Sec.~\ref{sec:critique_injection})
\[
\nabla_\omega \mathcal L_{\mathrm{aug}}(\omega)
\propto
\mathbb E_{\tau\sim\pi_\omega}
\left[
\sum_t
\left(A_t+\lambda u_t\right)
\nabla_\omega \log \pi_\omega(d_t\mid h_t)
\right],
\]
where $A_t$ is the original outcome-based advantage and $\lambda u_t$ is the critique-induced local correction. Thus, both OPD and \proj can be written under a \textbf{common local-supervision form}
\[
g(c)
=
\mathbb E_{\tau\sim\pi_\omega}
\left[
\sum_t
\left(A_t+c_t\right)
\nabla_\omega \log \pi_\omega(d_t\mid h_t)
\right],
\]
with
\[
c_t =
\begin{cases}
\beta\log \frac{\nu(d_t\mid h_t)}{\pi_\omega(d_t\mid h_t)},
& \text{teacher-density feedback in OPD},\\[3pt]
\lambda u_t,
& \text{directional-critique feedback in \proj}.
\end{cases}
\]
The similarity is therefore at the level of the local policy-gradient interface: both methods modify the learning signal on on-policy decisions through a local coefficient $c_t$.

\textbf{Essential Differences.} However, this formal resemblance should not obscure the difference in their objectives and semantics. OPD is fundamentally teacher-driven distillation: its local signal measures whether the teacher assigns higher relative likelihood to the sampled token or decision than the student. It is therefore primarily designed to transfer behavior from a stronger teacher to a student while reducing the distribution mismatch of off-policy distillation. 

OPD can in principle be applied to agentic active reasoning. However, long-horizon partially observable tasks may expose an efficiency challenge for such on-policy teacher supervision. 
If the student policy is weak, under SeL, it may repeatedly sample low-information interaction patterns. 
In this case, purely on-policy teacher supervision may spend much of its budget on suboptimal histories.
Unless the teacher feedback at these histories strongly induces corrective information-seeking actions, the positive signal for escaping the low-information regime can remain sparse, potentially reducing sample efficiency.
Moreover, early interaction errors can compound over subsequent turns~\citep{zou2026reducing}: in tool-integrated reasoning, recent studies~\cite{zhong2026sod} observe that erroneous tool calls can cascade across later reasoning steps, increasing student-teacher divergence and making token-level teacher supervision less reliable.

By contrast, \proj is instead directly motivated by information self-locking in agentic active reasoning. 
Its local signal asks whether the decision contributes to the information pipeline required for solving the task (\textit{cf.}, Sec.~\ref{sec:method_channel}). 
Specifically,
on the AS side, the critique indicates whether an environment-facing action elicits useful evidence. On the BT side, the critique indicates whether newly acquired evidence is internalized in a truth-aligned direction. 
\proj is designed for the structural failure mode of active reasoning: when outcome rewards cannot disentangle AS and BT, \textbf{lightweight directional turn-level} critiques can restore non-degenerate local learning signals, which are coarser than teacher logits, but are directly aligned with the AS-BT mechanism underlying SeL, without requiring calibrated dense rewards, learned process reward models, or teacher logits.

\subsection{Differences from Direct Intermediate Reward Shaping}
\label{app:reward-shaping}

A natural alternative to \proj is to convert the same stepwise critiques into intermediate rewards. In this part, we conceptually and empirically compare two ways of leveraging the same directional critiques introduced in Sec.~\ref{sec:method_channel}: i) \proj and ii) direct intermediate reward shaping.

\textbf{Conceptual Comparison.} Let $c_t\in[-1,1]$ denote a local critique signal at decision step $t$, where positive values indicate locally desirable behavior and negative values indicate undesirable behavior. A direct intermediate-reward baseline optimizes the shaped return
\[
\setlength{\abovedisplayskip}{3pt}
\setlength{\belowdisplayskip}{3pt}
\setlength{\abovedisplayshortskip}{3pt}
\setlength{\belowdisplayshortskip}{3pt}
R_{\mathrm{int}}(\tau)
=
R(\tau)+\alpha\sum_{t=0}^{H-1} c_t ,
\]
with shaping strength $\alpha>0$. This changes the optimized objective from the original sparse task return to
\[
\setlength{\abovedisplayskip}{3pt}
\setlength{\belowdisplayskip}{3pt}
\setlength{\abovedisplayshortskip}{3pt}
\setlength{\belowdisplayshortskip}{3pt}
J_{\mathrm{int}}(\omega)
=
\mathbb E_{\tau\sim\Pi_\omega}
\left[
R(\tau)+\alpha\sum_{t=0}^{H-1}c_t
\right].
\]
Accordingly, the value function is no longer trained to predict final task success alone. It must predict the shaped
future return
\[
\setlength{\abovedisplayskip}{3pt}
\setlength{\belowdisplayskip}{3pt}
\setlength{\abovedisplayshortskip}{3pt}
\setlength{\belowdisplayshortskip}{3pt}
V_{\mathrm{int}}^\pi(s_t)
=
\mathbb E_\pi
\left[
R(\tau)+\alpha\sum_{k\ge t}c_k
\,\middle|\, s_t
\right],
\]
which requires modeling when future positive or negative critiques will occur and how they accumulate.

This is a stronger requirement than what our setting supports. The critiques used by \proj are deliberately lightweight,
binary, and potentially noisy. 
They are intended to provide weak directional information about local AS or BT behavior,
They are not
calibrated estimates of each step's additive contribution to final task utility. 
Once injected into the reward, however,  noisy or weak critiques can affect both \textbf{critic regression} and \textbf{actor optimization}:
the critic must fit a noisier shaped return, and the resulting advantages can propagate corrupted local labels backward
through the trajectory.

By contrast, \proj uses the same local signals in a reward-preserving way. It keeps the original sparse return and critic target unchanged,
and injects critiques only into the actor-side credit allocation:
\begin{equation*}
\nabla_\omega\,\widehat{\mathcal L}_{\mathrm{aug}}(\omega)
\;\propto\;
\mathbb E_{\tau}\!\left[
\sum_{t=0}^{H-1}
\Big( A_t+\lambda\,u_t\Big)\,
\nabla_\omega \log \pi_{\omega,t}
\right].
\end{equation*}
The theoretical requirement is therefore weaker. 
Direct intermediate rewards require the critique sequence to behave like a
reliable process reward.
By contrast, under the settings of Proposition~\ref{prop:weighted_acc_AS_main}, \proj only requires the signed critique to be better than random in the weighted accuracy to improve the capability beyond the baseline update  in expectation.

\textbf{Empirical Findings.} This distinction also explains the empirical gap in Appendix~\ref{app:add_exp_analyses}. Under random sign corruption,
direct intermediate reward injection forces the critic to fit a corrupted shaped return. Under structured false-positive corruption,
it can further encourage critique chasing, because the policy is directly rewarded for producing more positively labeled
steps even when those labels are weakly related to final task success. 
By contrast, \proj is less susceptible to this failure mode because
the auxiliary signal is normalized and contrastive within each trajectory: increasing the number of positive labels does not
linearly increase the total reward mass. Hence, in our regime, \textbf{weak directional step-level critiques are better used as
reward-preserving credit reallocation signals than as direct intermediate utilities}. Note that we do not claim that all reward shaping is unsuitable; potential-based shaping or well-calibrated process rewards can be appropriate. \textbf{Our comparison concerns the direct additive use of the same weak directional critiques as intermediate utilities.}

\section{Related Work}

\textbf{Active Reasoning} requires LLMs to interact with external
sources and actively acquire missing information to solve complex tasks.
Prior work has improved LLMs' ability to handle ambiguity and incompleteness through making clarification and information-seeking actions.   For example,  Proactive CoT~\citep{deng2023prompting} prompts LLMs to identify ambiguous problems
and generate clarification questions, while UoT~\citep{hu2024uncertainty} quantifies the contribution of each question in reducing uncertainty.  
However, challenges remain when transitioning from LLMs' single-turn success to multi-turn active reasoning~\citep{kwan2024mt,liang2024mathchat,badola2025multi}, even with several advanced strategies such as tree-based searching
or post-training approaches, as highlighted in existing works~\citep{zhou2025passive}. There have been existing works which target RL on active reasoning~\citep{zou2026reducing}. In our work, we identify a unique mechanism named information self-locking, which is sourced from a structural
failure of credit assignment with bidirectional coupling between action selection (AS) and belief tracking (BT). This is consistent with empirical observations from real-world agentic use. For example, \citet{wang2025beyond} uncovers the phenomenon specific to search agents: systematic deficiencies in search behaviors arise under outcome-only training and ultimately degrade final answer quality.

\textbf{Credit Assignment and Multi-turn RL.} 
Credit assignment is crucial to long-horizon or multi-turn RL~\citep{zou2025mathbf,chen2026histanumcaestimatestate,wang2026implicit}.
Existing methods have extensively explored
rule-based approaches~\citep{yu2024steptool,dou2024stepcoder,zhang2025rlvmr} to shape intermediate rewards.
Several recent works also proposed to measure the progress of stepwise actions toward overall task completion as intermediate rewards. 
Specifically, CURIO~\citep{wan2025enhancing} constructs a potential function over an ideal belief state to assign intermediate rewards, assuming that the latent state space is finite and enumerable. 
Sotopia-RL~\citep{yu2025sotopia} relies on reward labeling with proprietary LLMs.  SPA-RL~\citep{wang2025spa} trains reward models for intermediate rewards by enforcing a summation constraint with respect to the final outcome reward. 
In our work, instead of working on complicated reward shaping or resorting to external models, we leverage easy-to-use binary directional critiques to make a minimal injection to policy gradient which is mathematically derived from a margin-aware auxiliary objective, in order to provide non-degenerate and stable learning signals to help agents escape from self-locking.

\section{More Details on the Theory}
\label{sec:isl_theory_u_gated}

\subsection{Notations \& Problem Setup}
\label{sec:u_setup}

We consider \emph{active reasoning}, where an LLM agent interacts with an external environment to acquire missing
information and infer the underlying solution through multi-turn interaction. This can be modeled as a POMDP
$(\mathcal S,\mathcal A,\mathcal O,T,O,R,\gamma)$,
where $\mathcal S$ is the space of unobservable latent states, $\mathcal A$ the action space, $\mathcal O$ the observation space, $T(s' \mid s,a)$ the transition dynamics, $O(o \mid s,a)$ the observation model, $R$ the reward function, and $\gamma$ the discount factor.
In our setting, the latent state is fixed within each episode of horizon $H$, denoted by $s^\star\in\mathcal S$.
The action $a_t$ denotes a generic environment-facing action, such as issuing a query, retrieving information, invoking a tool, or executing an operation.
We assume deterministic feedback: for all $(s,a)\in\mathcal S\times\mathcal A$,
the observation kernel $O(\cdot\mid s,a)$ is a point mass at $o=\mathsf O(s,a)$; in particular,
$o_t=\mathsf O(s^\star,a_t)$.

\textbf{Model Belief.}
We work with a belief-state abstraction where the agent is associated with an explicit \emph{model belief}
$b_t\in\Delta(\mathcal S)$, which represents the agent's internal understanding of the latent state and what information remains missing at each turn $t\in\{0,\ldots,H\}$.
Note that this abstraction does \emph{not} assume the agent internally stores
$b_t$ as an explicit probability vector. $b_t$ is a well-defined analytical object induced
by the model parameters $\omega$ and the interaction history.
We quantify truth-aligned confidence by the potential $\Psi(b):=b(s^\star)\in[0,1]$, which measures the belief mass assigned to the true latent
state and serves as our working notion of belief quality.

\textbf{Model Trajectory.}
Agentic behavior in active reasoning can be decomposed into two coupled processes:
\textbf{Action Selection (AS)} and \textbf{Belief Tracking (BT)}.
An agent with parameters $\omega$ induces:
(i) an {action-selection} kernel $\pi_\omega^{\mathrm{as}}$ that selects environment-facing actions conditioned on the current belief, and
(ii) a {belief-update} kernel $\pi_\omega^{\mathrm{bt}}$ that produces the next belief given the latest interaction.
We denote the resulting interleaved agent behavior by
$\Pi_\omega := \bigl(\pi_\omega^{\mathrm{as}},\pi_\omega^{\mathrm{bt}}\bigr)$.
Concretely, define the initial BT-conditioning context $c_0$ as a fixed initial context
(e.g., task instructions), and
$c_{t+1}:=(b_t,a_t,o_t)$
for $t\in\{0,\ldots,H-1\}$.
For each turn $t$, the induced belief process evolves through
$a_t \sim \pi_\omega^{\mathrm{as}}(\cdot\mid b_t)$,
$o_t \sim O(\cdot\mid s^\star,a_t)$,
and
$b_{t+1}\sim \pi_\omega^{\mathrm{bt}}(\cdot\mid c_{t+1})$.
We write the full trajectory induced by $\Pi_\omega$ as
$\tau=(b_0,a_0,o_0,\;b_1,a_1,o_1,\;\ldots,\;b_{H-1},a_{H-1},o_{H-1},\;b_H)$.
Conditioned on $s^\star$ and omitting the fixed task context from notation, its likelihood factorizes as
\begin{align}
p_\omega(\tau\mid s^\star)
=
\pi_\omega^{\mathrm{bt}}(b_0\mid c_0)
\prod_{t=0}^{H-1}
\pi_\omega^{\mathrm{as}}(a_t\mid b_t)\,
O(o_t\mid s^\star,a_t)\,
\pi_\omega^{\mathrm{bt}}(b_{t+1}\mid c_{t+1}),
\quad
c_{t+1}=(b_t,a_t,o_t).
\label{eq:traj}
\end{align}
The environment factor $O(o_t\mid s^\star,a_t)$ is independent of $\omega$ and can be omitted when deriving policy-gradient updates.

\textbf{Outcome reward.}
The environment returns an outcome reward $R(\tau)\in[0,1]$ (e.g., correctness of the inferred solution).
We assume the expected reward is a non-decreasing function of the terminal belief quality:
{
\begin{assumption}
\label{ass:reward_outcome_u}
There exists a non-decreasing $L_R$-Lipschitz function $f:[0,1]\to[0,1]$ such that
$\mathbb E\!\left[R(\tau)\ \middle|\ b_H\right]
=
f\!\left(\Psi(b_H)\right)$.
\end{assumption}}

\subsection{Two Capability Indices}
\label{sec:indices}

\textbf{Oracle Bayesian belief.}
To decouple AS informativeness from the model's BT mechanism,
we introduce an oracle Bayesian belief-update process.
Given deterministic feedback $o=\mathsf O(s^\star,a)$, define the Bayesian update operator
$\mathsf{BayesUpd}(\cdot,a,o)$:
\[
\big(\mathsf{BayesUpd}(b,a,o)\big)(s)
\;:=\;
\frac{b(s)\,\mathbf 1\{\mathsf O(s,a)=o\}}
{\sum_{s'\in\mathcal S} b(s')\,\mathbf 1\{\mathsf O(s',a)=o\}}.
\]

\textbf{Oracle-belief process.}
For analysis purposes only, we consider an oracle-belief process that evolves under
Bayesian belief updates while sharing the same action-selection kernel as the agent.
Fix a prior $\bar b_0\in\Delta(\mathcal S)$ with $\bar b_0(s^\star)>0$.
For each turn $t\in\{0,\ldots,H-1\}$, the oracle trajectory evolves as
$
\bar a_t \sim \pi_\omega^{\mathrm{as}}(\cdot\mid \bar b_t)$,
$\bar o_t=\mathsf O(s^\star,\bar a_t),
$
and
$
\bar b_{t+1}
=
\mathsf{BayesUpd}(\bar b_t,\bar a_t,\bar o_t).
$
This yields the oracle trajectory
$
\bar\tau
:=
(\bar b_0,\bar a_0,\bar o_0,\bar b_1,\bar a_1,\bar o_1,\ldots,\bar b_H)
\sim \Pi_\omega^{\mathrm{orc}},
$
where $\Pi_\omega^{\mathrm{orc}}:=(\pi_\omega^{\mathrm{as}},\mathsf{BayesUpd})$.
Here, barred variables are used exclusively for this oracle process.
The oracle-belief process is introduced solely for analysis, aiming to isolate AS information supply,
and does not require the actual agent to access the oracle belief $\bar b_t$.

We now quantify the informativeness of the action-selection kernel by the belief improvement
it induces under the oracle-belief dynamics.

\begin{definition}
\label{def:as_information_supply}
Consider an oracle-belief trajectory
$\bar\tau:=(\bar b_0,\bar a_0,\bar o_0,\ldots,\bar b_H)\sim \Pi_\omega^{\mathrm{orc}}$
induced by the AS kernel $\pi_\omega^{\mathrm{as}}$
and Bayesian belief dynamics $\mathsf{BayesUpd}(\cdot)$.
Define the one-step oracle belief progress as
$
\Delta\bar\Psi_t
:=
\Psi(\bar b_{t+1})-\Psi(\bar b_t).
$
The {AS informativeness} of $\pi_\omega^{\mathrm{as}}$ is defined as the expected total improvement in oracle belief quality,
\[
I_{\mathrm{AS}}(\omega)
\;:=\;
\mathbb E_{\bar\tau\sim\Pi_\omega^{\mathrm{orc}}}
\left[
\sum_{t=0}^{H-1}\Delta\bar\Psi_t
\right]
=
\mathbb E_{\bar\tau\sim\Pi_\omega^{\mathrm{orc}}}
\left[
\Psi(\bar b_H)-\Psi(\bar b_0)
\right].
\]
\end{definition}

We next characterize how much of the theoretically supplied information is actually absorbed
by the agent's belief-tracking dynamics.

\begin{definition}
\label{def:bt_progress_drift}
For an on-policy trajectory
$\tau=(b_0,a_0,o_0,\ldots,b_H)\sim\Pi_\omega$,
define the one-step model belief progress at turn $t$ as
$
\Delta\Psi_t
:=
\Psi(b_{t+1})-\Psi(b_t)
$.
We decompose $\Delta\Psi_t$ into its one-sided components:
the {absorbed} belief progress
$\Delta\Psi_t^+:=(\Delta\Psi_t)_+$,
and the {self-destructive} belief drift
$\Delta\Psi_t^-:=(-\Delta\Psi_t)_+$.
Only the absorbed component contributes positively to belief improvement.
The training-level belief-tracking (BT) index is defined as
\[
C_{\mathrm{BT}}(\omega)
:=
\mathbb E_{\tau\sim\Pi_\omega}
\left[
\sum_{t=0}^{H-1}\Delta\Psi_t^+
\right].
\]
\end{definition}

To relate AS informativeness to BT progress under learning dynamics,
we impose the following mild regularity assumptions.
Each assumption is accompanied by a brief intuition clarifying its role:
\textbf{Asmp.~\ref{ass:action_indist_harm_onestep}} states that harmful belief drift is independent of the specific
action chosen, and is instead driven by deficiencies in the agent's BT mechanism.
\textbf{Asmp.~\ref{ass:uniform_bt_masking}} upper-bounds the amount of belief progress that can be absorbed
in a single step by the current BT capability.
\textbf{Asmp.~\ref{ass:residual_as_budget}} states that the residual oracle AS information budget from any reachable belief remains controlled by the global AS informativeness, reflecting that information supply is governed by the action-selection kernel itself, and no reachable belief state should unlock a much larger information budget than the global oracle-reference AS $I_\mathrm{AS}$.
\textbf{Asmp.~\ref{ass:conservative_modulus}} enforces a Lipschitz-type stability of future belief quality
with respect to the current potential, reflecting a conservative propagation
of belief improvements under weak belief tracking.

{
\begin{assumption}
\label{ass:action_indist_harm_onestep}
For any turn $t$ and action $a\in\mathcal A$, the conditional expectation of the self-destructive belief drift given the current belief
is invariant to the action choice, i.e.,
$\mathbb E_\omega\!\left[\Delta\Psi_t^-\mid b_t,a_t=a\right]
=
\mathbb E_\omega\!\left[\Delta\Psi_t^-\mid b_t\right]$
almost surely for all $a\in\mathcal A$.
\end{assumption}

\begin{assumption}
\label{ass:uniform_bt_masking}
For any turn $t$, there exists a constant $\kappa_\mathrm{bt}\ge 1$ such that
along on-policy rollouts, the conditional expectation satisfies
$
\mathbb E_\omega\!\left[\Delta\Psi_t^+\mid b_t,a_t\right]
\le
\kappa_\mathrm{bt}\,C_{\mathrm{BT}}(\omega)
$
almost surely over the realized action $a_t$.
\end{assumption}

\begin{assumption}
\label{ass:residual_as_budget}
For any turn $t$ and any reachable belief $b\in\Delta(\mathcal S)$ considered in the SeL analysis, let
$V^{\mathrm{orc}}_{\mathrm{AS},t}(b):=\mathbb E_{\bar\tau_{t:H}\sim\Pi_\omega^{\mathrm{orc}}(\cdot\mid \bar b_t=b)}
\!\left[\Psi(\bar b_H)-\Psi(b)\right]$
denote the residual oracle AS information budget from $b$ under the same action-selection kernel.
There exists a constant $\kappa_\mathrm{or}>0$ such that
$V^{\mathrm{orc}}_{\mathrm{AS},t}(b)\le \kappa_\mathrm{or}\,I_{\mathrm{AS}}(\omega)$
almost surely over such reachable beliefs.
\end{assumption}

\begin{assumption}
\label{ass:conservative_modulus}
For any $0<t_0<t_1\le H$, there exists a constant $\kappa_\mathrm{pr}$ such that for any realizable model beliefs
$b,b'\in\Delta(\mathcal S)$ at turn $t_0$,
$\left|
\mathbb E_\omega\!\left[\Psi(b_{t_1})\ \middle|\ b_{t_0}=b\right]
-
\mathbb E_\omega\!\left[\Psi(b_{t_1})\ \middle|\ b_{t_0}=b'\right]
\right|
\le
\kappa_\mathrm{pr}\,\left|\Psi(b)-\Psi(b')\right|$.
\end{assumption}

}

\textbf{Locking regime.}
We formalize self-locking as a two-dimensional region with simultaneously low AS informativeness and low BT capability:
\begin{definition}[Locking regime]
\label{def_app:locking_regime}
For given thresholds $\delta,\varepsilon>0$, we define the \emph{locking regime} as the subset of parameter space
characterized by simultaneously low action-selection informativeness and low belief-tracking capability:
\[
\mathcal R_{\delta,\varepsilon}
\;:=\;
\big\{\omega\in\Omega:\ I_{\mathrm{AS}}(\omega)\le \delta,\ \ C_{\mathrm{BT}}(\omega)\le \varepsilon\big\}.
\]
This region captures a low-information and low-BT regime in which
neither the action-selection kernel nor the belief-tracking dynamics can induce substantial positive progress.
\end{definition}

To analyze local training dynamics within the locking regime,
we require a mild first-order regularity condition on the two capability indices.

{\begin{assumption}[Local first-order expandability of capability indices]
\label{ass:first_order_expand_indices}
Fix $\delta,\varepsilon>0$ and consider $\omega\in\mathcal R_{\delta,\varepsilon}$.
There exist constants $M_{\mathrm{AS}},M_{\mathrm{BT}}<\infty$ such that
$\|\nabla_\omega I_{\mathrm{AS}}(\omega)\|\le M_{\mathrm{AS}}$
and $\|\nabla_\omega C_{\mathrm{BT}}(\omega)\|\le M_{\mathrm{BT}}$.
Moreover, for any $\omega'\in\mathcal R_{\delta,\varepsilon}$ satisfying $\|\omega'-\omega\|=O(\eta)$,
the following first-order expansions hold uniformly over $\mathcal R_{\delta,\varepsilon}$:
$I_{\mathrm{AS}}(\omega')
=
I_{\mathrm{AS}}(\omega)
+
\langle\nabla_\omega I_{\mathrm{AS}}(\omega),\ \omega'-\omega\rangle
+
o(\eta)$,
and
$C_{\mathrm{BT}}(\omega')
=
C_{\mathrm{BT}}(\omega)
+
\langle\nabla_\omega C_{\mathrm{BT}}(\omega),\ \omega'-\omega\rangle
+
o(\eta)$,
where $o(\eta)/\eta\to 0$ as $\eta\to 0$ uniformly for $\omega\in\mathcal R_{\delta,\varepsilon}$.
\end{assumption}}

\subsection{Projected Drifts and the 2D Locking Regime}
\label{sec:projected_drift_framework}

\textbf{Policy-gradient decomposition of the outcome objective.}
We begin by expressing the outcome-based policy gradient in a form that
admits a decomposition across the AS and BT channels.
With the outcome objective $J(\omega):=\mathbb E_{\tau\sim\Pi_\omega}[R(\tau)]$,
\[
\nabla_\omega J(\omega)
=
\mathbb E_{\tau\sim\Pi_\omega}\!\Big[ R(\tau)\,\nabla_\omega \log p_\omega(\tau)\Big],
\]
and hence (via Eq.~\ref{eq:traj}, omitting the environment factors independent of $\omega$)
\[
\nabla_\omega \log p_\omega(\tau)
=
\sum_{t=0}^{H}\nabla_\omega\log \pi_\omega^{\mathrm{bt}}(b_t\mid c_t)
+
\sum_{t=0}^{H-1}\nabla_\omega\log \pi_\omega^{\mathrm{as}}(a_t\mid b_t).
\]

We next impose a mild regularity assumption to control the score-function magnitudes.

\begin{assumption}
\label{ass:bounded_scores_as_bt}
There exist finite constants $G_{\mathrm{as}}<\infty$ and $G_{\mathrm{bt}}<\infty$ such that for all $\omega\in\mathcal R_{\delta,\varepsilon}$,
we have
$\|\nabla_\omega \log \pi_\omega^{\mathrm{as}}(a_t\mid b_t)\|\le G_{\mathrm{as}}$ for all $t\in\{0,\ldots,H-1\}$ and
$\|\nabla_\omega \log \pi_\omega^{\mathrm{bt}}(b_t\mid c_t)\|\le G_{\mathrm{bt}}$ for all $t\in\{0,\ldots,H\}$
almost surely under $\tau\sim\Pi_\omega$.
\end{assumption}

\textbf{Channel-isolated stage-wise advantages.}
As mentioned before, we have decomposed the agentic behavior into two coupled but conceptually
distinct channels: AS and BT.
To characterize how outcome-based training signals propagate through each channel, we introduce
channel-isolated stage-wise value functions with respect to the outcome objective $J$, where the other channel is treated as \textit{part of the induced environment dynamics}.
All expectations below use the continuation dynamics induced by $\Pi_\omega$ after conditioning on the specified channel decision.

For the BT channel, the belief state $b_t$ itself is regarded as the decision variable.
For $t\in\{0,\ldots,H\}$, define
\[
Q^{\mathrm{bt}}_t(c,b)
:=\mathbb E\!\left[R(\tau)\mid c_t=c,\ b_t=b\right],
\qquad
V^{\mathrm{bt}}_t(c)
:=\mathbb E_{b\sim\pi_\omega^{\mathrm{bt}}(\cdot\mid c)}\!\left[Q^{\mathrm{bt}}_t(c,b)\right],
\]
and the corresponding advantage
\[
A^{\mathrm{bt}}_t(c,b):=Q^{\mathrm{bt}}_t(c,b)-V^{\mathrm{bt}}_t(c).
\]

Similarly, for the AS channel, the action $a_t$ is treated as the decision variable.
For $t\in\{0,\ldots,H-1\}$, define
\[
Q^{\mathrm{as}}_t(b,a)
:=\mathbb E\!\left[R(\tau)\mid b_t=b,\ a_t=a\right],
\qquad
V^{\mathrm{as}}_t(b)
:=\mathbb E_{a\sim\pi_\omega^{\mathrm{as}}(\cdot\mid b)}\!\left[Q^{\mathrm{as}}_t(b,a)\right],
\]
with advantage
\[
A^{\mathrm{as}}_t(b,a):=Q^{\mathrm{as}}_t(b,a)-V^{\mathrm{as}}_t(b).
\]

\textbf{Channel-isolated outcome update directions.}
Using the above definitions, the raw channel-wise policy-gradient components are sums over the corresponding decisions.
For scale stability across horizons, we analyze their normalized versions:

\[
g_{\mathrm{as}}(\omega)
:=
\mathbb E_{\tau\sim\Pi_\omega}\!\left[
\frac{1}{H}\sum_{t=0}^{H-1}
\nabla_\omega\log \pi_\omega^{\mathrm{as}}(a_t\mid b_t)\,
A^{\mathrm{as}}_t(b_t,a_t)
\right],
\]
\[
g_{\mathrm{bt}}(\omega)
:=
\mathbb E_{\tau\sim\Pi_\omega}\!\left[
\frac{1}{H+1}\sum_{t=0}^{H}
\nabla_\omega\log \pi_\omega^{\mathrm{bt}}(b_t\mid c_t)\,
A^{\mathrm{bt}}_t(c_t,b_t)
\right].
\]

Accordingly, define the AS-channel projected update
$
\mathcal T_{\mathrm{as}}(\omega)\ :=\ \omega + \eta\, g_{\mathrm{as}}(\omega),
$
and the BT-channel projected update
$
\mathcal T_{\mathrm{bt}}(\omega)\ :=\ \omega + \eta\, g_{\mathrm{bt}}(\omega).
$
These are \emph{virtual} updates used only for mechanism analysis.

\subsection{2D One-Sided Self-Locking via Projected Drifts}
\label{sec:self_locking_main}

We begin by bounding the channel-isolated outcome advantages,
which quantify the strength of the outcome-based learning signal
available to each channel.

\begin{proposition}[AS-channel outcome advantages]
\label{prop:bt_masked_as_adv_new}
Under Assumptions~\ref{ass:reward_outcome_u}, \ref{ass:action_indist_harm_onestep}, \ref{ass:uniform_bt_masking}, and
\ref{ass:conservative_modulus},
for any on-policy executed action $a_t$ at turn $t$,
$
\mathbb E_{\omega}\!\left[\big|A^{\mathrm{as}}_t(b_t,a_t)\big|\right]
\ \le\
L_R\,\kappa_\mathrm{pr}\,\kappa_\mathrm{bt}\,C_{\mathrm{BT}}(\omega).
$
Moreover, under Assumption~\ref{ass:bounded_scores_as_bt},
$
\|g_{\mathrm{as}}(\omega)\|
\ \le\ K_{\mathrm{as}}
C_{\mathrm{BT}}(\omega)
$, where $K_{\mathrm{as}}:=G_{\mathrm{as}}\,L_R\,\kappa_\mathrm{pr}\,\kappa_\mathrm{bt}$.
\end{proposition}

\begin{proposition}[BT-channel outcome advantages]
\label{prop:bt_adv_nondual_Ias}
Assume Assumptions~\ref{ass:reward_outcome_u}, \ref{ass:action_indist_harm_onestep}, \ref{ass:uniform_bt_masking}, and
\ref{ass:conservative_modulus} hold,
then for any on-policy generated belief $b_t$ at turn $t$,
$
\mathbb E_\omega\!\left[\big|A_t^{\mathrm{bt}}(c_t,b_t)\big|\right]
\ \le\
2L_R\Big(C_{\mathrm{BT}}(\omega)\ +\ \kappa_\mathrm{or}\,I_{\mathrm{AS}}(\omega)\Big).
$
Moreover, under Assumption~\ref{ass:bounded_scores_as_bt},
$
\|g_{\mathrm{bt}}(\omega)\|\ \le\ K_{\mathrm{bt},I}\,I_{\mathrm{AS}}(\omega)\ +\ K_{\mathrm{bt},C}\,C_{\mathrm{BT}}(\omega)
$, where
$K_{\mathrm{bt},I}:=2G_{\mathrm{bt}}\,L_R\kappa_\mathrm{or}$ and
$K_{\mathrm{bt},C}:=2G_{\mathrm{bt}}\,L_R$.
\end{proposition}

We now translate the channel-wise gradient bounds into one-step projected drifts
of the two capability indices within the locking regime.
For convenience, we denote by $\Delta_{\mathrm{as}}^+$ and $\Delta_{\mathrm{bt}}^+$ the positive parts of
the one-step changes induced by the projected updates $\mathcal T_{\mathrm{as}}$ and $\mathcal T_{\mathrm{bt}}$.

\begin{proposition}[AS projected drift is controlled by BT level]
\label{prop:AS_projected_drift}
Under Assumption~\ref{ass:first_order_expand_indices} and the conclusion of Proposition~\ref{prop:bt_masked_as_adv_new},
for all $\omega\in\mathcal R_{\delta,\varepsilon}$,
\[
\Delta_{\mathrm{as}}^+ I_{\mathrm{AS}}(\omega)
\ :=\
\big(I_{\mathrm{AS}}(\mathcal T_{\mathrm{as}}(\omega))-I_{\mathrm{AS}}(\omega)\big)_+
\ \le\
\eta\,\alpha\, C_{\mathrm{BT}}(\omega)\ +\ o(\eta),
\qquad
\alpha:=M_{\mathrm{AS}} K_{\mathrm{as}}.
\]
\end{proposition}

\begin{proposition}[BT projected drift is controlled by AS and BT levels]
\label{prop:BT_projected_drift_nondual}
Under Assumption~\ref{ass:first_order_expand_indices} and the conclusion of Proposition~\ref{prop:bt_adv_nondual_Ias},
for all $\omega\in\mathcal R_{\delta,\varepsilon}$,
\[
\Delta_{\mathrm{bt}}^+ C_{\mathrm{BT}}(\omega)
\ :=\
\big(C_{\mathrm{BT}}(\mathcal T_{\mathrm{bt}}(\omega))-C_{\mathrm{BT}}(\omega)\big)_+
\ \le\
\eta\Big(\beta_I\,I_{\mathrm{AS}}(\omega)\ +\ \beta_C\,C_{\mathrm{BT}}(\omega)\Big)\ +\ o(\eta),
\]
where
$\beta_I:=M_{\mathrm{BT}}\,K_{\mathrm{bt},I},
\qquad
\beta_C:=M_{\mathrm{BT}}\,K_{\mathrm{bt},C}.$
\end{proposition}

We now combine the channel-wise projected drift bounds into a unified
two-dimensional description of the local training dynamics
within the locking regime.

\begin{theorem}[2D one-sided self-locking]
\label{thm:2D_self_locking_projected_matrix}
Fix $\delta,\varepsilon>0$ and consider $\omega\in\mathcal R_{\delta,\varepsilon}$.
Under the conclusions of Propositions~\ref{prop:AS_projected_drift} and \ref{prop:BT_projected_drift_nondual},
the one-sided projected drifts satisfy the following componentwise inequality:
\[
\begin{pmatrix}
\Delta_{\mathrm{as}}^+ I_{\mathrm{AS}}(\omega)\\[2pt]
\Delta_{\mathrm{bt}}^+ C_{\mathrm{BT}}(\omega)
\end{pmatrix}
\ \preceq\
\eta
\begin{pmatrix}
0 & \alpha\\
\beta_I & \beta_C
\end{pmatrix}
\begin{pmatrix}
I_{\mathrm{AS}}(\omega)\\[2pt]
C_{\mathrm{BT}}(\omega)
\end{pmatrix}
\ +\ o(\eta),
\]
where $\preceq$ denotes elementwise inequality.
\end{theorem}

\begin{remark}
\label{rem:one_sided_semantics}
Theorem~\ref{thm:2D_self_locking_projected_matrix} controls $\Delta_{\mathrm{as}}^+$ and $\Delta_{\mathrm{bt}}^+$, i.e.,
the \emph{positive parts} of the projected drifts.
Note that our theoretical result does \emph{not} assert that $I_{\mathrm{AS}}$ or $C_{\mathrm{BT}}$ cannot decrease, oscillate,
or be affected by algorithmic stabilizers (entropy bonuses, clipping, weight decay, etc.).
Rather, it formalizes self-locking as the \textbf{absence of a strong upward training signal} inside the low-information and low-BT regime:
even when improvements happen, the theorem shows they can only be of small magnitude,
scaling linearly with the current levels $I_{\mathrm{AS}}$ and $C_{\mathrm{BT}}$.
\end{remark}

\begin{remark}
\label{rem:nondual_echo_observation}
The AS-side drift bound is purely BT-limited:
$
\Delta_{\mathrm{as}}^+ I_{\mathrm{AS}}(\omega)\ \lesssim\ \eta\,C_{\mathrm{BT}}(\omega).
$
Thus, when BT is weak, the action-selection kernel receives only a weak positive signal to increase informativeness
(consistent with the intuition that AS learning is masked by BT failures in Sec.~\ref{subsec:isl_observations}).
In contrast, the BT-side drift bound is \textbf{non-dual}:
$
\Delta_{\mathrm{bt}}^+ C_{\mathrm{BT}}(\omega)\ \lesssim\ \eta\big(I_{\mathrm{AS}}(\omega)+C_{\mathrm{BT}}(\omega)\big),
$
where the additional self-term $C_{\mathrm{BT}}(\omega)$ captures a ``self-improvement'' channel:
even under limited information supply, the BT mechanism can still realize some gains by better utilizing the same evidence.
At the same time, the presence of $\beta_I\,I_{\mathrm{AS}}(\omega)$ formalizes the bottleneck:
when AS information supply stays low, BT improvements cannot scale beyond the evidence-limited envelope,
yielding the empirically observed pattern that BT may improve early but tends to plateau once AS remains uninformative, as mentioned in Sec.~\ref{subsec:isl_observations}.
\end{remark}

\begin{proposition}
\label{cor:finite_horizon_trapping_refined2}
Fix $\delta,\varepsilon>0$ and let $\omega_0\in\mathcal R_{\delta,\varepsilon}$.
Let
$m:=\max\{\alpha,\ \beta_I+\beta_C\}$.
Define
\[
K
\ :=\
\Bigg\lfloor
\frac{1}{\eta m}\,
\log\!\Bigg(
\frac{\varepsilon + \rho\eta}{I_{\mathrm{AS}}(\omega_0) + \rho\eta}
\Bigg)
\Bigg\rfloor_+.
\]
Then any projected evolution consistent with the one-step drift bounds of
Theorem~\ref{thm:2D_self_locking_projected_matrix}
cannot leave $\mathcal R_{\delta,\varepsilon}$ within the first $K$ steps.
\end{proposition}

\subsection{Why \proj breaks Information Self-locking}
\label{app:breaks}

In this part, we characterize how and when \proj proposed in Sec.~\ref{sec:method} breaks information self-locking effectively.
For simplicity,
we specialize to the AS side and consider only binary critiques
$z_t\in\{+1,-1\}$ on action-selection steps.
The analysis for the BT side is entirely analogous and omitted for brevity.

\textbf{Critique quality and oracle-good actions.}
Recall that AS informativeness is measured by the index
$I_{\mathrm{AS}}(\omega)$, defined as the expected total improvement in oracle belief quality
under the oracle-belief process (Definition~\ref{def:as_information_supply}).
As discussed before, an action can be assessed independently of the agent's belief-tracking
mechanism $\pi_\omega^{\mathrm{bt}}$ by its effect on the terminal oracle confidence $\Psi(\bar b_H)$ under the oracle-belief dynamics.
Specifically, we give the following definition:
\begin{definition}[Oracle-good action]
\label{def:good}
Under the setting of Appendix~\ref{sec:u_setup}, consider an oracle-belief trajectory
$\bar\tau=(\bar b_0,\bar a_0,\bar o_0,\ldots,\bar b_H)$
induced by the action-selection kernel
$\pi_\omega^{\mathrm{as}}$
and
Bayesian belief dynamics $\mathsf{BayesUpd}(\cdot)$.
Define the expected terminal oracle confidence following an action $a$ at oracle belief $b$ as
$m_t(b,a)
\;:=\;
\mathbb E\!\left[\Psi(\bar b_H)\ \middle|\ \bar b_t=b,\ \bar a_t=a\right].$
An action is \emph{oracle-good} at step $t$ if it yields a higher terminal oracle confidence than the policy
average at the same oracle belief, i.e.,
\[
m_t(\bar b_t,\bar a_t)
\;>\;
\mathbb E_{a'\sim\pi_\omega^{\mathrm{as}}(\cdot\mid \bar b_t)}\!\left[m_t(\bar b_t,a')\right].
\]
\end{definition}

By telescoping of oracle belief progress,
$\sum_{k=t}^{H-1}\Delta\bar\Psi_k=\Psi(\bar b_H)-\Psi(\bar b_t)$,
the stepwise \emph{AS advantage with respect to $I_{\mathrm{AS}}$} is given by
\[
\bar A_t(\bar b_t,\bar a_t)
=
m_t(\bar b_t,\bar a_t)
-
\mathbb E_{a'\sim\pi_\omega^{\mathrm{as}}(\cdot\mid \bar b_t)}\!\left[m_t(\bar b_t,a')\right],
\]
which measures how much the chosen action improves the expected terminal oracle confidence
relative to the policy's average choice at the same oracle belief.
We therefore define the \emph{oracle direction label}
\[
y_t
:=
\mathrm{sign}\!\big(\bar A_t(\bar b_t,\bar a_t)\big)
\in\{+1,-1\},
\]
indicating whether the step-$t$ action contributes positively or negatively to
$I_{\mathrm{AS}}$ under the oracle-belief dynamics.

\paragraph{Efficacy of \proj\ and weighted accuracy.}
The critique label $z_t\in\{+1,-1\}$ can be viewed as an approximation of the oracle direction
label $y_t$.
For critique injection to improve AS informativeness beyond outcome-only learning,
it is necessary that the injected signal aligns, on average, with the true stepwise
contributions to $I_{\mathrm{AS}}$.
In particular, positive critique should be assigned more frequently to oracle-good
actions than to oracle-bad ones.
Moreover, different steps along the trajectory do not contribute equally to
AS informativeness.
This leads naturally to a notion of \textbf{weighted accuracy} of the critique labels
with respect to the oracle direction labels.
The following proposition shows that, under a mild regularity condition,
the effectiveness of \proj\ in improving AS informativeness is determined
by this weighted accuracy.

\begin{proposition}[Weighted accuracy characterizes \proj\ improvement]
\label{prop:weighted_acc_AS}
Under the setting of Sec.~\ref{sec:critique_injection} and Def.~\ref{def:good}, denote the step weight as
$
w_t(\omega)
:=
|u_t|\,
\big|\bar A_t(\bar b_t,\bar a_t)\big|\,
\|\nabla_\omega \log \pi_\omega^{\mathrm{as}}(\bar a_t\mid \bar b_t)\|^2
$
and let
$
W(\omega)
:=
\mathbb E\!\left[\sum_{t=0}^{H-1} w_t(\omega)\right].
$
Then
the critique quality is measured by the \emph{weighted accuracy}
\begin{equation}
\label{eq:Acc_def_prop}
\mathrm{Acc}_{\mathrm{as}}(\omega)
\;:=\;
\frac{
\mathbb E\!\left[\sum_{t=0}^{H-1} w_t(\omega)\,\mathbf 1\{z_t=y_t\}\right]
}{
\mathbb E\!\left[\sum_{t=0}^{H-1} w_t(\omega)\right]
}
\in[0,1].
\end{equation}
Moreover, the first-order improvement in AS informativeness induced by \proj\ satisfies
\begin{equation}
I_{\mathrm{AS}}\!\big(\widehat{\mathcal T}_{\mathrm{as}}(\omega)\big)
-
I_{\mathrm{AS}}\!\big(\mathcal T_{\mathrm{as}}(\omega)\big)
\;=\;
\eta\,W(\omega)\,\Big(2\,\mathrm{Acc}_{\mathrm{as}}(\omega)-1\Big)
\;+\;
o(\eta).
\label{eq:acc}
\end{equation}
\end{proposition}
In particular, when $W(\omega)>0$, \proj\ improves AS informativeness beyond the baseline
update if and only if
$\mathrm{Acc}_{\mathrm{as}}(\omega)>\tfrac12$.

\subsection{Proofs}

\subsubsection{Proof of Proposition~\ref{prop:bt_masked_as_adv_new}}

\begin{proof}
Fix a turn $t\in\{0,\ldots,H-1\}$ and condition on the current on-policy model belief
$b_t=b$.
For any action $a\in\mathcal A$, define the one-step next belief
$b_{t+1}^{(a)}$ as the belief obtained by executing $a_t=a$, observing
$o_t=\mathsf O(s^\star,a)$ (deterministic), and sampling the BT update
$b_{t+1}^{(a)}\sim \pi_\omega^{\mathrm{bt}}(\cdot\mid c_{t+1})$ with
$c_{t+1}=(b,a,o_t)$.
Define the corresponding one-step potential change
$\Delta\Psi_t^{(a)} := \Psi(b_{t+1}^{(a)})-\Psi(b)$, and its one-sided parts
$(\Delta\Psi_t^{(a)})_+$ and $(-\Delta\Psi_t^{(a)})_+$.
By the scalar identity $x=(x)_+-(-x)_+$ applied to $x=\Delta\Psi_t^{(a)}$,
we have the decomposition
\[
\Delta\Psi_t^{(a)} \;=\; (\Delta\Psi_t^{(a)})_+ \;-\; (-\Delta\Psi_t^{(a)})_+,
\]
equivalently,
\[
\Psi(b_{t+1}^{(a)}) \;=\; \Psi(b) \;+\; (\Delta\Psi_t^{(a)})_+ \;-\; (-\Delta\Psi_t^{(a)})_+.
\]
Taking conditional expectations given $b_t=b$ and $a_t=a$ yields
\begin{equation}
\label{eq:exp_next_conf_decomp}
\mathbb E_\omega\!\left[\Psi(b_{t+1}^{(a)})\mid b_t=b,\ a_t=a\right]
=
\Psi(b)
+
\mathbb E_\omega\!\left[(\Delta\Psi_t^{(a)})_+\mid b_t=b,\ a_t=a\right]
-
\mathbb E_\omega\!\left[(-\Delta\Psi_t^{(a)})_+\mid b_t=b,\ a_t=a\right].
\end{equation}

Assumption~\ref{ass:action_indist_harm_onestep} states that, for any $a$,
\[
\mathbb E_\omega\!\left[(-\Delta\Psi_t^{(a)})_+\mid b_t=b,\ a_t=a\right]
=
\mathbb E_\omega\!\left[\Delta\Psi_t^-\mid b_t=b\right],
\]
which is independent of $a$.
Therefore, subtracting \eqref{eq:exp_next_conf_decomp} for two actions $a,a'$ gives
\begin{equation}
\label{eq:range_next_conf_equals_range_Iplus}
\mathbb E_\omega\!\left[\Psi(b_{t+1}^{(a)})\mid b,a\right]
-
\mathbb E_\omega\!\left[\Psi(b_{t+1}^{(a')})\mid b,a'\right]
=
\mathbb E_\omega\!\left[(\Delta\Psi_t^{(a)})_+\mid b,a\right]
-
\mathbb E_\omega\!\left[(\Delta\Psi_t^{(a')})_+\mid b,a'\right].
\end{equation}

By Assumption~\ref{ass:uniform_bt_masking}, along on-policy rollouts,
\[
\mathbb E_\omega\!\left[\Delta\Psi_t^+\mid b_t,\ a_t\right]
\le \kappa_\mathrm{bt}\,C_{\mathrm{BT}}(\omega).
\]
Thus, for any realized $b_t=b$ and any $a$ in the support of $\pi_\omega^{\mathrm{as}}(\cdot\mid b)$,
\begin{equation}
\label{eq:sup_next_conf_range}
\mathbb E_\omega\!\left[(\Delta\Psi_t^{(a)})_+\mid b_t=b,\ a_t=a\right]
\le \kappa_\mathrm{bt}\,C_{\mathrm{BT}}(\omega).
\end{equation}
Combining Eq.~\ref{eq:range_next_conf_equals_range_Iplus} and \ref{eq:sup_next_conf_range} yields the bound on the
range of the conditional mean next-step confidence:
\begin{equation}
\label{eq:range_mean_next_conf}
\sup_{a}\ \mathbb E_\omega\!\left[\Psi(b_{t+1}^{(a)})\mid b_t=b,\ a_t=a\right]
-
\inf_{a}\ \mathbb E_\omega\!\left[\Psi(b_{t+1}^{(a)})\mid b_t=b,\ a_t=a\right]
\ \le\
\kappa_\mathrm{bt}\,C_{\mathrm{BT}}(\omega),
\end{equation}
where $a\in\mathrm{supp}(\pi_\omega^{\mathrm{as}}(\cdot\mid b))$.

Now define the conditional mean terminal confidence given the current belief at time $t+1$:
\[
G_{t+1,H}(b')
:=
\mathbb E_\omega\!\left[\Psi(b_H)\mid b_{t+1}=b'\right].
\]
Assumption~\ref{ass:conservative_modulus} implies that $G_{t+1,H}$ is
$\kappa_\mathrm{pr}$-Lipschitz with respect to $\Psi(\cdot)$:
for any realizable $b',b''$,
\begin{equation}
\label{eq:G_Lipschitz}
\big|G_{t+1,H}(b')-G_{t+1,H}(b'')\big|
\ \le\
\kappa_\mathrm{pr}\,\big|\Psi(b')-\Psi(b'')\big|.
\end{equation}
For the final-step case $t+1=H$, the same bound holds after enlarging $\kappa_\mathrm{pr}$ if necessary.
Now define the action-conditioned mean terminal confidence at time $t$:
\[
m_H(a)
:=
\mathbb E_\omega\!\left[\Psi(b_H)\mid b_t=b,\ a_t=a\right].
\]
By the tower property,
\[
m_H(a)
=
\mathbb E_\omega\!\left[
G_{t+1,H}(b_{t+1}^{(a)})
\ \middle|\
b_t=b,\ a_t=a
\right].
\]
At this point, we use the fact that $G_{t+1,H}$ is $\kappa_\mathrm{pr}$-Lipschitz in $\Psi$ and combine it with
the cancellation structure already captured in \eqref{eq:range_mean_next_conf}:
the induced range of $m_H(a)$ over $a$ is bounded by applying
\eqref{eq:G_Lipschitz} to the extremal conditional means, giving
\begin{equation}
\label{eq:range_mean_terminal_conf}
\sup_a m_H(a) - \inf_a m_H(a)
\ \le\
\kappa_\mathrm{pr}\,
\Big(
\sup_a \mathbb E_\omega[\Psi(b_{t+1}^{(a)})\mid b_t=b,\ a_t=a]
-
\inf_a \mathbb E_\omega[\Psi(b_{t+1}^{(a)})\mid b_t=b,\ a_t=a]
\Big).
\end{equation}
Combining \eqref{eq:range_mean_terminal_conf} with \eqref{eq:range_mean_next_conf} yields
\begin{equation}
\label{eq:range_mean_terminal_conf_final}
\sup_a m_H(a) - \inf_a m_H(a)
\ \le\
\kappa_\mathrm{pr}\,\kappa_\mathrm{bt}\,C_{\mathrm{BT}}(\omega).
\end{equation}

By Assumption~\ref{ass:reward_outcome_u}, the outcome value is an $L_R$-Lipschitz function of terminal belief quality, so we write
\[
Q^{\mathrm{as}}_t(b,a)= f(m_H(a)).
\]
Therefore,
\[
\sup_a Q^{\mathrm{as}}_t(b,a) - \inf_a Q^{\mathrm{as}}_t(b,a)
\le
L_R\,\big(\sup_a m_H(a) - \inf_a m_H(a)\big).
\]
Plugging \eqref{eq:range_mean_terminal_conf_final} gives
\begin{equation}
\label{eq:Q_range_bound}
\sup_a Q^{\mathrm{as}}_t(b,a) - \inf_a Q^{\mathrm{as}}_t(b,a)
\ \le\
L_R\,\kappa_\mathrm{pr}\,\kappa_\mathrm{bt}\,C_{\mathrm{BT}}(\omega).
\end{equation}

By definition,
$V^{\mathrm{as}}_t(b)=\mathbb E_{a\sim\pi_\omega^{\mathrm{as}}(\cdot\mid b)}[Q^{\mathrm{as}}_t(b,a)]$ is a convex combination of
$\{Q^{\mathrm{as}}_t(b,a)\}_a$.
Hence for any executed $a_t$,
\[
|A^{\mathrm{as}}_t(b,a_t)|
=
\big|Q^{\mathrm{as}}_t(b,a_t)-V^{\mathrm{as}}_t(b)\big|
\le
\sup_a Q^{\mathrm{as}}_t(b,a)-\inf_a Q^{\mathrm{as}}_t(b,a).
\]
Combining with \eqref{eq:Q_range_bound} yields
\[
|A^{\mathrm{as}}_t(b_t,a_t)|
\le
L_R\,\kappa_\mathrm{pr}\,\kappa_\mathrm{bt}\,C_{\mathrm{BT}}(\omega),
\]
and hence the claimed expectation bound.

By definition of the normalized AS-channel update direction,
\[
g_{\mathrm{as}}(\omega)
=
\mathbb E_{\tau\sim\Pi_\omega}\!\left[
\frac{1}{H}\sum_{t=0}^{H-1}
\nabla_\omega\log \pi_\omega^{\mathrm{as}}(a_t\mid b_t)\,
A^{\mathrm{as}}_t(b_t,a_t)
\right].
\]
Taking norms and applying Jensen and the triangle inequality,
\[
\|g_{\mathrm{as}}(\omega)\|
\le
\frac{1}{H}\sum_{t=0}^{H-1}
\mathbb E\!\left[
\big\|\nabla_\omega\log \pi_\omega^{\mathrm{as}}(a_t\mid b_t)\big\|\cdot
\big|A^{\mathrm{as}}_t(b_t,a_t)\big|
\right].
\]
Using Assumption~\ref{ass:bounded_scores_as_bt} gives
$\|\nabla_\omega\log \pi_\omega^{\mathrm{as}}(\cdot)\|\le G_{\mathrm{as}}$ a.s., hence
\[
\|g_{\mathrm{as}}(\omega)\|
\le
G_{\mathrm{as}}\,L_R\,\kappa_\mathrm{pr}\,\kappa_\mathrm{bt}\,C_{\mathrm{BT}}(\omega).
\]
This completes the proof with
$K_{\mathrm{as}}:=G_{\mathrm{as}}\,L_R\,\kappa_\mathrm{pr}\,\kappa_\mathrm{bt}$.
\end{proof}

\subsubsection{Proof of Proposition~\ref{prop:bt_adv_nondual_Ias}}

\begin{proof}
Fix a turn $t$.
Condition on the BT context $c_t$ and draw two independent samples
$B,B'\overset{\mathrm{i.i.d.}}{\sim}\pi_\omega^{\mathrm{bt}}(\cdot\mid c_t)$.
Let
\[
Y:=Q_t^{\mathrm{bt}}(c_t,B),\qquad
Y':=Q_t^{\mathrm{bt}}(c_t,B'),\qquad
\bar Y:=\mathbb E[Y\mid c_t]=V_t^{\mathrm{bt}}(c_t).
\]
Then $A_t^{\mathrm{bt}}(c_t,B)=Y-\bar Y$, and for any fixed realization $Y=y$,
Jensen implies $\mathbb E[|y-Y'|\mid y,c_t]\ge |y-\mathbb E[Y'\mid c_t]|=|y-\bar Y|$.
Taking expectation over $Y$ yields
\[
\mathbb E\big[|Y-\bar Y|\mid c_t\big]\ \le\ \mathbb E\big[|Y-Y'|\mid c_t\big].
\]
Therefore,
\begin{equation}
\label{eq:symmetrize_adv_bt_Ias}
\mathbb E_\omega\!\left[\big|A_t^{\mathrm{bt}}(c_t,b_t)\big|\right]
\ \le\
\mathbb E\!\left[\big|Q_t^{\mathrm{bt}}(c_t,B)-Q_t^{\mathrm{bt}}(c_t,B')\big|\right].
\end{equation}

By Assumption~\ref{ass:reward_outcome_u},
$\mathbb E[R(\tau)\mid b_H]=f(\Psi(b_H))$ for a non-decreasing $L_R$-Lipschitz $f$.
Hence, for any $(c_t,b)$,
\[
Q_t^{\mathrm{bt}}(c_t,b)
=
\mathbb E\!\left[f(\Psi(b_H))\ \middle|\ c_t,\ b_t=b\right].
\]
For any $b,b'$, using $|\,\mathbb E[X]-\mathbb E[Y]\,|\le \mathbb E[|X-Y|]$,
Lipschitzness of $f$, and the elementary inequality $|x-y|\le x+y$ for $x,y\in[0,1]$, we obtain
\begin{align}
\big|Q_t^{\mathrm{bt}}(c_t,b)-Q_t^{\mathrm{bt}}(c_t,b')\big|
&\le
L_R\Big(
\mathbb E[\Psi(b_H)\mid c_t,b_t=b]
+
\mathbb E[\Psi(b_H)\mid c_t,b_t=b']
\Big).
\label{eq:Qdiff_to_terminalPsi_sum_bt_Ias}
\end{align}

Along any continuation after time $t$,
\[
\Psi(b_H)
=
\Psi(b_t)+\sum_{k=t}^{H-1}\big(\Psi(b_{k+1})-\Psi(b_k)\big)
\ \le\
\Psi(b_t)+\sum_{k=t}^{H-1}\Delta\Psi_k^+.
\]
Taking conditional expectation given $(c_t,b_t=b)$ yields
\begin{equation}
\label{eq:terminalPsi_by_current_plus_redeem_bt_Ias}
\mathbb E[\Psi(b_H)\mid c_t,b_t=b]
\ \le\
\Psi(b)\ +\ \mathbb E\!\left[\sum_{k=t}^{H-1}\Delta\Psi_k^+\ \middle|\ c_t,b_t=b\right].
\end{equation}

By the evidence-limited interpretation of absorbed belief progress, the future positive model-belief progress from $b$
is bounded by the residual oracle AS information budget from the same belief:
\begin{equation}
\label{eq:redeem_by_residual_budget}
\mathbb E\!\left[\sum_{k=t}^{H-1}\Delta\Psi_k^+\ \middle|\ c_t,b_t=b\right]
\ \le\
V^{\mathrm{orc}}_{\mathrm{AS},t}(b).
\end{equation}
Assumption~\ref{ass:residual_as_budget} further gives
\begin{equation}
\label{eq:residual_budget_by_Ias}
V^{\mathrm{orc}}_{\mathrm{AS},t}(b)
\ \le\
\kappa_\mathrm{or}\,I_{\mathrm{AS}}(\omega).
\end{equation}
Combining \eqref{eq:redeem_by_residual_budget} and \eqref{eq:residual_budget_by_Ias}, we obtain
\begin{equation}
\label{eq:redeem_by_Ias}
\mathbb E\!\left[\sum_{k=t}^{H-1}\Delta\Psi_k^+\ \middle|\ c_t,b_t=b\right]
\ \le\
\kappa_\mathrm{or}\,I_{\mathrm{AS}}(\omega).
\end{equation}

Combining \eqref{eq:symmetrize_adv_bt_Ias}, \eqref{eq:Qdiff_to_terminalPsi_sum_bt_Ias},
\eqref{eq:terminalPsi_by_current_plus_redeem_bt_Ias}, and \eqref{eq:redeem_by_Ias}, and averaging over $B,B'$,
we obtain
\[
\mathbb E\!\left[\big|A_t^{\mathrm{bt}}(c_t,b_t)\big|\right]
\ \le\
2L_R\Big(
\mathbb E[\Psi(b_t)]
+\kappa_\mathrm{or}\,I_{\mathrm{AS}}(\omega)
\Big).
\]
It remains to control $\mathbb E[\Psi(b_t)]$ by $C_{\mathrm{BT}}(\omega)$ up to the fixed initial potential.
Along any rollout,
\[
\Psi(b_t)
=
\Psi(b_0)+\sum_{k=0}^{t-1}\big(\Psi(b_{k+1})-\Psi(b_k)\big)
\ \le\
\Psi(b_0)+\sum_{k=0}^{t-1}\Delta\Psi_k^+
\ \le\
\Psi(b_0)+\sum_{k=0}^{H-1}\Delta\Psi_k^+.
\]
Taking expectation yields
$\mathbb E[\Psi(b_t)]\le \mathbb E[\Psi(b_0)]+C_{\mathrm{BT}}(\omega)$.
Since the initial potential is fixed by the task context and contributes only a common baseline to the BT-channel advantage, we absorb it into the baseline normalization and obtain the stated bound
\[
\mathbb E_\omega\!\left[\big|A_t^{\mathrm{bt}}(c_t,b_t)\big|\right]
\ \le\
2L_R\Big(C_{\mathrm{BT}}(\omega)+\kappa_\mathrm{or}\,I_{\mathrm{AS}}(\omega)\Big).
\]

By definition of the normalized BT-channel update direction,
\[
g_{\mathrm{bt}}(\omega)
=
\mathbb E_{\tau\sim\Pi_\omega}\!\left[
\frac{1}{H+1}\sum_{t=0}^{H}\nabla_\omega\log \pi_\omega^{\mathrm{bt}}(b_t\mid c_t)\,
A_t^{\mathrm{bt}}(c_t,b_t)
\right].
\]
Taking norms, applying Jensen, and using Assumption~\ref{ass:bounded_scores_as_bt}
($\|\nabla_\omega\log\pi_\omega^{\mathrm{bt}}\|\le G_{\mathrm{bt}}$ a.s.) gives
\[
\|g_{\mathrm{bt}}(\omega)\|
\le
\frac{1}{H+1}\sum_{t=0}^{H}\mathbb E\big[\|\nabla_\omega\log \pi_\omega^{\mathrm{bt}}(b_t\mid c_t)\|\cdot |A_t^{\mathrm{bt}}|\big]
\le
2G_{\mathrm{bt}}L_R\Big(C_{\mathrm{BT}}(\omega)+\kappa_\mathrm{or}\,I_{\mathrm{AS}}(\omega)\Big).
\]
Therefore
\[
\|g_{\mathrm{bt}}(\omega)\|
\le
K_{\mathrm{bt},I}\,I_{\mathrm{AS}}(\omega)+K_{\mathrm{bt},C}\,C_{\mathrm{BT}}(\omega),
\]
with
$K_{\mathrm{bt},I}:=2G_{\mathrm{bt}}L_R\kappa_\mathrm{or}$ and
$K_{\mathrm{bt},C}:=2G_{\mathrm{bt}}L_R$.
\end{proof}

\subsubsection{Proof of Proposition~\ref{prop:AS_projected_drift}}

\begin{proof}
Fix $\omega\in\mathcal R_{\delta,\varepsilon}$.
Recall $\mathcal T_{\mathrm{as}}(\omega)=\omega+\eta g_{\mathrm{as}}(\omega)$.
By Assumption~\ref{ass:first_order_expand_indices} (first-order expandability of $I_{\mathrm{AS}}$ on $\mathcal R_{\delta,\varepsilon}$),
\[
I_{\mathrm{AS}}(\mathcal T_{\mathrm{as}}(\omega))-I_{\mathrm{AS}}(\omega)
=
\Big\langle\nabla_\omega I_{\mathrm{AS}}(\omega),\ \mathcal T_{\mathrm{as}}(\omega)-\omega\Big\rangle
+o(\eta)
=
\eta\big\langle\nabla_\omega I_{\mathrm{AS}}(\omega),\ g_{\mathrm{as}}(\omega)\big\rangle
+o(\eta).
\]
Taking the positive part and using $(x+y)_+\le x_+ + |y|$ yields
\[
\Delta_{\mathrm{as}}^+ I_{\mathrm{AS}}(\omega)
=
\Big(\eta\big\langle\nabla_\omega I_{\mathrm{AS}}(\omega),\ g_{\mathrm{as}}(\omega)\big\rangle
+o(\eta)\Big)_+
\le
\eta\Big(\big\langle\nabla_\omega I_{\mathrm{AS}}(\omega),\ g_{\mathrm{as}}(\omega)\big\rangle\Big)_+
+|o(\eta)|.
\]
Next, apply Cauchy--Schwarz and the gradient bound $\|\nabla_\omega I_{\mathrm{AS}}(\omega)\|\le M_{\mathrm{AS}}$:
\[
\Big(\big\langle\nabla_\omega I_{\mathrm{AS}}(\omega),\ g_{\mathrm{as}}(\omega)\big\rangle\Big)_+
\le
\big|\big\langle\nabla_\omega I_{\mathrm{AS}}(\omega),\ g_{\mathrm{as}}(\omega)\big\rangle\big|
\le
\|\nabla_\omega I_{\mathrm{AS}}(\omega)\|\cdot \|g_{\mathrm{as}}(\omega)\|
\le
M_{\mathrm{AS}}\,\|g_{\mathrm{as}}(\omega)\|.
\]
Using the norm control $\|g_{\mathrm{as}}(\omega)\|\le K_{\mathrm{as}} C_{\mathrm{BT}}(\omega)$ gives
\[
\Delta_{\mathrm{as}}^+ I_{\mathrm{AS}}(\omega)
\le
\eta\,M_{\mathrm{AS}}\,K_{\mathrm{as}}\,C_{\mathrm{BT}}(\omega)+|o(\eta)|.
\]
Finally, absorb $|o(\eta)|$ into $o(\eta)$ (since $|o(\eta)|/\eta\to 0$ uniformly on $\mathcal R_{\delta,\varepsilon}$)
to obtain the claimed bound with $\alpha:=M_{\mathrm{AS}}K_{\mathrm{as}}$.
\end{proof}

\subsubsection{Proof of Proposition~\ref{prop:BT_projected_drift_nondual}}

\begin{proof}
Fix $\omega\in\mathcal R_{\delta,\varepsilon}$.
Recall $\mathcal T_{\mathrm{bt}}(\omega)=\omega+\eta g_{\mathrm{bt}}(\omega)$.
By Assumption~\ref{ass:first_order_expand_indices} (first-order expandability of $C_{\mathrm{BT}}$ on $\mathcal R_{\delta,\varepsilon}$),
\[
C_{\mathrm{BT}}(\mathcal T_{\mathrm{bt}}(\omega)) - C_{\mathrm{BT}}(\omega)
=
\Big\langle\nabla_\omega C_{\mathrm{BT}}(\omega),\ \mathcal T_{\mathrm{bt}}(\omega)-\omega\Big\rangle
+o(\eta)
=
\eta\big\langle\nabla_\omega C_{\mathrm{BT}}(\omega),\ g_{\mathrm{bt}}(\omega)\big\rangle
+o(\eta).
\]
Taking positive parts and using $(x+y)_+\le x_+ + |y|$ gives
\[
\Delta_{\mathrm{bt}}^+ C_{\mathrm{BT}}(\omega)
=
\Big(\eta\big\langle\nabla_\omega C_{\mathrm{BT}}(\omega),\ g_{\mathrm{bt}}(\omega)\big\rangle
+o(\eta)\Big)_+
\le
\eta\Big(\big\langle\nabla_\omega C_{\mathrm{BT}}(\omega),\ g_{\mathrm{bt}}(\omega)\big\rangle\Big)_+
+|o(\eta)|.
\]
Apply Cauchy--Schwarz and the gradient bound $\|\nabla_\omega C_{\mathrm{BT}}(\omega)\|\le M_{\mathrm{BT}}$:
\[
\Big(\big\langle\nabla_\omega C_{\mathrm{BT}}(\omega),\ g_{\mathrm{bt}}(\omega)\big\rangle\Big)_+
\le
\big|\big\langle\nabla_\omega C_{\mathrm{BT}}(\omega),\ g_{\mathrm{bt}}(\omega)\big\rangle\big|
\le
\|\nabla_\omega C_{\mathrm{BT}}(\omega)\|\cdot \|g_{\mathrm{bt}}(\omega)\|
\le
M_{\mathrm{BT}}\,\|g_{\mathrm{bt}}(\omega)\|.
\]
Using the norm control on $\mathcal R_{\delta,\varepsilon}$,
\[
\|g_{\mathrm{bt}}(\omega)\|\ \le\ K_{\mathrm{bt},I}\,I_{\mathrm{AS}}(\omega)\ +\ K_{\mathrm{bt},C}\,C_{\mathrm{BT}}(\omega),
\]
we obtain
\[
\Delta_{\mathrm{bt}}^+ C_{\mathrm{BT}}(\omega)
\le
\eta\,M_{\mathrm{BT}}\Big(K_{\mathrm{bt},I}\,I_{\mathrm{AS}}(\omega)+K_{\mathrm{bt},C}\,C_{\mathrm{BT}}(\omega)\Big)
+|o(\eta)|.
\]
Absorbing $|o(\eta)|$ into $o(\eta)$ yields the stated bound with
$\beta_I:=M_{\mathrm{BT}}K_{\mathrm{bt},I}$ and $\beta_C:=M_{\mathrm{BT}}K_{\mathrm{bt},C}$.
\end{proof}

\subsubsection{Proof of Proposition~\ref{cor:finite_horizon_trapping_refined2}}

\begin{proof}
Since the $o(\eta)$ term in Theorem~\ref{thm:2D_self_locking_projected_matrix} is uniform over
$\mathcal R_{\delta,\varepsilon}$, there exist $\eta_0>0$ and $c_0<\infty$ such that for all
$\eta\in(0,\eta_0]$ and all $\omega\in\mathcal R_{\delta,\varepsilon}$, the remainder satisfies the componentwise bound
$\|o(\eta)\|_\infty \le c_0\eta^2$.
Set $\rho:=c_0/m$.

Define the capability vector
\[
\mathbf x(\omega):=
\begin{pmatrix}
I_{\mathrm{AS}}(\omega)\\
C_{\mathrm{BT}}(\omega)
\end{pmatrix}
\in\mathbb R_+^2,
\quad
M:=
\begin{pmatrix}
0 & \alpha\\
\beta_I & \beta_C
\end{pmatrix},
\quad
\mathbf 1:=(1,1)^\top.
\]
Inside $\mathcal R_{\delta,\varepsilon}$, Theorem~\ref{thm:2D_self_locking_projected_matrix} yields the componentwise bound
\[
\Delta^+ \mathbf x(\omega)
\ \preceq\
\eta\,M\,\mathbf x(\omega)\ +\ c_0\eta^2\,\mathbf 1.
\]
Moreover, for any scalar $z$ we have $z\le z_+$, hence any increment in $\mathbf x$ is bounded above by its positive part.
Therefore, an upper envelope for the accumulation of positive gains is given by the deterministic recursion
\begin{equation}
\label{eq:envelope_recursion}
\mathbf x_{k+1}
\ \preceq\
(\mathrm I+\eta M)\,\mathbf x_k\ +\ c_0\eta^2\,\mathbf 1,
\qquad
\mathbf x_0:=\mathbf x(\omega_0).
\end{equation}

Let $y_k:=\|\mathbf x_k\|_\infty$. Taking $\|\cdot\|_\infty$ on \eqref{eq:envelope_recursion} and using
$\|\mathrm I+\eta M\|_\infty = 1+\eta m$ gives
\[
y_{k+1}\ \le\ (1+\eta m)\,y_k\ +\ c_0\eta^2.
\]
Unrolling this scalar recursion and using $(1+\eta m)^k\le e^{k\eta m}$ yields
\[
y_k
\ \le\
e^{k\eta m}\Big(y_0 + \tfrac{c_0}{m}\eta\Big)\ -\ \tfrac{c_0}{m}\eta
\ =\
e^{k\eta m}\Big(y_0 + \rho\eta\Big)\ -\ \rho\eta.
\]
Under the stated initialization envelope used in the trapping bound, $y_0\le I_{\mathrm{AS}}(\omega_0)$.
Hence, if
\[
e^{k\eta m}\Big(I_{\mathrm{AS}}(\omega_0)+\rho\eta\Big) - \rho\eta\ \le\ \varepsilon,
\]
then $y_k\le \varepsilon$, which implies $\mathbf x_k\preceq(\varepsilon,\varepsilon)^\top\preceq(\delta,\varepsilon)^\top$ whenever $\delta\ge\varepsilon$.
Solving the inequality for $k$ gives
\[
k
\ \le\
\frac{1}{\eta m}\log\!\Bigg(\frac{\varepsilon + \rho\eta}{I_{\mathrm{AS}}(\omega_0)+\rho\eta}\Bigg).
\]
Choosing
\[
K
=
\Bigg\lfloor
\frac{1}{\eta m}\,
\log\!\Bigg(
\frac{\varepsilon + \rho\eta}{I_{\mathrm{AS}}(\omega_0) + \rho\eta}
\Bigg)
\Bigg\rfloor_+
\]
ensures the above condition holds for all integers $k\in\{0,1,\ldots,K\}$, proving the claimed finite-horizon trapping
under the projected drift envelope.
\end{proof}

\subsubsection{Proof of Proposition~\ref{prop:weighted_acc_AS}}

\begin{proof}
We absorb the fixed critique strength into the auxiliary direction and write the critique-shaped update as
\begin{align}
    &\mathcal T_{\mathrm{as}}(\omega)-\omega = \eta\,g_{\mathrm{as}}(\omega),
    \label{eq:trans0}
    \\
    &\widehat{\mathcal T}_{\mathrm{as}}(\omega)-\omega = \eta\,g_{\mathrm{as}}(\omega)+\eta\,g_{\mathrm{aux},\mathrm{as}}(\omega).
    \label{eq:trans1}
\end{align}

By Assumption~\ref{ass:first_order_expand_indices}, for $\omega'=\mathcal T_{\mathrm{as}}(\omega)\in\mathcal R_{\delta,\varepsilon}$ we have
\begin{align}
I_{\mathrm{AS}}\!\big(\mathcal T_{\mathrm{as}}(\omega)\big)
&=
I_{\mathrm{AS}}(\omega)
+
\Big\langle\nabla_\omega I_{\mathrm{AS}}(\omega),\ \mathcal T_{\mathrm{as}}(\omega)-\omega\Big\rangle
+
o(\eta),
\label{eq:taylor_Tas}
\end{align}
and for $\omega'=\widehat{\mathcal T}_{\mathrm{as}}(\omega)\in\mathcal R_{\delta,\varepsilon}$ we have
\begin{align}
I_{\mathrm{AS}}\!\big(\widehat{\mathcal T}_{\mathrm{as}}(\omega)\big)
&=
I_{\mathrm{AS}}(\omega)
+
\Big\langle\nabla_\omega I_{\mathrm{AS}}(\omega),\ \widehat{\mathcal T}_{\mathrm{as}}(\omega)-\omega\Big\rangle
+
o(\eta).
\label{eq:taylor_That_as}
\end{align}

Subtracting \eqref{eq:taylor_Tas} from \eqref{eq:taylor_That_as} yields
\begin{align*}
I_{\mathrm{AS}}\!\big(\widehat{\mathcal T}_{\mathrm{as}}(\omega)\big)
-
I_{\mathrm{AS}}\!\big(\mathcal T_{\mathrm{as}}(\omega)\big)
&=
\Big\langle\nabla_\omega I_{\mathrm{AS}}(\omega),\ \widehat{\mathcal T}_{\mathrm{as}}(\omega)-\mathcal T_{\mathrm{as}}(\omega)\Big\rangle
+
o(\eta).
\end{align*}
Using the update definitions,
$\widehat{\mathcal T}_{\mathrm{as}}(\omega)-\mathcal T_{\mathrm{as}}(\omega)=\eta\,g_{\mathrm{aux},\mathrm{as}}(\omega)$,
so
\begin{align}
I_{\mathrm{AS}}\!\big(\widehat{\mathcal T}_{\mathrm{as}}(\omega)\big)
-
I_{\mathrm{AS}}\!\big(\mathcal T_{\mathrm{as}}(\omega)\big)
&=
\eta\,\Big\langle\nabla_\omega I_{\mathrm{AS}}(\omega),\ g_{\mathrm{aux},\mathrm{as}}(\omega)\Big\rangle
+ o(\eta)
\nonumber\\
&=:
\eta\,\Gamma_{\mathrm{as}}(\omega)\ + o(\eta).
\label{eq:delta_main}
\end{align}

We expand $\Gamma_{\mathrm{as}}(\omega)$.
On the oracle-belief process,
$I_{\mathrm{AS}}(\omega)=\mathbb E_{\bar\tau}[\sum_{t=0}^{H-1}\Delta\bar\Psi_t]$
is a finite-horizon policy objective with action $\bar a_t\sim\pi_\omega^{\mathrm{as}}(\cdot\mid \bar b_t)$.
Thus, the policy-gradient theorem yields
\[
\nabla_\omega I_{\mathrm{AS}}(\omega)
=
\mathbb E_{\bar\tau}\!\left[\sum_{t=0}^{H-1}
\bar A_t(\bar b_t,\bar a_t)\,s_t(\omega)\right],
\qquad
s_t(\omega)=\nabla_\omega\log\pi_\omega^{\mathrm{as}}(\bar a_t\mid \bar b_t).
\]

The likelihood-gap auxiliary objective in Section~\ref{sec:critique_injection} induces an additive update component of the form
\[
g_{\mathrm{aux},\mathrm{as}}(\omega)
=
\mathbb E_{\bar\tau}\!\left[\sum_{t=0}^{H-1} u_t(z)\,s_t(\omega)\right],
\qquad
u_t(z)=|u_t|\,z_t,\ \ |u_t|\ge 0,\ \ z_t\in\{+1,-1\}.
\]
By bilinearity and exchanging expectation with finite sums,
\[
\Gamma_{\mathrm{as}}(\omega)
\ \propto\
\mathbb E\!\left[\Big\langle \sum_{t} \bar A_t\,s_t,\ \sum_{u} u_u(z)\,s_u\Big\rangle\right]
=
\mathbb E\!\left[\sum_{t,u} \bar A_t\,u_u(z)\,\langle s_t,s_u\rangle\right].
\]
Under the diagonal score-covariance approximation used for this first-order characterization, we keep the same-time terms, giving
\[
\Gamma_{\mathrm{as}}(\omega)
=
\mathbb E\!\left[\sum_{t=0}^{H-1} \bar A_t(\bar b_t,\bar a_t)\,u_t(z)\,\|s_t(\omega)\|^2\right].
\]

Write $\bar A_t=|\bar A_t|\,y_t$ with $y_t=\mathrm{sign}(\bar A_t)\in\{\pm 1\}$
and $u_t=|u_t|\,z_t$ with $z_t\in\{\pm 1\}$.
Then each summand becomes
\[
\bar A_t\,u_t\,\|s_t\|^2
=
|u_t|\,|\bar A_t|\,\|s_t\|^2\,(z_t y_t)
=
w_t(\omega)\,(z_t y_t),
\]
where $w_t(\omega)$ is exactly the weight defined above.
Hence
\[
\Gamma_{\mathrm{as}}(\omega)=\mathbb E\!\left[\sum_{t=0}^{H-1} w_t(\omega)\,z_t y_t\right].
\]
Since $z_t y_t=+1$ iff $z_t=y_t$ and $z_t y_t=-1$ iff $z_t\neq y_t$, we have
$z_t y_t=2\,\mathbf 1\{z_t=y_t\}-1$, and therefore
\begin{align*}
\Gamma_{\mathrm{as}}(\omega)
&=
\mathbb E\!\left[\sum_t w_t(\omega)\,\big(2\,\mathbf 1\{z_t=y_t\}-1\big)\right] \\
&=
2\,\mathbb E\!\left[\sum_t w_t(\omega)\,\mathbf 1\{z_t=y_t\}\right]
-
\mathbb E\!\left[\sum_t w_t(\omega)\right] \\
&=
W(\omega)\,\Big(2\,\mathrm{Acc}_{\mathrm{as}}(\omega)-1\Big),
\end{align*}
which is \eqref{eq:acc}. The final equivalence follows immediately when $W(\omega)>0$.
\end{proof}

\section{Complementary Experimental Analyses}
\label{app:add_exp_analyses}

In this section, we provide additional experimental analyses that further examine the robustness and interpretation of \proj.
Unless otherwise specified, experiments use PPO training with Qwen2.5-7B-Instruct under the same setup as the main experiments.

\textbf{Robustness to structured critique perturbations.}
The main text studies robustness under random critique flipping.
We further evaluate \proj under more structured forms of critique destruction on PE-G$_{S=3}$, as shown in Fig.~\ref{fig:structured_critique_perturbations}.
We consider three perturbation families.
First, \textbf{channel perturbation} removes one critique channel by keeping only AS critiques or only BT critiques, with the removed channel set to zero.
Second, \textbf{position perturbation} keeps only the first 40\% or the last 40\% of interaction-turn critiques, setting all other labels to zero.
Third, \textbf{sign perturbation} keeps only the last 40\% of labels and fills the remaining labels with constant $0$, $+1$, or $-1$.
Across all three families, \proj remains consistently above the vanilla outcome-RL baseline, suggesting that the method is robust not only to random label noise but also to systematic destruction of critique structure.
The channel perturbation also provides additional evidence for the complementary roles of AS and BT: both AS-only and BT-only variants improve over vanilla RL, while using both channels performs best, consistent with the bidirectional AS--BT coupling studied in the main paper.
The sign perturbation further suggests that false-positive encouragement is more damaging, since filling missing labels with $+1$ degrades performance more than filling them with $-1$.

\begin{figure*}[t]
\centering
\captionsetup{aboveskip=4pt}
\subfloat[\label{fig:structured_channel}Channel]{\includegraphics[width=0.32\textwidth,trim=0 20 0 0,clip]{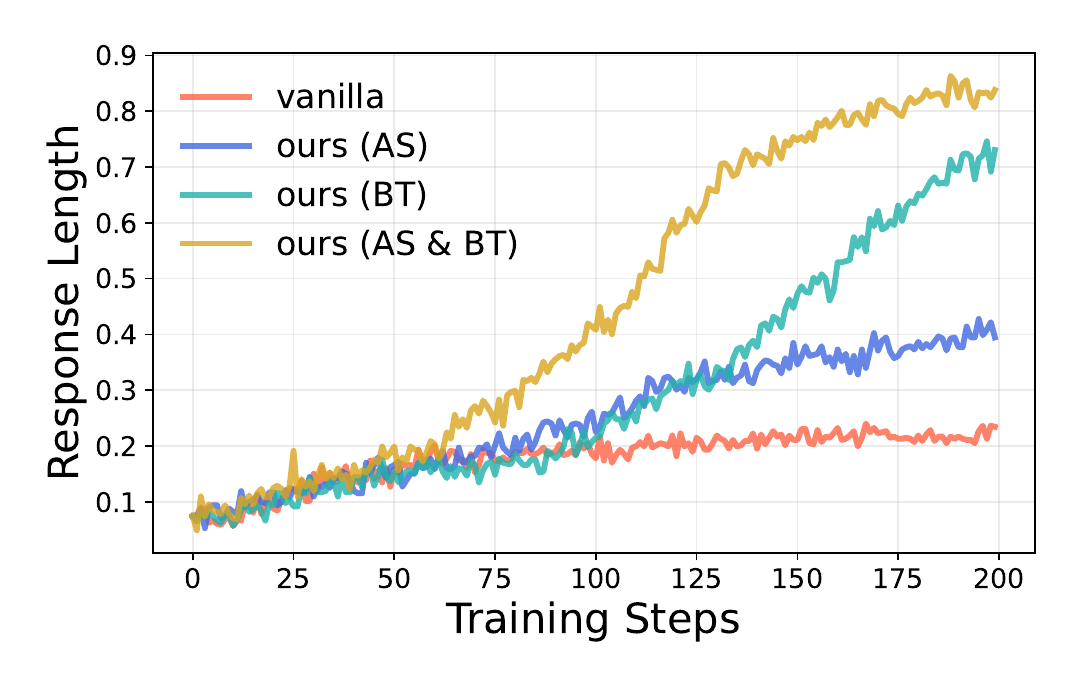}}
\subfloat[\label{fig:structured_position}Position]{\includegraphics[width=0.32\textwidth,trim=0 20 0 0,clip]{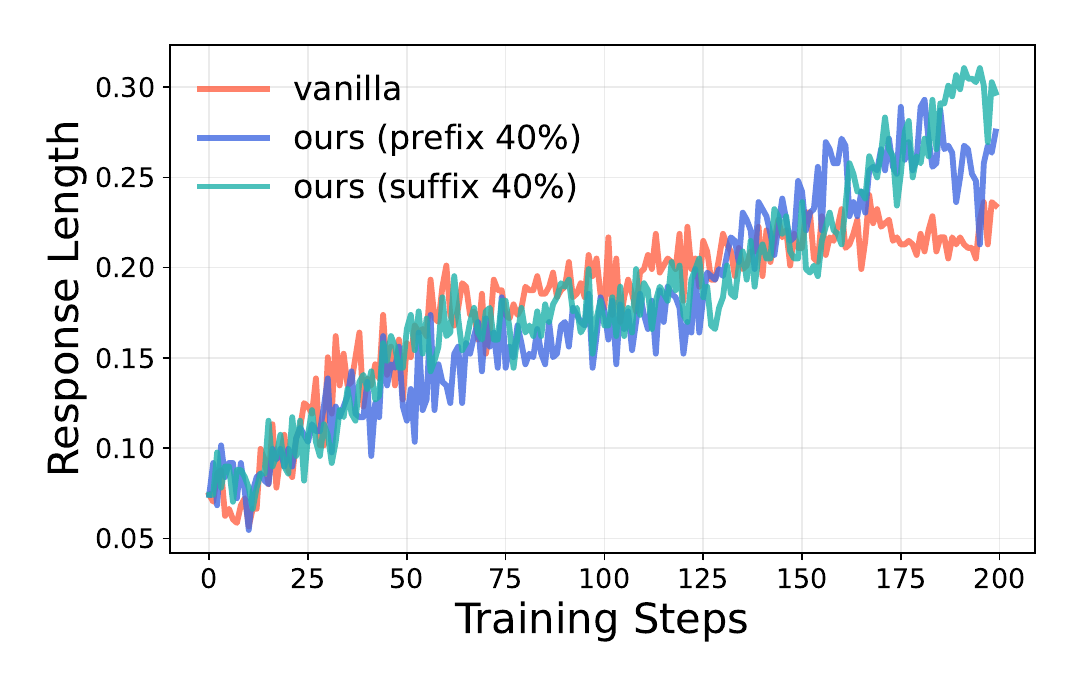}}
\subfloat[\label{fig:structured_sign}Sign]{\includegraphics[width=0.32\textwidth,trim=0 20 0 0,clip]{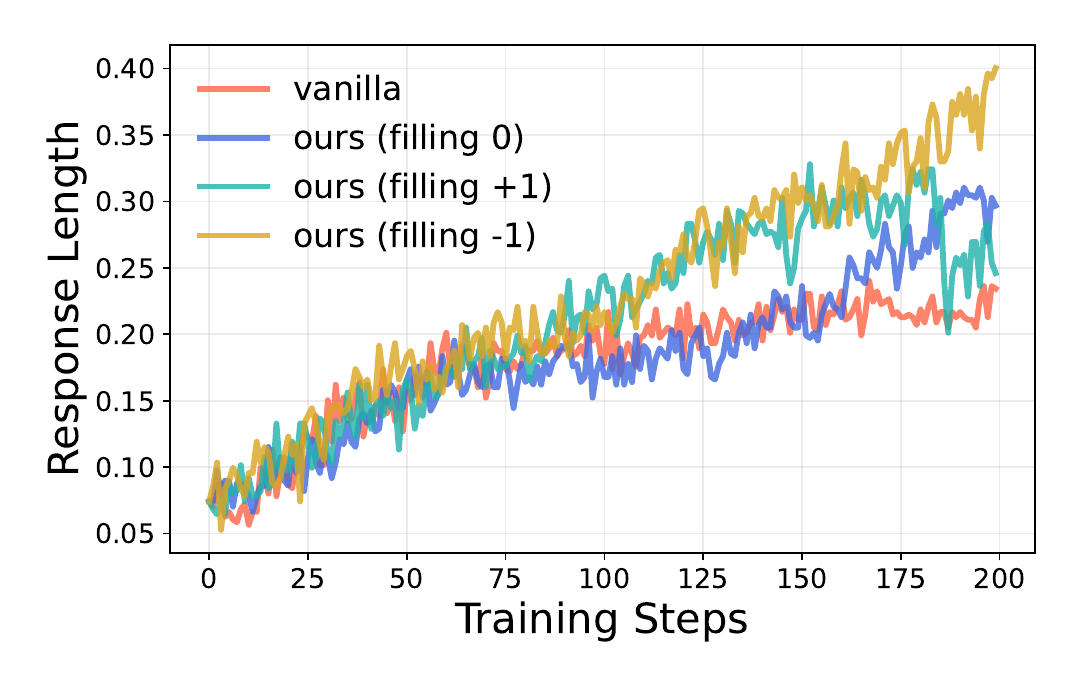}}
\caption{
Robustness of \proj under structured critique perturbations on PE-G$_{S=3}$ (PPO training with Qwen2.5-7B-Instruct).
(a) \textbf{Channel}: using only AS critiques, only BT critiques, or both channels, with the removed channel set to $0$.
(b) \textbf{Position}: keeping only the first 40\% or last 40\% of interaction-turn critiques, with the remaining labels set to $0$.
(c) \textbf{Sign}: keeping only the last 40\% of labels and filling the remaining labels with constant $0$, $+1$, or $-1$.
Across all three perturbation families, \proj remains consistently above the vanilla baseline, indicating robustness not only to random noise but also to several forms of structured critique destruction.
}
\label{fig:structured_critique_perturbations}
\end{figure*}

\textbf{Robustness to belief-elicitation prompt variation.}
Since BT critiques are constructed from a task-relevant confidence readout $\widetilde \Psi$ (\textit{cf.,} Sec.~\ref{sec:method_channel}), we also test whether \proj is sensitive to the exact prompt used to elicit this readout.
Fig.~\ref{fig:a.1} reports a variant of the PE-G$_{S=3}$ experiment where the belief-elicitation prompt is changed to encourage more aggressive confidence updates, while keeping the rest of the training pipeline unchanged.
The resulting reward dynamics remain substantially above the vanilla PPO baseline, indicating that \proj does not require a perfectly calibrated or uniquely specified belief readout.
Instead, the confidence readout only needs to provide a weakly informative directional signal on average, which aligns with the weighted-accuracy interpretation in Proposition~\ref{prop:weighted_acc_AS_main}.

\begin{figure*}
    \centering
    \captionsetup{aboveskip=4pt}
    \subfloat[\label{fig:a.1}Alternative belief-elicitation]{\includegraphics[width=0.31\textwidth]{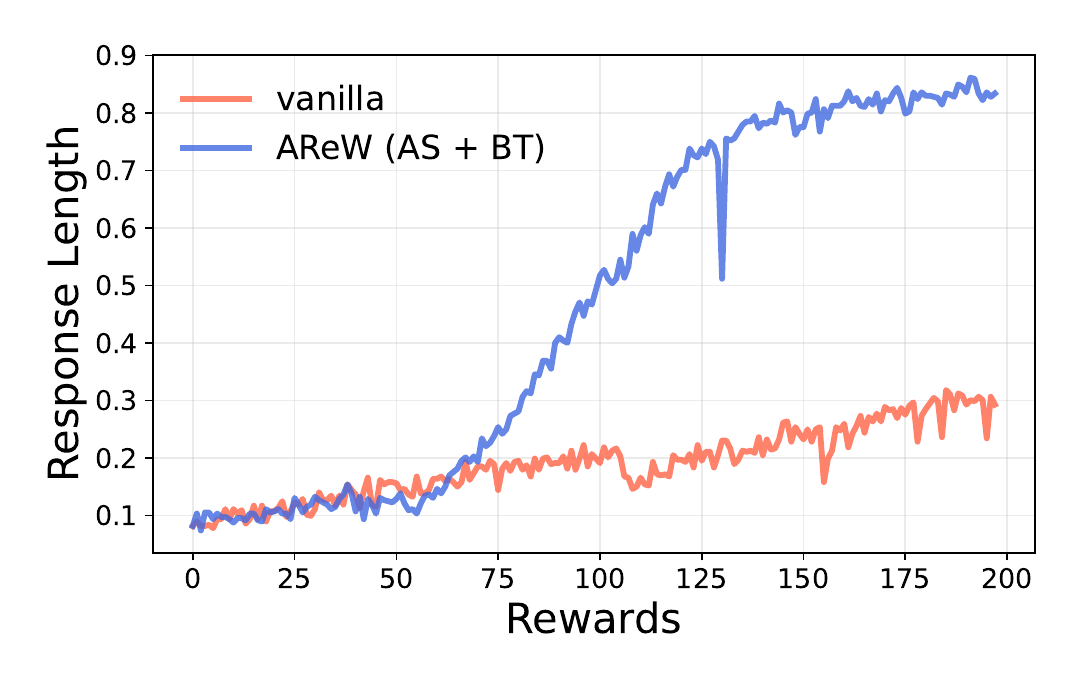}} 
    \hspace{3em}
    \subfloat[\label{fig:a.2}Various $\lambda u$ strength of \proj]{\includegraphics[width=0.3\textwidth]{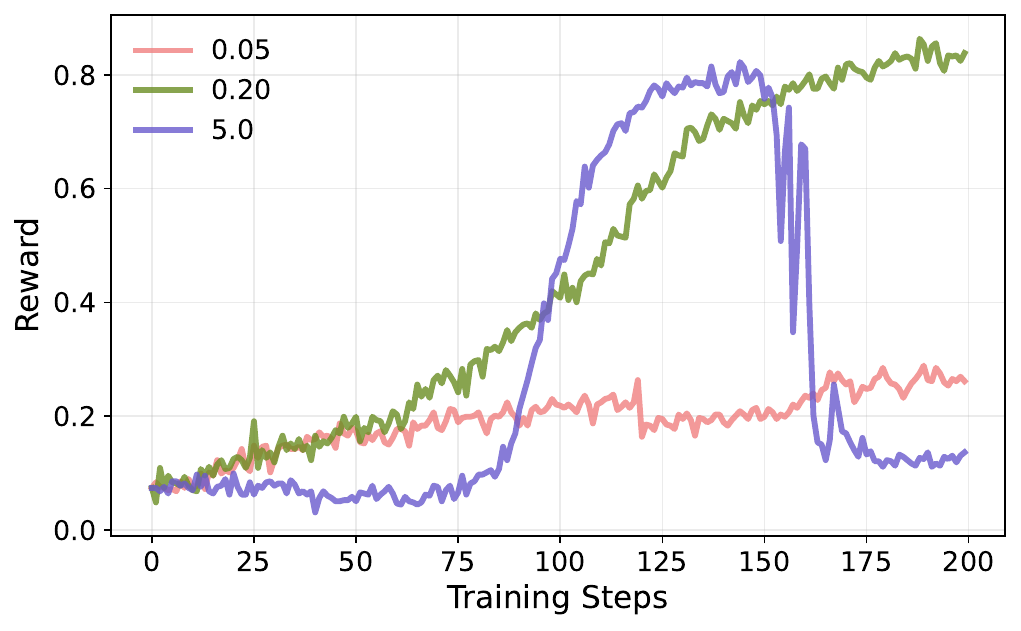}} 
   \caption{
\textbf{(a)} Training dynamics of rewards on PE-G$_{S=3}$ (PPO training with Qwen2.5-7B-Instruct).
Compared with the experiment in Fig.~\ref{fig:m.1.1}, this variant uses an alternative belief-elicitation prompt that encourages more aggressive confidence updates, with the rest of the training pipeline unchanged.
\textbf{(b)}  Training dynamics of rewards on  under different strength of \proj.
    }
    \label{fig:a_misc}
    \vspace{-0.1in}
\end{figure*}

\textbf{Comparison with directly injecting intermediate rewards.}
We next compare \proj with a direct intermediate-reward baseline that uses the same critique source, but injects the critique as an additional dense reward rather than as advantage reweighting.
This comparison isolates the effect of the injection mechanism: both methods receive the same weak directional information, but intermediate reward shaping changes the return and critic target, whereas \proj preserves the original outcome reward and only reallocates actor-side credit within a trajectory.
Table~\ref{tab:perturb_compare} shows that \proj consistently outperforms direct intermediate reward injection under random critique sign flipping.
The gap is especially large on PE-G$_{S=3}$ under noisy critiques: at $\alpha=0.2$, intermediate reward drops to 32.33, while \proj reaches 65.00; at $\alpha=0.5$, intermediate reward falls back to the vanilla level, while \proj still improves over vanilla.
These results support the view that weak and noisy step-level signals are better used for credit reallocation than for redefining the episode utility.

\begin{table}[t]
\caption{Robustness  comparing \proj with direct intermediate reward under random critique sign flipping (PPO training on PE-G$_{S=3}$ and FloDial-Hard datasets, Qwen2.5-7B-Instruct), where each  critique is independently flipped with probability $\alpha$.}
\vspace{-0.05in}
\label{tab:perturb_compare}
\centering
\resizebox{0.65\textwidth}{!}{
\begin{tabular}{llcccc}
\midrule
& & vanilla RL & $\alpha$-0 & $\alpha$-0.2 & $\alpha$-0.5 \\
\midrule
\multirow{2}{*}{PE-G$_{S=3}$}
& Intermediate reward 
& \multirow{2}{*}{18.33} 
& 73.67 
& 32.33 
& 18.33 \\
& \textsc{AReW} 
& 
& 80.33 {\oo {\scriptsize $\uparrow$ 6.66}} 
& 65.00 {\oo {\scriptsize $\uparrow$ 32.67}} 
& 30.30 {\oo {\scriptsize $\uparrow$ 11.97}} \\
\midrule
\multirow{2}{*}{\textsc{FloDial-Hard}}
& Intermediate reward 
& \multirow{2}{*}{21.33} 
& 33.00 
& 28.00 
& 21.00 \\
& \textsc{AReW} 
& 
& 36.00 {\oo {\scriptsize $\uparrow$ 3.00}} 
& 29.00 {\oo {\scriptsize $\uparrow$ 1.00}} 
& 23.30 {\oo {\scriptsize $\uparrow$ 2.30}} \\
\midrule
\end{tabular}
}
\vspace{-0.12in}
\end{table}

\textbf{Reduced critique-chasing under corrupted positive signals.}
We further test a structured reward-hacking scenario where each negative critique is independently flipped to positive with probability $p$.
This perturbation creates misleading positive supervision and directly tests whether the method chases corrupted critique labels.
Table~\ref{tab:flip_sign} reports both final task performance and the average final positive/negative token ratio, where the positive count includes flipped positives.
Direct intermediate reward produces much larger positive/negative ratios, suggesting that the policy more aggressively optimizes toward corrupted positive critique signals.
In contrast, \proj maintains lower positive/negative ratios while achieving better performance under both $p=0.3$ and $p=0.5$.
This indicates that advantage reweighting is less susceptible to critique chasing than direct reward injection, because it uses critiques to redistribute policy-gradient magnitude without changing the sparse task objective.

\begin{table}[t]
\caption{Robustness  comparing \proj with direct intermediate reward under structured critique sign flipping, where each negative critique is independently flipped to positive with probability $p$
(PPO training on PE-G$_{S=3}$ dataset, Qwen2.5-7B-Instruct). We report the average final positive/negative token ratio (including flipped positives), \textit{pos/neg}, and the resulting test performance under two corruption levels. }
\vspace{-0.05in}
\label{tab:flip_sign}
\centering
\resizebox{0.55\textwidth}{!}{
\begin{tabular}{lcccc}
\midrule
& \multicolumn{2}{c}{$p$-0.3} & \multicolumn{2}{c}{$p$-0.5} \\
\cmidrule(lr){2-3} \cmidrule(lr){4-5}
& pos/neg & performance & pos/neg & performance \\
\midrule
Intermediate reward 
& 5.94 
& 23.33 
& 7.26 
& 17.33 \\
\textsc{AReW} 
& 1.58 
& 30.67 {\oo {\scriptsize $\uparrow$ 7.34}} 
& 4.22 
& 26.00 {\oo {\scriptsize $\uparrow$ 8.67}} \\
\midrule
\end{tabular}
}
\vspace{-0.12in}
\end{table}

\textbf{Unified AS/BT evaluation with rubric-based labeling.}
The main experiments use task-grounded AS and BT proxies, which are designed to measure whether the agent acquires useful information and whether it incorporates the acquired evidence.
To check whether the conclusions depend on these task-specific proxy definitions, we additionally evaluate AS and BT capabilities using GPT-5 rubric-based labeling.
Table~\ref{tab:gpt5_proxy_corr} compares the rubric-based evaluation with the proxy metrics used in the main text on PE-G$_{S=3}$, MediQ, and FloDial-Hard.
Across all three tasks and both evaluation sources, \proj--AS+BT consistently improves over vanilla RL.
Although the absolute proxy values differ across evaluation sources, the qualitative trend remains stable: \proj improves both action-selection quality and belief-tracking behavior.
This supports the interpretation that the observed AS/BT improvements are not artifacts of a single hand-designed proxy, but reflect a broader improvement in the agent's information acquisition and evidence-integration behavior.

\begin{table}[t]
\caption{Evaluations of AS and BT capabilities over vanilla RL and \proj -- AS + BT (PPO algorithm with Qwen-2.5-7B-Instruct). We compare two evaluation sources: (1) GPT-5 rubric-based labeling and (2) the critiques used in the main text.
We annotate the improvement of \proj over vanilla RL in each cell.
Across all three tasks, AREW consistently improves over vanilla RL under both evaluation sources.}
\vspace{-0.05in}
\label{tab:gpt5_proxy_corr}
\centering
\resizebox{0.8\textwidth}{!}{
\begin{tabular}{llcccc}
\midrule
& & \multicolumn{2}{c}{AS} & \multicolumn{2}{c}{BT} \\
\cmidrule(lr){3-4} \cmidrule(lr){5-6}
& & GPT-5 labeling & Used in the main text & GPT-5 labeling & Used in the main text \\
\midrule
\multirow{2}{*}{PE-G$_{S=3}$}
& vanilla RL 
& 0.177 
& 0.049 
& 0.238 
& -0.005 \\
& \proj -- AS + BT 
& 0.669 {\oo {\scriptsize $\uparrow$ 0.49}} 
& 0.834 {\oo {\scriptsize $\uparrow$ 0.79}} 
& 0.459 {\oo {\scriptsize $\uparrow$ 0.22}} 
& 0.014 {\oo {\scriptsize $\uparrow$ 0.02}} \\
\midrule
\multirow{2}{*}{\textsc{MediQ}}
& vanilla RL 
& 0.275 
& 0.222 
& 0.355 
& 0.060 \\
& \proj -- AS + BT 
& 0.335 {\oo {\scriptsize $\uparrow$ 0.06}} 
& 0.539 {\oo {\scriptsize $\uparrow$ 0.32}} 
& 0.666 {\oo {\scriptsize $\uparrow$ 0.31}} 
& 0.103 {\oo {\scriptsize $\uparrow$ 0.04}} \\
\midrule
\multirow{2}{*}{\textsc{FloDial-Hard}}
& vanilla RL 
& 0.343 
& 0.250 
& 0.495 
& 0.130 \\
& \proj -- AS + BT 
& 0.468 {\oo {\scriptsize $\uparrow$ 0.13}} 
& 0.370 {\oo {\scriptsize $\uparrow$ 0.12}} 
& 0.636 {\oo {\scriptsize $\uparrow$ 0.14}} 
& 0.790 {\oo {\scriptsize $\uparrow$ 0.66}} \\
\midrule
\end{tabular}
}
\vspace{-0.12in}
\end{table}

\textbf{Effect of the reweighting strength $\lambda u$.}
We examine the impact of the reweighting strength $\lambda u$ on training dynamics.
As illustrated in Fig.~\ref{fig:a.2}, insufficient reweighting fails to provide enough signal to escape the self-locking regime, resulting in slow and limited improvement.
Conversely, overly aggressive reweighting accelerates early optimization but often leads to unstable training and eventual performance collapse.
We conjecture that large $\lambda u$ amplifies high-variance advantage estimates and over-emphasizes a small subset of steps, making the policy update brittle and sensitive to noise.
These observations suggest a non-trivial trade-off between exploration, stability, and convergence speed.
We leave more adaptive or structure-aware reweighting strategies as an important direction for future work.

\section{Setup Details}
\subsection{Dataset Details and Prompt Templates}\label{app:dataset}

In this section, we present more details for the datasets and tasks evaluated in this work. 

\textbf{Preference Estimation -- Gated (PE-G)}, adapted from \citet{badola2025multi} and \citet{zou2026reducing}.
Adapted from \citet{badola2025multi},
Gated-PE is an interactive preference inference task under constrained information acquisition.
The agent is given a finite set of items $\mathcal{X} = \{x_1,\dots,x_N\}$, where each item $x_i$ is represented by a known attribute vector $\mathbf{a}_i \in \mathbb{R}^D$.
The user is characterized by an unknown latent preference vector $\mathbf{w}^\star \in [0,1]^D$.
Through interaction, the agent maintains and iteratively refines an estimate $\mathbf{w}_t \in [0,1]^D$ of this latent preference.
At each decision point, the agent actively selects a low-dimensional attribute subspace $S_t \subseteq \{1,\dots,D\}$ and an item comparison $(x_i, x_j) \in \mathcal{X} \times \mathcal{X}$ designed to elicit the user's preference feedback restricted to $S_t$.
Based on the observed feedback, the agent updates its belief state.
The objective is to accurately recover $\mathbf{w}^\star$ under sparse, outcome-based supervision.

\textbf{MediQ}, adapted from \citet{li2024mediq}.
Adapted from \citet{li2024mediq},
MediQ is an interactive medical inference task that models hypothesis-level belief tracking under partial observability.
The agent is provided with a clinical vignette and an associated medical question whose answer lies in a finite hypothesis set of size $D$.
The agent maintains a belief state $\mathbf{w}_t \in [0,1]^D$, where each dimension represents the current support for a candidate hypothesis.
Through iterative interaction, the agent actively queries the LLM-simulated user for diagnostic information, receives structured feedback, and updates each hypothesis score accordingly.
The learning objective is to progressively concentrate belief mass onto the correct hypothesis.

\textbf{FloDial}, adapted from \citet{raghu2021end,hu2024uncertainty} consists of multi-turn diagnostic dialogues for resolving user-reported issues.
It provides a scenario where a
customer support technician interacts with customers to identify and resolve faults or issues within
computer systems, electronic devices, machinery, or other complex systems. The agent simulates the customer support
technician, which chat with the customer to further check the specific issues of device through multi-turn interactions.

\subsection{Baseline Details}\label{app:baselines}

Here we introduce RL algorithms used in our experiments.
Formally, given an actor model $\pi_\theta$,
the likelihood of a response $y$ to a query $x$ under the policy $\pi_\theta$ is modeled as $\pi_\theta (y | x)=\prod_{t=1}^{|y|} \pi_\theta (y_t | x, y_{<t} )$.
Given a query-response pair $(x, y)$, a verifier $r$ generates its reward $r(x,y)\in[0,1]$.

\textbf{Proximal Policy Optimization (PPO)}~\citep{schulman2017proximal}
employs the following objective for policy optimization:
\begin{align}
\mathcal{J}_\text{PPO}(\theta) = \mathbb{E}_{ x \sim \mathcal{D},\, y \sim \pi_{\theta_\text{old}}( \cdot | x) }
\left[ \frac{1}{|y|} \sum_{t=1}^{|y|} 
\min \left( w_{t}(\theta) \widehat{A}_{t},  \, \mathrm{clip} \left( w_{t}(\theta), 1 - {\varepsilon}, 1 + {\varepsilon}\right) \widehat{A}_{t} \right)
\right],
\end{align}
where the importance ratio of the token $y_t$ is defined as
$
w_{t}(\theta) = \frac{ \pi_{\theta} (y_{t} | x, y_{<t}) }{ \pi_{\theta_\text{old}} (y_{t} | x,y_{<t})}
$,
the advantage $\widehat{A}_{t}$ of $y_t$ is typically computed via Generalized Advantage Estimation (GAE)~\citep{schulman2015high} with temporal-difference errors, and $\varepsilon$ is the clipping range of importance ratios.

\textbf{Group Relative Policy Optimization (GRPO)}~\citep{shao2024deepseekmath} proposes computing the relative advantage of each response within a group of responses of the same query using the following objective (omitting the KL regularization term):
\begin{align}
\mathcal{J}_\text{GRPO}(\theta) = \mathbb{E}_{ x,\, \{y_i\}_{i=1}^G  }
\left[ \frac{1}{G} \sum_{i=1}^{G} \frac{1}{|y_i|} \sum_{t=1}^{|y_i|} 
\min \left( w_{i,t}(\theta) \widehat{A}_{i,t},  \,  \mathrm{clip} \left( w_{i,t}(\theta), 1 - {\varepsilon}, 1 + {\varepsilon}\right) \widehat{A}_{i,t} \right)
\right],
\label{equ:grpo}
\end{align}
where $\{y_i\}_{i=1}^G\sim \pi_{\theta_\text{old}}( \cdot | x)$ and $G$ is the group size. The importance ratio $w_{i,t}(\theta)$ and advantage $\widehat{A}_{i,t}$ of token $y_{i,t}$ are defined as:
\begin{align}
    w_{i,t}(\theta)=\frac{ \pi_{\theta} (y_{i,t} | x, y_{i,<t}) }{ \pi_{\theta_\text{old}} (y_{i,t} | x,y_{i,<t})},\,\,
    \widehat{A}_{i,t} = \frac{r(x, y_i) - \mathrm{mean} \left( \{ r(x, y_i) \}_{i=1}^G \right) }{ \mathrm{std} \left( \{ r(x, y_i) \}_{i=1}^G \right) },
\end{align}
respectively, where all the tokens in $y_i$ share the same advantage.

\textbf{Group Sequence Policy Optimization (GSPO)}~\citep{zheng2025group} extends GRPO by defining the importance ratio at the sequence level with length normalization, with sequence-level clipping, rewarding, and optimization. The objective is:
\begin{align}
\mathcal{J}_{\text{GSPO}}(\theta)
= \mathbb{E}_{x,\{y_i\}_{i=1}^G} \Bigg[ \frac{1}{G} \sum_{i=1}^G \min \big( s_i(\theta)\widehat A_i,\ \operatorname{clip}(s_i(\theta),1-\epsilon,1+\epsilon)\widehat A_i \big) \Bigg],
\end{align}
where
\[
s_i(\theta) = \left( \frac{\pi_\theta(y_i|x)}{\pi_{\theta_{\text{old}}}(y_i|x)} \right)^{1/|y_i|}
= \exp\!\left( \frac{1}{|y_i|}\sum_{t=1}^{|y_i|}\log \frac{\pi_\theta(y_{i,t}|x,y_{i,<t})}{\pi_{\theta_{\text{old}}}(y_{i,t}|x,y_{i,<t})} \right).
\]
\% By operating on entire sequences rather than token-level updates, GSPO achieves more stable training dynamics for LLM reinforcement learning.

\subsection{Supplementary Implementation Details}\label{app:impl}

\subsubsection{Additional Setting on Datasets}
Here we provide additional implementation details. 
The maximum number of interaction turns is set at 10 for PE-G, 12 for PE-F, 8 for MediQ, 10 for FloDial-Easy and FloDial-Hard. 
For RL training, we define task-specific rewards aligned with their evaluation metrics: 
for PE-G and PE-F, the reward is the similarity improvement compared to the initial default guess (all 0.5), where PE-G leverages binary reward ($\mathbf{1}[\mathrm{improvement}] > 0.03$) and PE-F leverages the continuous improvement value (normalized to $[0,1]$).
For MediQ and FloDial, the reward is binary, checking if the final decision made by the agents aligns with the ground-truth.
All rewards are provided only at the terminal step of each trajectory, consistent with the outcome-based RL setting.

The AS and BT proxies on PE and MediQ datasets can be seen in Sec.~\ref{sec:phenomenon}. For the FloDial dataset, the query is considered as uninformative if the user replies ``unknown'', and the BT proxy is defined by whether the agent can increase the confidence of the ground truth when receiving informative feedback. 

\textbf{Critique construction (AS vs.\ BT).}
Across datasets, we use two lightweight stepwise critiques:
$z_t^\mathrm{AS}$ (AS-channel) evaluates whether the \emph{query step} at turn $t$ yields informative feedback,
and $z_t^\mathrm{BT}$ (BT-channel) evaluates whether the subsequent \emph{belief update} is consistent with that feedback.
Below we summarize the dataset-specific instantiations.

\textbf{PE-G and PE-F.}

\emph{AS critique.}
Let the action be a comparison between two items, inducing a signed attribute-difference vector
$m_{ij}\in\mathbb R^{D}$ (restricted to the focused coordinates when applicable).
We mark the query as informative if it exhibits a non-trivial trade-off across coordinates, i.e.,
\[
z_t^\mathrm{AS}=+1
\quad\Longleftrightarrow\quad
\exists k,\ell\in[D]\ \text{s.t.}\ m_{ij,k}\,m_{ij,\ell}<0,
\]
and set $z_t^\mathrm{AS}=-1$ for repeated or invalid queries or when $m_{ij}$ is one-signed.

\emph{BT critique.}
Conditioned on an informative query ($z_t^\mathrm{AS}=+1$), we compare the updated preference estimate $v_t$
against the previous valid estimate $v_{t-1}$ using similarity to the ground-truth preference $w^\star$:
\[
z_t^\mathrm{BT}
=
\mathrm{sign}\!\Big(
\mathrm{sim}(v_t,w^\star)-\mathrm{sim}(v_{t-1},w^\star)
\Big),
\]
and set $z_t^\mathrm{BT}=0$ when the query is uninformative or invalid.

\textbf{MediQ.}

\emph{AS critique.}
Let $f_t$ denote the patient feedback to the query at turn $t$.
We set $z_t^\mathrm{AS}=+1$ iff the feedback is informative (i.e., not an ``Unknown / cannot answer'' response)
and the query is not repeated; otherwise $z_t^\mathrm{AS}=-1$.
We additionally explore counterfactual pattern (exploration on other datasets as future work).
When counterfactual queries are enabled, 
we intervene the belief from the last turn and obtain the counterfactual query.
We then discourage semantic repetition by requiring
the counterfactual query to be sufficiently different from the actual query (measured by token-level overlap),
and mark the step as uninformative when this difference is small.

\emph{BT critique.}
Let $w_t\in[0,1]^4$ be the belief over $4$ hypotheses and $s^\star$ the ground-truth hypothesis.
Define the margin
\[
\mathrm{mar}(w_t;s^\star)
:=
w_t(s^\star)-\max_{s\neq s^\star} w_t(s).
\]
If the feedback is uninformative, we enforce \emph{invariance}:
$z_t^\mathrm{BT}=+1$ if $\mathrm{mar}(w_t;s^\star)=\mathrm{mar}(w_{t-1};s^\star)$ and $z_t^\mathrm{BT}=-1$ otherwise.
If the feedback is informative, we employ a counterfactual-consistency check:
We intervene the observation to ``unknown'' and obtain
 $w_t^{\mathrm{cf}}$ to be a counterfactual update from the same previous belief $w_{t-1}$ but from the ``unknown'' response; then
\[
z_t^\mathrm{BT}=+1
\quad\Longleftrightarrow\quad
\|w_t^{\mathrm{cf}}-w_{t-1}\|_2 \le \|w_t-w_{t-1}\|_2,
\]
and $z_t^\mathrm{BT}=-1$ otherwise.

\textbf{FloDial.}
\emph{AS critique.}
User feedback takes the form \texttt{Yes/No} when the query matches a reference diagnostic item,
and \texttt{Unknown} otherwise. We set
\[
z_t^\mathrm{AS}=
\begin{cases}
+1, & \text{if feedback is \texttt{Yes}},\\
0,  & \text{if feedback is \texttt{No}},\\
-1, & \text{if feedback is \texttt{Unknown} (no match)}.
\end{cases}
\]
\emph{BT critique.}
Let $w_t\in[0,1]^S$ be the belief over candidate faults and $s^\star$ the ground-truth fault.
When the feedback is informative (\texttt{Yes/No}), we encourage increasing confidence on $s^\star$:
\[
z_t^\mathrm{BT}=+1 \ \Longleftrightarrow\ w_t(s^\star) > w_{t-1}(s^\star),
\qquad
z_t^\mathrm{BT}=-1 \ \text{otherwise}.
\]
When the feedback is uninformative (\texttt{Unknown}), we use an invariance-style rule:
we require the belief to remain unchanged
(e.g., $z_t^\mathrm{BT}=+1$ iff $w_t=w_{t-1}$).

\subsubsection{Training Configurations}
All expriments are trained on a single node with 8 B200 GPUs, based on the implementations of Verl~\citep{sheng2025hybridflow}. All training tasks on PPO are conducted for 200 steps (GRPO and GSPO 100 steps) with the actor model optimized using a learning rate of $1.0\times10^{-6}$. 
\% 
For distributed training, we adopt Fully Sharded Data
Parallelism (FSDP), using BFloat16 precision throughout both training and evaluation.
\% 
For efficient LLM rollouts, we adopt vLLM~\footnote{\url{https://docs.vllm.ai/en/latest/}} with a tensor parallel size of 1. The rollout sampling uses a temperature of 1.0 for all datasets.

For the PPO baseline, we use Generalized
Advantage Estimation (GAE) with parameters $\lambda=1$ and $\gamma=1$. The clip ratio $\varepsilon$ are set to 0.2. For GRPO training, we sample 3 responses per prompt, and the rollout parameters with the clip ratio are consistent with the PPO setting. For the GSPO algorithm,  the clip ratio $\varepsilon_{low}$ and $\varepsilon_{high}$ are set to 0.0003 and 0.0004, respectively, while others keep consistent with GRPO training.

\begin{figure*}
\begin{tcolorbox}[
  colback=gray!5!white,
  colframe=gray!75!black,
  title=Input Prompts for the MovieRec Preference Estimation Dataset
]
\scriptsize

You are an interactive preference estimation agent. The goal is to infer a user's
{\var{len\_attributes}}-dimensional hidden preference vector on movies through multi-round interaction.

\vspace{1em}
\#\# Setup:

- You are given {\var{len\_seen}} movies, each with scores on
  {\var{len\_attributes}} dimensions (indexed $1 \dots {\var{len\_attributes}}$):

{\var{seen\_movie\_sample}}

\vspace{0.8em}
- Maintain an estimate of the user's preference vector:

\quad Guess: $w_1, w_2, \dots, w_{{\var{len\_attributes}}}$

\vspace{0.5em}
- Initialization:

\quad Guess: $0.5,0.5,\dots,0.5$

\vspace{1em}
\#\# Interaction Protocol:

Interaction alternates between two types of rounds.

\vspace{0.8em}
1) \textbf{Action Round (odd-numbered rounds: 1,3,5,\dots)}

In an Action Round, you must choose which information to query.
You must output exactly:

\begin{flushleft}
\texttt{<interact>}\\
Focus: k1,k2,k3\\
Pair: p1,p2\\
\texttt{</interact>}
\end{flushleft}

Rules:

- Focus must contain exactly three distinct integers in
  $[1,{\var{len\_attributes}}]$.
  Ensure that all dimensions receive opportunities to be disambiguated over the course of interaction.

- Pair must contain exactly two distinct integers in
  $[1,{\var{len\_seen}}]$.

- Avoid uninformative pairs, such as dominance pairs where one movie is better
  on all focused dimensions $k_1,k_2,k_3$.

\vspace{0.5em}
After you output Focus and Pair, the user will provide feedback:

\vspace{0.3em}
- User Feedback: ``Yes'', ``No'', or ``Equal''.

\quad \textbf{``Yes'' means $p_1$ is preferred when considering only dimensions $k_1,k_2,k_3$.}

\vspace{0.8em}
2) \textbf{Update Round (even-numbered rounds: 2,4,6,\dots)}

In an Update Round, you must update the preference estimate based on the most recent
Focus, Pair, and User Feedback.

You must output exactly:

\begin{flushleft}
\texttt{<interact>}\\
Guess: w1,w2,...,w{\var{len\_attributes}}\\
\texttt{</interact>}
\end{flushleft}

Rules:

- Guess must be comma-separated numbers.

- Use the feedback to adjust the relative importance of the focused dimensions
  in a way consistent with the observed preference.

\vspace{1em}
\#\# Reasoning:

Before each \texttt{<interact>} block, you may briefly reason about what to do in a:

\begin{flushleft}
\texttt{<scratch>...\ </scratch>}
\end{flushleft}

\vspace{0.8em}
Round 1 is an Action Round.

\end{tcolorbox}
\vspace{-3mm}
\caption{Prompt Template for MovieRec Preference Estimation.}
\label{fig:prompt_movrec}
\end{figure*}

\begin{figure*}
\begin{tcolorbox}[
  colback=gray!5!white,
  colframe=gray!75!black,
  title=Input Prompts for the MediQ Dataset
]
\scriptsize

You are an interactive medical inference agent.
Your goal is to iteratively maintain and refine a 4-dimensional state vector that tracks the relative support for four potential hypotheses associated with a given medical question, through multi-round interaction.

\vspace{1em}
\#\# Setup:

- You are given:
  (i) a clinical vignette describing a patient scenario, and
  (ii) a medical question associated with this scenario.

- The question admits four potential hypotheses, labeled A, B, C, and D.

- Maintain a state vector $(w_A,w_B,w_C,w_D)$, where each $w\in[0,1]$,
  representing your current level of support for each hypothesis.

- Initialization:

\quad Guess: $0.5,0.5,0.5,0.5$

- In each Action round, you can issue \textbf{only one} query.
  The query must be a single, atomic question intended to reduce uncertainty among the four hypotheses;
  compound or multi-part questions are not allowed.

\vspace{1em}
\#\# Alternating Interaction Protocol:

Interaction alternates between two types of rounds:

\vspace{0.8em}
1) \textbf{Action Round (odd-numbered rounds: 1,3,5,\dots)}

- You must output only a query in the exact format:

\begin{flushleft}
\texttt{<interact>}\\
Query: \dots\\
\texttt{</interact>}
\end{flushleft}

- After you output a Query, the user will provide feedback.

\vspace{0.8em}
2) \textbf{Update Round (even-numbered rounds: 2,4,6,\dots)}

- You must update and output only the state vector (Guess) based on:
  (a) the most recent Query, and
  (b) the feedback returned for that Query.

- Output in the exact format:

\begin{flushleft}
\texttt{<interact>}\\
Guess: wA,wB,wC,wD\\
\texttt{</interact>}
\end{flushleft}

- Guess must be comma-separated numbers.

\vspace{0.8em}
\#\# State Update Rules (in Update rounds, after feedback):

- Each dimension can be adjusted independently, depending on how the feedback affects that hypothesis.

\vspace{0.8em}
\#\# Query Policy (in Action rounds):

- Ask atomic, clinically meaningful queries that help differentiate among the four hypotheses.

- Avoid repeating previously asked queries.

\vspace{0.8em}
Before each \texttt{<interact>} block, briefly reason (in a few sentences) about what you should do in a:
\texttt{<scratch>...\ </scratch>} block, following the protocol above.

\vspace{1em}
Let's get started:

\vspace{0.5em}
Clinical Vignette: {\var{clinical\_vignette}}

\vspace{0.3em}
Medical Question: {\var{medical\_question}}

\vspace{0.3em}
Potential Hypotheses: {\var{potential\_hypotheses}}

\vspace{0.8em}
Round 1 is an Action Round. Output your first Query.

\end{tcolorbox}
\vspace{-3mm}
\caption{Prompt Template for MediQ.}
\label{fig:prompt_mediq}
\end{figure*}

\begin{figure*}
\begin{tcolorbox}[
  colback=gray!5!white,
  colframe=gray!75!black,
  title=Input Prompts for the FloDial Dataset
]
\scriptsize

You are an interactive troubleshooting diagnosis agent.
Your goal is to iteratively maintain and refine a state vector that tracks the relative plausibility of
{\var{num\_candidates}} candidate descriptions for a given problem, through multi-round interaction.

\vspace{1em}
\#\# Setup:

- You are given:
  (i) a task description describing the user's problem, and
  (ii) a list of {\var{num\_candidates}} candidate descriptions.

- Each candidate is a possible explanation, diagnosis, or resolution suggestion related to the problem.

- Exactly \textbf{one} of the candidate descriptions best matches the actual situation described by the user.

- Your goal is to identify which candidate description is the most consistent with the situation,
  by asking diagnostic yes/no questions.

\vspace{0.8em}
- Maintain a state vector $(w_1,w_2,\dots,w_{{\var{num\_candidates}}})$,
  where each $w_i\in[0,1]$, representing your current level of support for candidate $i$.

- Initialization:

\quad Guess: $0.5,0.5,\dots,0.5$ (length {\var{num\_candidates}})

\vspace{0.8em}
- In each Action Round, you issue \textbf{only one} query.

\vspace{1em}
\#\# Alternating Interaction Protocol:

Interaction alternates between two types of rounds:

\vspace{0.8em}
1) \textbf{Action Round (odd-numbered rounds: 1,3,5,\dots)}

- Output only a query in the exact format:

\begin{flushleft}
\texttt{<interact>}\\
Query: \dots\\
\texttt{</interact>}
\end{flushleft}

- The query must be a single, atomic yes/no diagnostic question intended to reduce uncertainty among
  the {\var{num\_candidates}} candidate descriptions.

- Avoid repeating previously asked queries.

- Do \textbf{not} directly ask whether a specific candidate description is correct or incorrect.

- Compound or multi-part questions are not allowed.

\vspace{0.8em}
2) \textbf{Update Round (even-numbered rounds: 2,4,6,\dots)}

- Update and output only the state vector (Guess), based on the most recent query and the feedback returned for that query.

- Output in the exact format:

\begin{flushleft}
\texttt{<interact>}\\
Guess: w1,w2,...,w{\var{num\_candidates}}\\
\texttt{</interact>}
\end{flushleft}

- Each $w_i$ must remain within $[0,1]$.

- If the user replies ``Unknown'', leave all weights unchanged.

- Each dimension can be adjusted independently by $0$ (not changed), $+0.1$, or $-0.1$,
  depending on how the feedback affects that candidate.

\vspace{0.8em}
Before each \texttt{<interact>} block, briefly reason (in a few sentences) in a
\texttt{<scratch>...\ </scratch>} block about what to ask or how the feedback affects your belief.

\vspace{1em}
Let's get started.

\vspace{0.8em}
Task Description:

{\var{task\_description}}

\vspace{0.6em}
Candidate Descriptions:

{\var{candidate\_descriptions}}

\vspace{0.8em}
Round 1 is an Action Round. Output your first query.

\end{tcolorbox}
\vspace{-3mm}
\caption{Prompt Template for FloDial.}
\label{fig:prompt_flodial}
\end{figure*}

\end{document}